%% file: neurips_2026.tex
\documentclass{article}

\PassOptionsToPackage{numbers}{natbib}
 \usepackage[preprint]{neurips_2026}

\usepackage[utf8]{inputenc} 
\usepackage[T1]{fontenc}    
\usepackage{url}            
\usepackage{booktabs}       
\usepackage{amsfonts}       
\usepackage{nicefrac}       
\usepackage{microtype}      
\usepackage{xcolor}         

\usepackage{algorithm}
\usepackage{algorithmic}
\input{macros}

\newcommand{\authorcell}[1]{%
  \parbox[t]{0.45\textwidth}{%
    \centering\normalfont
    #1%
  }%
}

\allowdisplaybreaks
\title{What's in a Smoothness Constant? Tighter Rates for Local SGD with Bounded Second-order Heterogeneity}

\author{%
	\authorcell{%
    \textbf{Kumar Kshitij Patel}\thanks{Equal contribution.} $^{\hspace{0.1em},}$\thanks{Part of this work was done while the authors participated in the Federated and Collaborative Learning Program at the Simons Institute for the Theory of Computing, Berkeley.}\\
	Institute for Foundations of Data Science\\
	Yale University\\
	New Haven, CT, USA 06520\\
	\texttt{kumarkshitij.patel@yale.edu}\\
    }
	\And
    \authorcell{%
    \textbf{Rustem Islamov}\footnotemark[1]\\
	Department of Mathematics and Computer Science\\
	University of Basel\\
	Basel, Switzerland 4001\\
	\texttt{rustem.islamov@unibas.ch}\\
    }
    \And
    \authorcell{%
    \textbf{Sebastian U Stich}\footnotemark[2]\\
	CISPA Helmholtz Center for\\
    Information Security\\
	Saarbrücken, Germany 66123\\
	\texttt{stich@cispa.de}\\
    }
    \And
    \authorcell{%
    \textbf{Aurelien Lucchi}\\
    Department of Mathematics and Computer Science\\
    University of Basel\\
	Basel, Switzerland 4001\\
    \texttt{aurelien.lucchi@unibas.ch}\\
    }
    \And
    \authorcell{%
    \textbf{Eduard Gorbunov}\\
    Department of Statistics and Data Science\\
    MBZUAI\\
	Abu Dhabi, UAE\\
    \texttt{eduard.gorbunov@mbzuai.ac.ae}\\
    }
    \And
    \authorcell{%
    \textbf{Lingxiao Wang}\\
    Department of Data Science\\
    New Jersey Institute of Technology\\
	Newark, NJ, USA 07102\\
    \texttt{lw324@njit.edu}\\
    }
}

\begin{document}

	\maketitle

	\begin{abstract}
		Local SGD, also known as Federated Averaging, is a widely used distributed optimization algorithm. Although Local SGD often outperforms alternatives such as Mini-batch SGD in practice, theory still only partially explains when and why local updates help under realistic data heterogeneity. Recent work by \citet{patel2025revisiting} shows that a bounded second-order heterogeneity assumption captures the efficiency of Local SGD for strongly convex objectives, and conjectures that the same principle extends to the general convex setting. In this paper, we prove this conjecture by establishing an improved convergence guarantee for Local SGD on general convex objectives under bounded second-order heterogeneity. We also improve the best-known lower bounds for Local SGD in this setting, showing that our upper bounds are nearly tight. Together, these results provide a sharper, more fine-grained convergence theory for Local SGD. As a further application of our techniques, we provide a lower bound for serial SGD with replacement, showing how second-order heterogeneity captures the impact of rare high-curvature clients.
	\end{abstract}
    
    \newpage
    
    \etocdepthtag.toc{main}
    
    \begingroup
    \setcounter{tocdepth}{2} 
    \etocsettagdepth{main}{subsection}
    \etocsettagdepth{appendix}{none}
    \tableofcontents
    \endgroup

	\section{Introduction}\label{sec:introduction}

	In large-scale distributed training of machine learning algorithms, communication is often the primary bottleneck, outweighing computation in both runtime and energy costs~\citep{mcmahan_ramage_2017,mcmahan2016communication,douillard2023diloco}. This raises a key question: how much local computation can each client perform between communication rounds while preserving convergence guarantees? To formalize this trade-off, we study the following distributed optimization problem:
	\begin{equation}\label{prob:scalar}
		\min_{x\in \rr^d}\cb{F(x)\coloneqq \frac{1}{M}\sum_{m=1}^{M}F_m(x)} ,
	\end{equation}
	where \(F_m(x)\coloneqq \ee_{z_m \sim \ddd_m}[f(x;z_m)]\) is the population loss on machine \(m\),
	defined using a loss function \(f(\cdot;~z\in~\zzz)\) and a local data distribution \(\ddd_m\in \Delta(\zzz)\). Optimization problems of the above form arise throughout machine learning, from multi-GPU training in data centers \citep{douillard2023diloco,krizhevsky2012imagenet} to decentralized training across millions of edge devices \citep{mcmahan2016communication,mcmahan2016federated}. A simple and widely studied model of communication for such problems is that of \emph{intermittent communication} (IC)~\citep{woodworth2018graph}. In the IC setting, machines communicate for \(R\) rounds, and within each round, each machine performs \(K\) units of computation. The pair \((K, R)\) captures the trade-off between local computation and communication.
	
	While numerous algorithms \citep{kairouz2019advances, wang2021field, karimireddy2020scaffold, bullins2021stochastic, patel2022towards, mishchenko2022proxskip, pathak2020fedsplit} have been proposed to solve Problem~\eqref{prob:scalar} in the IC setting, in practice, the dominant approaches remain first-order methods that compute $K$ stochastic gradients per client between communication rounds. This naturally raises a central question: how should these $K$ stochastic gradients be used? In particular, should each machine use them to perform \emph{sequential local updates}, or instead \emph{average} them to obtain a lower-variance estimate of the true local gradient? A clean way to study this question is to contrast \textbf{Local SGD} (Federated Averaging) \citep{mcdonald2009efficient, mcmahan2016federated, mcmahan2016communication}, which emphasizes local updates, with \textbf{Mini-batch SGD} \citep{dekel2012optimal}, which emphasizes averaging. These canonical baselines isolate the two key design primitives underlying local computation.

	Early analyses of Local SGD identified regimes in which it can outperform Mini-batch SGD, but these results rely on restrictive assumptions about data heterogeneity. In particular, many works either focus on the homogeneous setting, where all clients share the same data distribution \citep{stich2018local, woodworth2020local}, or assume bounded \emph{first-order heterogeneity}, which controls how much local gradients differ across clients. Concretely, these assumptions require that the gradients $\nabla F_m(x)$ remain uniformly close to the global gradient $\nabla F(x)$ for all $x$, effectively limiting the variability between client objectives. While analytically convenient, such conditions are quite stylized, as they rule out realistic settings in which data distributions can differ substantially across clients (see further discussion after Assumption~\ref{ass:zeta_everywhere}).

	In parallel, a separate line of work has shown that \emph{second-order heterogeneity}, which instead controls the similarity of local curvature (e.g., Hessians) across clients, can provably reduce the communication complexity of related distributed methods \citep{shamir2014dane, sun2022distributed, kovalev2022optimal, karimireddy2020mime, murata2021bias, patel2022towards, jiang2024fedred, jiang2024dane}, without requiring restrictive bounds on gradient dissimilarity.\footnote{Note that $\tau$-second-order heterogeneity (\Cref{ass:tau}) always holds with $\tau\le 2H$ if local objectives are $H$-smooth.} Building on this evidence, \citet{patel2024limits} conjectured that Local SGD should also benefit from bounded second-order heterogeneity. If true, this would identify a significantly broader and more realistic regime in which Local SGD can provably outperform Mini-batch SGD. While \citet{patel2025revisiting} recently confirmed this conjecture in the strongly convex setting, the general convex case has remained largely open; we close this gap in this paper.

	\subsection{Our Contributions}	
	We establish the first convergence guarantees for Local SGD in the general convex setting that \emph{explicitly capture the benefits of bounded second-order heterogeneity} (see Table~\ref{tab:rates}, \Cref{thm:lsgd_consensus_informal,thm:lsgd_singlemachine_informal}). We also prove a complementary upper bound under uniformly bounded first-order heterogeneity (\Cref{thm:lsgd_zeta_informal}; see \Cref{tab:rates}). Our second-order heterogeneity results show that Local SGD can provably improve upon Mini-batch SGD without relying on potentially restrictive uniform first-order heterogeneity assumptions, thereby identifying a broader and more realistic regime in which local updates are fundamentally advantageous.

	The key technical novelty in our analysis is a trajectory-dependent control of heterogeneity. Prior analyses of Local SGD typically bound client drift using a uniform heterogeneity parameter that controls gradient disagreement globally across the parameter space. Instead, we show that the consensus error depends only on the gradient disagreement encountered along the trajectory of the average iterate. We then couple this trajectory-wise heterogeneity with the iterate error itself via a self-bounding recursion, thereby allowing the dynamics of Local SGD to control its own heterogeneity.

	We further provide a new lower bound (\Cref{thm:lsgd_lb_informal,cor:zeta_lower_bound}) by developing sharper constructions that disentangle the distinct roles of first- and second-order heterogeneity~\citep{patel2025revisiting}. These results reveal that our upper bound is nearly tight and isolate the precise regimes in which second-order effects yield gains (see \Cref{fig:lower_vs_upper}). Together, our results resolve the conjecture of \citet{patel2024limits} in the general convex setting, providing definitive evidence that second-order heterogeneity is not merely a technical artifact but a key structural property governing the efficiency of local update methods.

	A by-product of our lower-bound idea is a refined lower-bound construction for multi-epoch SGD with replacement (\Cref{thm:fresh_sgd_heterogeneous_lb_informal}). It shows that the standard dependence on worst-case client smoothness remains tight even in heterogeneous settings, where the worst-case smoothness is much larger than that of the average objective.

    \subsection{Related work}
	The contrast between mini-batch and Local SGD has driven over a \emph{decade of work} on establishing tight non-asymptotic convergence guarantees for Local SGD \citep{mcdonald2009efficient, zinkevich2010parallelized, zhang2016parallel, stich2018local, dieuleveut2019communication, khaled2020tighter, koloskova2020unified, karimireddy2020scaffold, woodworth2020local, woodworth2020minibatch, yuan2020federated, woodworth2021min, woodworth2021minimax, glasgow2022sharp, wang2022unreasonable, patel2023federated, patel2023still, patel2024limits, patel2025revisiting, mangold2025refined, patel2025makes, luo2025revisiting, khaled2025understanding}. This effort has proved remarkably fruitful: it has motivated new primitives such as control variates \citep{karimireddy2020scaffold}, revealed separations between the min-max complexities of serial and distributed optimization \citep{woodworth2018graph, woodworth2021min}, helped discover a novel form of acceleration \citep{mishchenko2022proxskip}, motivated algorithms for training large language models~\citep{douillard2023diloco,charles2025communication,therien2025muloco} and more broadly supplied numerous analytical tools that extend to settings such as decentralized optimization \citep{koloskova2020unified}, communication compression~\citep{stich2018sparsified,karimireddy2019error,gao2024econtrol}, personalization \citep{pillutla2022federated,mishchenko2025partially}, quantization~\citep{alistarh2018convergence,magnusson2019quantization}, asynchronous updates~\citep{ye2018coding,stich2020error,islamov2024asgrad}, differential privacy~\citep{wei2020federated,girgis2021shuffled, lowy2023private,murata2023diff2, wang2024efficient,islamov2025double,shulgin2025smoothed}, and Byzantine robustness~\citep{alistarh2018byzantine,karimireddy2021learning,islamov2026byzantine}. Our paper further contributes to this broader discourse, providing both novel guarantees and analytical tools.

	\section{Setting and Preliminaries}\label{sec:setting}

	\noindent\textbf{Notation.} For $i \leq j \in \zz_{\geq 0}$, we use $[i,j]$ to denote $\cb{i, i+1, \dots, j-1, j}$, and when $i=1$, we denote it by $[j]$. We use $\lesssim$, $\gtrsim$, and $\cong$ for comparisons up to universal numerical constants. We also use $\delta(t)\coloneqq t - t\bmod K$ to denote the most recent communication round at or before time $t$.

	\subsection{Regularity Assumptions and the Algorithms}

    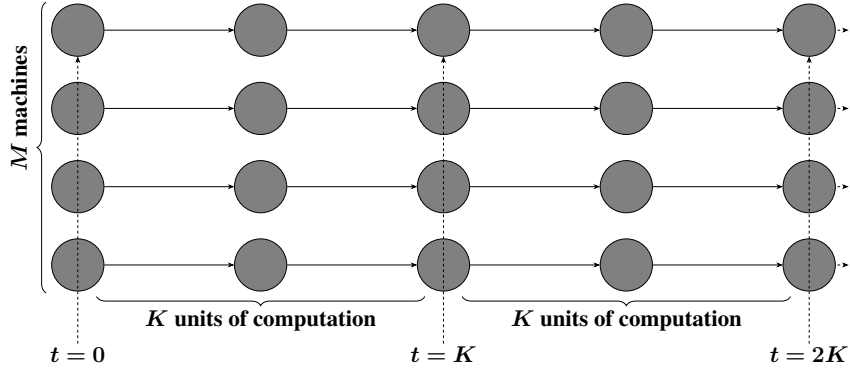
\begin{figure}[H]
		\centering
		\resizebox{0.8\textwidth}{!}{\input{IC}}
		\caption{Illustration of the IC setting: $M$ clients communicate $R$ times with $K$ units of computation in between. The baselines we consider in the paper all require $K$ stochastic gradient computations between communication rounds on each client.}
		\label{fig:IC}
	\end{figure}
    
	An instance of Problem \eqref{prob:scalar} can be characterized by the client distributions $\{\ddd_m\in \Delta(\zzz)\}_{m\in[M]}$ and a differentiable\footnote{Our analysis can be extended to sub-differentiable loss functions using standard changes to our assumptions and analysis.} loss function $f(\cdot;z\in\zzz):\rr^d\to \rr$. We denote the set of all problem instances by $\ppp$. To further restrict the problems we study, we assume throughout that for all $m\in[M]$, the objective function $F_m$ is differentiable, convex, and $H$-smooth, i.e., for all $x,y\in\rr^d$, 
	\begin{align}\label{ass:smth_cvx}
		0 \leq F_m(y) - F_m(x) - \inner{\nabla F_m(x)}{y-x} \leq \frac{H}{2}\norm{x-y}^2  .
	\end{align}
	Each machine \( m \in [M] \) computes stochastic gradients of \( F_m \) by sampling data from its distribution \( z \sim \ddd_m \) independent of other machines and time steps. These gradients satisfy for all \( x \in \mathbb{R}^d \),
	\begin{equation}
		\label{ass:stoch_first_order}
		\ee_{z\sim \ddd_m}
		\sb{\nabla f(x;z)
		}=\nabla F_m(x) \quad\text{and}\quad \ee_{z\sim \ddd_m}\sb{\norm{\nabla f(x; z) - \nabla F_m(x)}^2
		} \leq \sigma^2 .
	\end{equation}
	Finally, we assume that the average objective $F$ has at least one bounded minimizer, i.e., 
	\begin{equation}\label{ass:bounded_optima}
		\exists\ x^\star\in \arg\min_{x\in \rr^d} F(x)\quad \text{s.t.} \quad\norm{x^\star}\leq B .
	\end{equation}
	We will denote the class of all problems satisfying the above assumptions \eqref{ass:smth_cvx}, \eqref{ass:stoch_first_order}, and \eqref{ass:bounded_optima} by $\ppp^{H, B, \sigma}$. 
	
	We next describe our algorithms, i.e., Local SGD and Mini-batch SGD, with intermittent communication (see \Cref{fig:IC} for an illustration). The key difference between these algorithms is in how each client utilizes its local stochastic gradients between communication rounds.

	\paragraph{Local SGD.} Local SGD runs $K$ SGD steps independently on each machine, and then averages and synchronizes the resulting models across machines at each communication round. Specifically, at each time step $t \in [0,T-1]$, machine $m\in[M]$ samples $z_t^m \sim 
	\ddd_m$ and updates as follows:
	\begin{equation}\label{eq:local_updates}
			x_{t+1}^m \coloneqq 
			\begin{cases}
				x_t^m - \eta\,\nabla f(x_t^m; z_t^m)
				& \qquad (t+1)\bmod K \neq 0 ,\\[2pt]
				 \frac{1}{M}\sum_{n=1}^{M}\sb{x_t^n - \eta\,\nabla f(x_t^n; z_t^n)}
				& \qquad (t+1)\bmod K = 0 .
			\end{cases}
	\end{equation}
	We assume all the machines are initialized at $x_0=0$. The above algorithm is often called ``vanilla local SGD'' to distinguish it from another popular variant of Local SGD with inner and outer step sizes~\cite{karimireddy2020scaffold}. \Cref{alg:local-sgd} presents this two-step-size version of Local SGD, which often performs better in practice~\cite{charles2020outsized}, and has motivated an entire class of algorithms with an inner and outer optimization loop\footnote{For instance, see DILOCO for LLM optimization~\citep{douillard2023diloco,charles2025communication} and the variance-reduced local optimization algorithms for distributed non-convex optimization~\citep{murata2021bias,patel2022towards,wang2024efficient}.}. In \Cref{alg:local-sgd}, $x_{r,k}^m$ is the $k^{th}$ local model on machine $m$, leading up to the $r^{th}$ round of communication, while $x_r$ is the consensus (global) model at the end of the $r^{th}$ communication. For Local SGD, $\eta$ is referred to as the inner (local) step-size, while $\beta$ is the outer (server) step-size. Note that setting $\beta=1$ recovers \textit{``vanilla Local SGD''} with a single step-size.
	
	\begin{figure}
		\centering
		
		\begin{minipage}[t]{0.49\textwidth}
			\begin{algorithm}[H]
				\caption{Local SGD / Federated Averaging with server step-size $\beta$}
				\label{alg:local-sgd}
				\begin{algorithmic}[1]
					\STATE {\bfseries Input:} communication rounds $R$, local steps $K$, step-size $\eta$, server step-size $\beta$, client distributions $\cb{\ddd_1,\dots,\ddd_M}$
					\STATE {\bfseries Initialize:} global model $x_0 = 0$
					\FOR{$r=1$ {\bfseries to} $R$}
					\FOR{$m=1$ {\bfseries to} $M$}
					\STATE $x_{r,0}^m \gets x_{r-1}$
					\FOR{$k=0$ {\bfseries to} $K-1$}
					\STATE Sample $z_{r,k}^m \overset{\mathrm{i.i.d.}}{\sim} \mathcal D_m$
					\STATE $x_{r,k+1}^m \gets x_{r,k}^m - \eta \nabla f(x_{r,k}^m; z_{r,k}^m)$
					\ENDFOR
					\ENDFOR
					\STATE $x_r \gets x_{r-1} + \frac{\beta}{M}\sum_{m=1}^M (x_{r,K}^m - x_{r-1})$
					\ENDFOR
				\end{algorithmic}
			\end{algorithm}
		\end{minipage}
		\hfill
		\begin{minipage}[t]{0.49\textwidth}
			\begin{algorithm}[H]
				\caption{Mini-batch SGD in the IC setting}
				\label{alg:mb-sgd}
				\begin{algorithmic}[1]
					\STATE {\bfseries Input:} communication rounds $R$, local batch size $K$, step-size $\eta$, client distributions $\cb{\ddd_1,\dots,\ddd_M}$
					\STATE {\bfseries Initialize:} global model $x_0 = 0$
					\FOR{$r=1$ {\bfseries to} $R$}
					\STATE $G_r \gets 0$
					\FOR{$m=1$ {\bfseries to} $M$}
					\FOR{$k=0$ {\bfseries to} $K-1$}
					\STATE Sample $z_{r,k}^m \overset{\mathrm{i.i.d.}}{\sim}  \mathcal D_m$
					\STATE $g_{r,k}^m \gets \nabla f(x_{r-1}; z_{r,k}^m)$
					\STATE $G_r \gets G_r + g_{r,k}^m$
					\ENDFOR
					\ENDFOR
					\STATE $x_r \gets x_{r-1} - \frac{\eta}{M} G_r$
					\ENDFOR
				\end{algorithmic}
			\end{algorithm}
		\end{minipage}
		\caption{Two canonical distributed optimization baselines implemented in the intermittent communication setting: Local SGD with an inner and outer step-size and ``large'' mini-batch SGD}
	\end{figure}

	\paragraph{Mini-batch SGD.} For Mini-batch SGD (see \Cref{alg:mb-sgd}), within each communication round, all machines compute stochastic gradients at the same shared iterate (the current global model), rather than evolving local models independently. The algorithm then aggregates these gradients into a single update to the global model. Because this algorithm uses an effective batch size of $MK$, it is often called ``Large'' Mini-batch SGD~\citep{lin2018don}. This synchronized evaluation removes client drift and ensures that each round directly optimizes the global objective \(F\) (see \citet{woodworth2020minibatch} for further discussion). Specifically, for optimization problems in the class $\ppp^{H,B,\sigma}$, mini-batch SGD achieves the following convergence rate, which is tight and cannot be improved under any data-heterogeneity assumption\footnote{A folklore argument showing that mini-batch SGD is agnostic to data heterogeneity, is to put the same data distribution on each machine and allude to lower bounds in serial stochastic optimization (see \citet[Remark 10]{patel2024limits}). 
    Having said that, it is important to note that when we drop the assumption that the smoothness parameter is identical across machines, the second-order heterogeneity assumption essentially captures the gap between the worst smoothness constant across clients and the smoothness of the average objective. We discuss this nuance further in \Cref{sec:fresh_sgd_soh} while analyzing SGD with replacement.
    }.
	\begin{align}\label{eq:mbsgd_rate}
		\ee\sb{F(\hat x^{\mathrm{MB\text{-}SGD}})}-F(x^\star) \cong \frac{HB^2}{R} + \frac{\sigma B}{\sqrt{MKR}} .
	\end{align}
	The above rate also explains why Local SGD can outperform Mini-batch SGD in regimes with ample local computation and when $\sigma$ is small: Local SGD uses the $K$ gradient evaluations to move the model $K$ times, rather than to form a single averaged update direction. In particular, even with $K\to\infty$, the optimization error of Mini-batch SGD does not improve. In practice, we observe that sequential local updates often use available computation more effectively than pure mini-batching \citep{mcmahan2016communication,goyal2017accurate,lin2018don,douillard2023diloco}. This gap
	motivates the central question in the theory of Local SGD: under what notions of data heterogeneity can Local SGD provably dominate Mini-batch SGD in a non-trivial regime?

    \begin{table}[t]
		\centering
		\setlength{\tabcolsep}{7pt}
		\renewcommand{\arraystretch}{1.2}
		\begin{tabular}{p{0.23\textwidth} p{0.7\textwidth}}
			\toprule
			\rowcolor{headergray}
			\textbf{Paper (Result type)} &  \textbf{Function sub-optimality:} $\boldsymbol{\ee\sb{F(\hat x^{\textrm{Local\text{-}SGD}})} - F(x^\star) \lesssim \text{or} \gtrsim}$\arraybackslash\\
			\midrule
			
			\rowcolor{sectiongray}
			\multicolumn{2}{l}{\textbf{Problem class $\boldsymbol{\ppp^{H, B, \sigma}_{\zeta_\star}}$:} \textit{\Cref{ass:zeta_optimal}}}\\
			\midrule
			
			\multirow{2}{*}{\begin{minipage}[t]{\linewidth}
					\citet{koloskova2020unified}\\
					(Upper bound)
			\end{minipage}}
			& \multirow{2}{*}{$\frac{HB^2}{R} + \frac{(H\sigma^2B^4)^{1/3}}{K^{1/3}R^{2/3}} 
				+ \frac{\sigma B}{\sqrt{MKR}} +  \frac{(H\zeta_\star^2B^4)^{1/3}}{R^{2/3}}$ }\\
			&\\
			\midrule
			
			\multirow{2}{*}{\begin{minipage}[t]{\linewidth}
					\citet{patel2024limits}\\
					(Lower Bound)
			\end{minipage}}
			& \multirow{2}{*}{$\frac{HB^2}{R} + \frac{(H\sigma^2B^4)^{1/3}}{K^{1/3}R^{2/3}} 
				+ \frac{\sigma B}{\sqrt{MKR}} +  \frac{(H\zeta_\star^2B^4)^{1/3}}{R^{2/3}}$ }\\
			&\\
			\midrule
			
			\rowcolor{sectiongray}
			\multicolumn{2}{l}{\textbf{Problem class $\boldsymbol{\ppp^{H, B, \sigma}_{\zeta_\star, \tau}}$:} \textit{Assumptions \ref{ass:zeta_optimal} and \ref{ass:tau} }}\\
			\midrule

			\multirow{2}{*}{\begin{minipage}[t]{\linewidth}
					\textbf{\Cref{thm:lsgd_consensus_informal} (Our)}\\
					(Upper Bound)
			\end{minipage}}
			& \multirow{2}{*}{$\frac{H B^2}{K R}+  \frac{(H\tau^2)^{1/3}B^2}{R^{2/3}} +  \frac{(H\sigma^2B^4)^{1/3}}{K^{1/3} R^{2/3}} + \frac{\sigma B}{\sqrt{M K R}}+  \frac{(H\zeta_\star^2B^4)^{1/3}}{R^{2/3}}$}\\
			&\\
			\midrule
			
			\multirow{2}{*}{\begin{minipage}[t]{\linewidth}
					\textbf{\Cref{thm:lsgd_singlemachine_informal} (Our)}\\
					(Upper Bound)
			\end{minipage}}
			& \multirow{2}{*}{$\frac{H B^2}{K R}+ \frac{\tau B^2}{\sqrt{R}} +  \frac{\sigma B}{\sqrt{KR}} +  \frac{\zeta_\star B}{\sqrt{R}}$}\\
			&\\
			\midrule

			\multirow{2}{*}{\begin{minipage}[t]{\linewidth}
					\textbf{\Cref{thm:lsgd_lb_informal} (Our)$^*$}\\
					(Lower Bound)
			\end{minipage}}
			& \multirow{2}{*}{$\frac{HB^2}{KR}
				+\min \left\{
				\frac{HB^2}{R},
				\;
				\frac{\tau B^2}{\sqrt R}
				\right\}
				+
				\min\left\{
				\frac{\sigma B}{\sqrt{KR}},
				\frac{(H\sigma^2B^4)^{1/3}}{K^{1/3}R^{2/3}}
				\right\}
				+ 
				\frac{\sigma B}{\sqrt{MKR}}
				$}\\
			&\\
			& \multirow{2}{*}{$
				\hspace{50mm} +\ 
				\min\left\{HB^2,
				\frac{\zeta_\star^2}{H},
				\frac{(H\zeta_\star^2B^4)^{1/3}}{R^{2/3}}
				\right\}
				$}\\
			&\\
			\midrule

			\multirow{2}{*}{\begin{minipage}[t]{\linewidth}
					\citet{patel2025revisiting}\\
					(Lower Bound)
			\end{minipage}}
			& \multirow{2}{*}{$\frac{H B^2}{K R} + \frac{\tau B^2}{R} + \min\left\{ \frac{\sigma B}{\sqrt{K R}}, \frac{(H\sigma^2B^4)^{1/3}}{K^{1/3} R^{2/3}} \right\} + \frac{\sigma B}{\sqrt{M K R}}+  \frac{(\tau\zeta_\star^2B^4)^{1/3}}{R^{2/3}} $}\\
			&\\
			\midrule

			\rowcolor{sectiongray}
			\multicolumn{2}{l}{\textbf{Problem class $\boldsymbol{\ppp^{H, B, \sigma}_{\zeta}}$:} \textit{\Cref{ass:zeta_everywhere}}}\\
			\midrule
			
			\multirow{3}{*}{\begin{minipage}[t]{\linewidth}
					\citet{woodworth2020minibatch,luo2025revisiting}\\
					(Upper bound)
			\end{minipage}}
			& \multirow{3}{*}{$\frac{HB^2}{KR} + \frac{(H\sigma^2B^4)^{1/3}}{K^{1/3}R^{2/3}} + \frac{\sigma B}{\sqrt{MKR}} + \textcolor{Black}{\frac{(H\zeta^2B^4)^{1/3}}{R^{2/3}}} $}\\
			&\\
			&\\
			\midrule
			
			\multirow{2}{*}{\begin{minipage}[t]{\linewidth}
					\textbf{\Cref{thm:lsgd_zeta_informal} (Our)}\\
					(Upper Bound)
			\end{minipage}}
			& \multirow{2}{*}{$\frac{H B^2}{K R}+ \frac{\sigma B}{\sqrt{KR}} +  \frac{\zeta B}{\sqrt{R}}$}\\
			&\\
			\midrule
			\multirow{2}{*}{\begin{minipage}[t]{\linewidth}
					\textbf{\Cref{cor:zeta_lower_bound} (Our)}\\
					(Lower Bound)
			\end{minipage}}
			& \multirow{2}{*}{$\frac{HB^2}{KR}
				+
				\min\left\{
				\frac{\sigma B}{\sqrt{KR}},
				\frac{(H\sigma^2B^4)^{1/3}}{K^{1/3}R^{2/3}}
				\right\}
				+ 
				\frac{\sigma B}{\sqrt{MKR}}
                $}\\
			&\\
            & \multirow{2}{*}{$
				\hspace{50mm} +\ 
				\min\left\{HB^2,
				\frac{\zeta^2}{H},
				\frac{(H\zeta^2B^4)^{1/3}}{R^{2/3}}
				\right\}
				$}\\
			&\\
			\midrule

			\rowcolor{sectiongray}
			\multicolumn{2}{l}{\textbf{Problem class $\boldsymbol{\ppp^{H,B,\sigma}_{hom}}$:} \emph{Homogeneous setting}}\\
			\midrule
			
			\multirow{2}{*}{\begin{minipage}[t]{\linewidth}
					\citet{woodworth2020local}$^\ddagger$\\
					(Upper bound)
			\end{minipage}}
			& \multirow{2}{*}{$\frac{HB^2}{KR} + \min\cb{\frac{\sigma B}{\sqrt{KR}}, \frac{(H\sigma^2B^4)^{1/3}}{K^{1/3}R^{2/3}}} 
				+ \frac{\sigma B}{\sqrt{MKR}}$}\\
			&\\
			\midrule
			
			\multirow{2}{*}{\begin{minipage}[t]{\linewidth}
					\citet{glasgow2022sharp}\\
					(Lower Bound)
			\end{minipage}}
			& \multirow{2}{*}{$\frac{HB^2}{KR} + \min\cb{\frac{\sigma B}{\sqrt{KR}}, \frac{(H\sigma^2B^4)^{1/3}}{K^{1/3}R^{2/3}}} 
				+ \frac{\sigma B}{\sqrt{MKR}}$}\\
			&\\
			
			\bottomrule
		\end{tabular}
		\vspace{0.5em}
		\caption{Convergence guarantees of Local SGD for general convex functions across different data heterogeneity assumptions. \textbf{From top to bottom, we go from more heterogeneous problem classes to less heterogeneous ones}. 
			$^\star$The lower bound assumes $K,M\geq 2$.  $^\ddagger$This rate represents the minimum of two convergence rates demonstrated by \citet{woodworth2020local}.
		}
		\label{tab:rates}
        \vspace{-2em}
	\end{table}

	\subsection{A Hierarchy of Data Heterogeneity Assumptions}
	
	\paragraph{Homogeneous setting.} When data heterogeneity is low, the Local SGD iterates across machines remain close between communication rounds, so Local SGD incurs only a small \emph{consensus error} \citep{stich2018local,patel2025revisiting}. In the extreme homogeneous case
	$\mathcal{D}_1=\cdots=\mathcal{D}_M$, the client objectives coincide, and the local iterates on different machines evolve identically in expectation. This case, therefore, serves as a natural starting point, and much of the early work studied Local SGD under homogeneity.
	
	We denote the homogeneous problem class by
	$\mathcal{P}^{H,B,\sigma}_{\mathrm{hom}}\subseteq \mathcal{P}^{H,B,\sigma}$.
	As indicated by the bottom two rows of Table~\ref{tab:rates}, the minimax-optimal rate of Local SGD over $\mathcal{P}^{H,B,\sigma}_{\mathrm{hom}}$ is already known
	\citep{woodworth2020local,glasgow2022sharp}. However, homogeneity also hides the central difficulty in federated and distributed learning, as it ignores differences among clients. While Local SGD can outperform mini-batch SGD in this setting, this advantage arises only in regimes where single-machine SGD also performs better than mini-batch SGD~\citep{woodworth2021min}.

	\paragraph{Approximate simultaneous realizability.}
	A peculiarity of the homogeneous setting is that, since all clients share the same minimizer set, each client can, in principle, identify this set on its own in low-noise regimes. Thus, collaboration primarily helps by reducing variance. This perspective makes collaboration look rather narrow: communication primarily reduces noise, rather than enabling qualitatively better optimization under heterogeneity. One could instead ask what happens in problems where clients are allowed to have different data distributions but have at least some shared minimizers, i.e., they are simultaneously realizable~\cite{blum2017collaborative,han2023effect}? The following first-order assumption formalizes and generalizes this idea. 
	\begin{assumption}\label{ass:zeta_optimal}
		Objectives $\{F_m\}_{m\in[M]}$ satisfy $\zeta_\star$-first-order heterogeneity at the optima if there exists a minimizer of the average objective $x^\star\in \arg\min_{x\in \rr^d} F(x)$ such that $\norm{x^\star}\leq B$ and,
		\begin{align*}\frac{1}{M}\sum_{m\in[M]}\norm{\nabla F_m(x^\star)}^2 \leq \zeta_\star^2 .
		\end{align*}
	\end{assumption}
	We denote the class of problems that satisfy the above assumptions by $\ppp^{H, B, \sigma}_{\zeta_\star}$. Tight convergence rates for Local SGD over this problem class are known from the results of \citet{koloskova2020unified} and \citet{patel2024limits}, as summarized in the first two rows of Table~\ref{tab:rates}. Unfortunately, Local SGD cannot strictly dominate Mini-batch SGD for this problem class for any value of $\zeta_\star$. This indicates that stronger assumptions are needed to establish a provable advantage for Local SGD.   
	
	\paragraph{Uniformly bounded first-order heterogeneity.}
	Much of the early analysis of Local SGD considers a stronger variant of
	$\ppp^{H,B,\sigma}_{\zeta_\star}$ by requiring the local gradients to track the global gradient \emph{uniformly over the domain}~\citep{khaled2020tighter,woodworth2020minibatch,karimireddy2020scaffold}. This yields the following more stringent notion of first-order heterogeneity.
	\begin{assumption}\label{ass:zeta_everywhere}
		Objectives $\{F_m\}_{m\in[M]}$ satisfy $\zeta$-first-order heterogeneity if for all $x\in\mathbb{R}^d$,
		\begin{align*}
			\frac{1}{M}\sum_{m\in[M]}\bigl\|\nabla F_m(x)-\nabla F(x)\bigr\|_2^2 \le \zeta^2 .
		\end{align*}
	\end{assumption}
	We denote\footnote{With a slight abuse of notation, we use $\zeta_\star$ and $\zeta$ as both scalar values and to demarcate the problem classes implied by Assumptions \ref{ass:zeta_optimal} and \ref{ass:zeta_everywhere}, respectively. The usage should be clear from the context.} the corresponding problem class by $\ppp_{\zeta}^{H,B,\sigma}$.
	Under Assumption~\ref{ass:zeta_everywhere}, \citet{woodworth2020minibatch,luo2025revisiting} established regimes where vanilla Local SGD improves over mini-batch SGD (see Table~\ref{tab:rates}). Furthermore, the limit $\zeta\to 0$ smoothly recovers the homogeneous setting, making this interpolation between heterogeneity and homogeneity particularly appealing. At the same time, $\ppp_{\zeta}^{H,B,\sigma}$ is a very restricted problem class. For instance, in the case of quadratic objectives, a uniform bound as in Assumption~\ref{ass:zeta_everywhere} forces all clients to share the same Hessian (see, e.g., \citet[Remarks 2, 3]{patel2024limits}), which makes the problem essentially homogeneous.

	This motivates a natural question: is there an intermediate heterogeneity model that is weaker than Assumption~\ref{ass:zeta_everywhere} but stronger than Assumption~\ref{ass:zeta_optimal}, under which Local SGD provably outperforms mini-batch SGD while still allowing meaningful heterogeneity? Since \Cref{ass:zeta_everywhere} controls global geometric heterogeneity across clients, the intermediate heterogeneity class must also impose such global control without requiring gradients to remain uniformly close across clients.

	\paragraph{Controlling global geometry using bounded second-order heterogeneity.}
	\citet{patel2024limits} conjectured that the appropriate assumption to capture a nontrivial regime in which Local SGD can outperform mini-batch SGD is \emph{bounded second-order heterogeneity}. 
	\begin{assumption}\label{ass:tau}
		Objectives $\{F_m\}_{m\in[M]}$ satisfy $\tau$-second-order heterogeneity if, for all $x, y\in\rr^d$,
		\begin{equation}
			\frac{1}{M}\sum_{m\in[M]}\bigl\|\nabla F_m(x)-\nabla F_m(y) - \rb{\nabla F(x) - \nabla F(y)} \bigr\|_2^2 \le \tau^2\norm{x-y}^2 .\nonumber
		\end{equation}
	\end{assumption}
    We denote by $\ppp^{H,B,\sigma}_{\zeta_\star,\tau}$ the class of problems satisfying
	Assumptions~\ref{ass:zeta_optimal} and~\ref{ass:tau}.
    \begin{remark}[Bounded Hessian heterogeneity implies \Cref{ass:tau}]
\label{rem:hessian_implies_tau}
Suppose that each \(F_m\) is twice differentiable and that, for every
\(z\in\rr^d\),
\[
    \frac{1}{M}\sum_{m\in[M]}
    \bigl\|\nabla^2 F_m(z)-\nabla^2 F(z)\bigr\|_{2}^2
    \le \tau^2\enspace.
\]
Then the objectives satisfy \(\tau\)-second-order heterogeneity in the sense of
\Cref{ass:tau}. Indeed, for any \(x,y\in\rr^d\), letting \(v=x-y\), the
fundamental theorem of calculus gives
\[
    \nabla F_m(x)-\nabla F_m(y)
    -
    \bigl(\nabla F(x)-\nabla F(y)\bigr)
    =
    \int_0^1
    \bigl(\nabla^2 F_m(y+tv)-\nabla^2 F(y+tv)\bigr)v\,dt\enspace.
\]
Therefore, by Jensen's inequality and the assumed Hessian heterogeneity bound,
\[
\begin{aligned}
    &\frac{1}{M}\sum_{m\in[M]}
    \bigl\|
    \nabla F_m(x)-\nabla F_m(y)
    -
    \bigl(\nabla F(x)-\nabla F(y)\bigr)
    \bigr\|_2^2 \\
    &\qquad\le
    \int_0^1
    \frac{1}{M}\sum_{m\in[M]}
    \bigl\|\nabla^2 F_m(y+tv)-\nabla^2 F(y+tv)\bigr\|_{2}^2
    \norm{x-y}^2\,dt
    \le
    \tau^2\norm{x-y}^2\enspace.
\end{aligned}
\]
This is why we refer to \Cref{ass:tau} as a second-order heterogeneity
assumption.
\end{remark}
	 \begin{remark}[$\ppp^{H,B,\sigma}_{\zeta_\star,\tau}$ interpolates between problem classes]\label{rem:interpolation}
\label{rem:tau_interpolation}
\Cref{ass:tau} controls how different the ``second-order'' geometry can be
across clients, as opposed to the ``first-order'' geometry controlled by
\Cref{ass:zeta_everywhere}. It also interpolates between the problem classes
discussed above. At one extreme, if each \(F_m\) is \(H\)-smooth, then \(F\) is
also \(H\)-smooth, and
\[
    \bigl\|
    \nabla F_m(x)-\nabla F_m(y)
    -
    \bigl(\nabla F(x)-\nabla F(y)\bigr)
    \bigr\|_2
    \le 2H\norm{x-y}\enspace.
\]
Thus, \Cref{ass:tau} is vacuous when \(\tau=2H\), and
\(
    \boldsymbol{\ppp^{H,B,\sigma}_{\zeta_\star,\tau=2H}
    \subseteq
    \ppp^{H,B,\sigma}_{\zeta_\star}}
\). 

At the other extreme, when \(\tau=0\), the gradient gaps
\(\nabla F_m(x)-\nabla F(x)\) are constant in \(x\) implying,
\[
    \nabla F_m(x)-\nabla F(x)
    =
    \nabla F_m(x^\star)-\nabla F(x^\star)
    =
    \nabla F_m(x^\star)\enspace,
\]
and therefore
\[
    \frac{1}{M}\sum_{m\in[M]}
    \bigl\|\nabla F_m(x)-\nabla F(x)\bigr\|_2^2
    =
    \frac{1}{M}\sum_{m\in[M]}
    \bigl\|\nabla F_m(x^\star)\bigr\|_2^2
    \le \zeta_\star^2\enspace.
\]
Thus when \(\tau=0\), we have
\[
    \boldsymbol{\ppp^{H,B,\sigma}_{\zeta_\star,\tau=0}
    \subseteq
    \ppp^{H,B,\sigma}_{\zeta=\zeta_\star}}.
\]
This is precisely why $\ppp^{H,B,\sigma}_{\zeta_\star,\tau}$ sits between the two problem classes $\ppp^{H,B,\sigma}_{\zeta_\star}$ and $\ppp^{H,B,\sigma}_{\zeta}$ in \Cref{tab:rates}. 
\end{remark}

    The conjecture of \citet{patel2024limits} states that, over the problem class $\ppp^{H,B,\sigma}_{\zeta_\star,\tau}$,
	Local SGD should provably outperform mini-batch SGD while still allowing meaningful heterogeneity. Their conjecture is supported by three sources of evidence: (i) their own analysis of Local SGD under second-order heterogeneity together with the restrictive Assumption~\ref{ass:zeta_everywhere}; (ii) near-optimal analyses of related local-update methods in nonconvex settings \citep{patel2022towards}; and (iii) an independent line of work on distributed proximal methods (for example, DANE and its variants \citep{shamir2014dane, sun2022distributed, kovalev2022optimal, jiang2024fedred, jiang2024dane}) whose communication complexity improves under the same Hessian-similarity condition.
	
	Recently, \citet{patel2025revisiting} verified this conjecture in the strongly convex subclass of $\ppp^{H,B,\sigma}_{\zeta_\star,\tau}$, but their result does not extend to general convex objectives via a black-box convex-to-strongly-convex reduction as their guarantees impose restrictions on the strong convexity parameter.
	At the same time, they established a lower bound for the general convex class $\ppp^{H,B,\sigma}_{\zeta_\star,\tau}$ that improves with $\tau$ and $\zeta_\star$ and interpolates to the homogeneous setting as these parameters vanish.
	In the next section, we show that our new upper bound for Local SGD on $\ppp^{H,B,\sigma}_{\zeta_\star,\tau}$ matches this lower bound (up to constants) in the regime of small $\tau$, while retaining the desirable interpolation properties that previous upper bounds achieved only on the much smaller problem class $\ppp^{H,B,\sigma}_{\zeta}$ (see Table~\ref{tab:rates}).

    \begin{figure}
		\centering
		\includegraphics[width=0.6\linewidth]{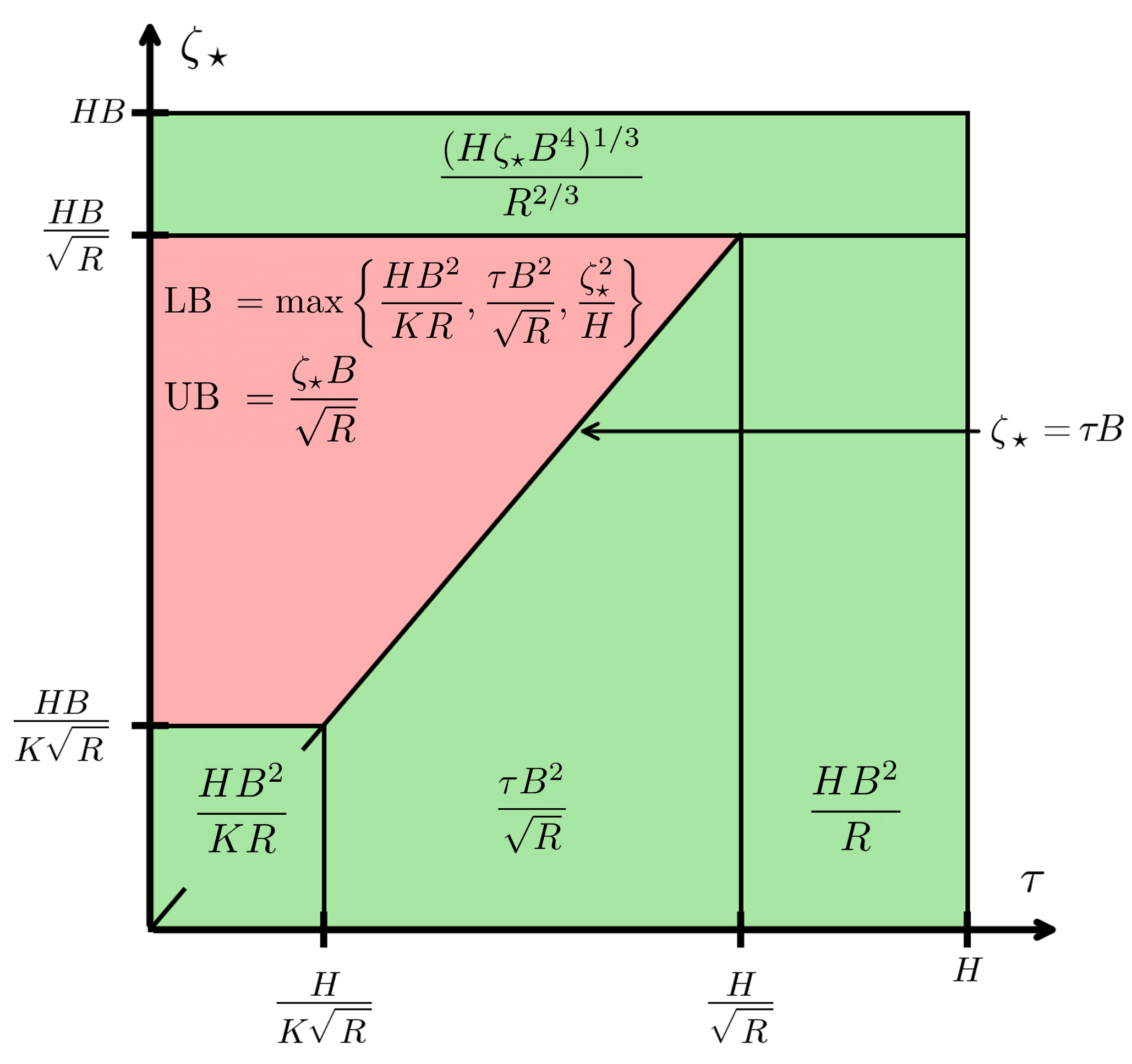}
 		\caption{The tightness of our upper and lower bounds, highlighted in the setting with $\sigma=0$ and $\zeta_\star~\leq~HB$ on the $\tau$--$\zeta_\star$ plane. The green areas indicate regimes where the upper and lower bounds match, along with the dominant term in each, while the red area indicates the regime where the bounds differ. All markers on the axes are up to numerical constants. See \Cref{app:phase} for details. 
        }
		\label{fig:lower_vs_upper}
	\end{figure}

	\section{Improved Analyses of Local SGD for Classes $\ppp^{H,B,\sigma}_{\zeta_\star,\tau}$ 
    and $\ppp^{H,B,\sigma}_{\zeta}$}\label{sec:upper}

	In this section, we will prove three upper bounds for Local SGD. The first analyzes
	perturbations against dense mini-batch SGD~\citep{woodworth2020local}, i.e.,
	$KR$ steps of mini-batch SGD (as opposed to $R$ steps). The second analyzes perturbations to $KR$ steps of SGD on a single machine under the additional assumption of uniform second-order heterogeneity. The third gives a complementary single-machine-style analysis under uniformly
	bounded first-order heterogeneity. We will denote by $\bar x_t \coloneqq \tfrac1M\sum_{m=1}^M x_t^m$ the average iterate for all time
	steps $t\in[0,T]$. This ``ghost sequence'' is useful for analysis even at the
	time steps where the average is not actually computed.

	\subsection{Perturbing Against Dense Mini-batch SGD}
	
	\paragraph{Step 1: A canonical recursion that isolates consensus error.}
	We begin with the standard smoothness-based one-step recursion for the average iterate, which appears in several analyses of Local SGD~\citep{stich2018local,woodworth2020local,woodworth2020minibatch}. This inequality mirrors the descent lemma for dense mini-batch SGD, with an additional term capturing the deviation of local iterates from their average. As a result, it decomposes the error into three components: optimization error, stochastic noise, and \emph{consensus error}, which measures the client drift due to the lack of complete synchrony.
	
	\begin{lemma}[Average iterate recursion with consensus error]
		\label{lem:avg-recursion}
		Consider the problems in $\ppp^{H,B,\sigma}$. If we run vanilla Local SGD with constant step-size $\eta\le 1/(10H)$, then for every $t\in[0,T-1]$,
		\begin{align}\label{lem:woodworth_descent_lemma} 
			\ee[F(\bar{x}_t)-F^\star] 
			\le \frac{\ee[\|\bar{x}_t-x^\star\|^2_2]-\ee[\|\bar{x}_{t+1}-x^\star\|^2_2]}{\eta} 
			+ \frac{3\eta\sigma^2}{M} 
			+ \frac{2H}{M}\sum_{m=1}^M\ee[\|\bar{x}_t-x_t^m\|^2_2] .
		\end{align}
	\end{lemma}
	
	The consensus error $C(t)\coloneqq \frac{1}{M}\sum_{m=1}^M\ee[\|\bar x_t-x_t^m\|_2^2]$ is the main difficulty in heterogeneous settings. In prior works \citep{woodworth2020minibatch,luo2025revisiting}, the consensus error is directly bound using the uniform heterogeneity Assumption~\ref{ass:zeta_everywhere}. Our goal is to replace this $\zeta$-based control with a bound relying on $(\zeta_\star,\tau)$. A key observation is that the same recursion also controls the expected iterate error of the average trajectory. Specifically, denote $A_{\max}\coloneqq \max_{t\in[0,T]} \ee[\|\bar{x}_t-x^\star\|^2_2]$. By unrolling recursion \eqref{lem:woodworth_descent_lemma}, one can show
	\begin{equation}\label{eq:A_max_bound} 
		A_{\max}
		\le
		B^2+\frac{3\eta^2\sigma^2 T}{M}+2H\eta TC_{\max} ,\qquad\text{where} \qquad C_{\max} \coloneqq  \max_{s\in[0,T]}C(s) .
	\end{equation} 
	Thus, controlling the consensus error directly yields control over the expected iterate error of the average trajectory. Further details of this derivation are provided in \Cref{lem:amax_ub} in \Cref{app:new_technical_lemmas}.

	\paragraph{Step 2: Controlling consensus error via trajectory-wise gradient dissimilarity.}
	Our next step is to relate the consensus error to the gradient heterogeneity \emph{along the trajectory} of the average iterates, rather than uniformly over the entire domain. In particular, denote $V_{\max} \coloneqq \max_{s\in[0,T]} \frac{1}{M}\sum_{m\in[M]}\ee[\|\nabla F_m(\bar x_s)-\nabla F(\bar x_s)\|^2_2]$.
	Then in \Cref{lem:consensus_bound_V_max} in \Cref{app:new_technical_lemmas} we show that 
	\begin{equation}\label{eq:Ct_bound} 
		C_{\max} 
		\;\lesssim\; \eta^2 K^2 V_{\max} + \eta^2 K \sigma^2 .
	\end{equation}
	
	This bound shows that the consensus error is governed by the \emph{largest gradient dissimilarity encountered along the average trajectory}. In contrast to prior analyses (e.g., \citet{woodworth2020minibatch}), which directly upper-bound $V_{\max}$ using the uniform Assumption~\ref{ass:zeta_everywhere}, we defer control of $V_{\max}$ to the next step, which uses a trajectory-dependent control, thus avoiding the dependence on $\zeta$.

	\paragraph{Step 3: Controlling $V_{\max}$ using the iterate error.} We bound $V_{\max}$ using Assumptions \ref{ass:zeta_optimal} and \ref{ass:tau} (see \Cref{lem:vmax_ub} in \Cref{app:new_technical_lemmas}):
	\begin{equation}\label{eq:V_max_bound} 
		V_{\max} \le 2\tau^2 A_{\max} + 2\zeta_\star^2 .
	\end{equation}

	\paragraph{Step 4: Self-bounding loop and the main theorem.}
	We now close the argument by combining the bounds from the previous steps. 
	Equation~\eqref{eq:A_max_bound} shows that $A_{\max}$ is controlled by the largest consensus error $C_{\max}$. On the other hand, combining equations \eqref{eq:Ct_bound} with \eqref{eq:V_max_bound} shows that the largest consensus error itself can be bounded in terms of $A_{\max}$. Together, these relations form a self-bounding loop:
	\[ 
	A_{\max} \;\longrightarrow\; C_ {\max} \;\longrightarrow\; V_{\max} \;\longrightarrow\; A_{\max} .
	\]
	Under a sufficiently small step-size $\eta$, this recursion can be closed, allowing us to simultaneously bound $A_{\max}$ and $C_ {\max}$ purely in terms of $(\zeta_\star,\tau)$ (see Lemmas~\ref{lem:amax_ub} and \ref{lem:cmax_ub} in \Cref{app:new_technical_lemmas} for details).
	
	This self-bounding argument is the key technical ingredient of the proof: it replaces the need for \Cref{ass:zeta_everywhere} by controlling the entire trajectory of Local SGD through its own dynamics. Substituting the resulting consensus bound into Lemma~\ref{lem:avg-recursion} yields our new upper bound (see \Cref{app:upper_bounds}).
	
	\begin{theorem}[Informal version of \Cref{thm:lsgd_consensus}]\label{thm:lsgd_consensus_informal}
		For problems in the class $\ppp_{\zeta_\star, \tau}^{H, B, \sigma}$, using an appropriate step-size, the Local SGD output $\hat x = \frac{1}{T}\sum_{t=0}^{T-1}\bar{x}_t$ satisfies,
		\begin{align*}
			\ee\sb{F(\hat x) - F(x^\star)}\lesssim \frac{HB^2}{KR} + \frac{(H\tau^2)^{1/3}B^2}{R^{2/3}} + \frac{\sigma B}{\sqrt{MKR}} + \frac{(H\zeta_\star^2 B^4)^{1/3}}{R^{2/3}} + \frac{(H\sigma^2 B^4)^{1/3}}{K^{1/3}R^{2/3}} .
		\end{align*}
	\end{theorem}
	Comparing the above theorem to that of \citet{woodworth2020minibatch} for the class $\ppp_{\zeta_\star}^{H, B, \sigma}$, we note that replacing $\zeta$ by $\zeta_\star + \tau B$ recovers our theorem (up to numerical constants). This is not a mere coincidence, and in fact, for the step-size used in \Cref{thm:lsgd_consensus_informal} we show that $A_ {\max}\lesssim B^2$ (see Lemmas \ref{lem:amax_ub} and \ref{lem:cmax_ub}), which allows us to bound the gradient heterogeneity on the trajectory of Local SGD, i.e., $V_ {\max} \lesssim \tau^2B^2 + \zeta^2_\star$. Notably, this connects our proof to the insight of \citet[Proposition 13]{patel2024limits} and the guarantees of \citet{patel2025revisiting} which rely on a similar idea, albeit with strong convexity.    
	
	\begin{remark}[Beating Mini-batch SGD]
	    Comparing our bound in \Cref{thm:lsgd_consensus_informal} with the mini-batch SGD rate in \eqref{eq:mbsgd_rate}, we see that Local SGD improves over mini-batch SGD whenever
	\[
	K\gg 1,
	\qquad
	\tau \lesssim \frac{H}{\sqrt R},
	\qquad
	\zeta_\star \lesssim \frac{HB}{\sqrt R},
	\qquad
	\sigma \lesssim HB\sqrt{\frac KR} .
	\]
	Compared with the analysis under \Cref{ass:zeta_everywhere} by \citet{woodworth2020minibatch}, this is a significantly larger regime, highlighting the role of second-order heterogeneity (cf. \citet[Theorem 15]{patel2024limits}).
	\end{remark}
	
	\subsection{Perturbing Against Single-Machine SGD}
	
	We now proceed to prove our second new upper bound under the second-order heterogeneity \Cref{ass:tau}. Our starting point is the observation that, even in the homogeneous setting, the optimal upper bound for Local SGD (see \Cref{tab:rates}) due to \citet{woodworth2020local} requires two analyses, and specifically one of them shows that Local SGD is never worse than single-machine SGD. In the following theorem, we extend this analysis from the class $\ppp_{hom}^{H, B, \sigma}$ to $\ppp_{\zeta_\star, \tau}^{H, B, \sigma}$. 
	
	\begin{theorem}[Informal version of \Cref{thm:lsgd_singlemachine_new}]\label{thm:lsgd_singlemachine_informal}
		For problems in the class $\ppp_{\zeta_\star, \tau}^{H, B, \sigma}$, using an appropriate step-size, the Local SGD output $\hat x = \frac{1}{T}\sum_{t=0}^{T-1}\bar{x}_t$ satisfies,
		\begin{align*}
			\ee\sb{F(\hat x) - F(x^\star)}\lesssim \frac{HB^2}{KR} + \frac{\sigma B}{\sqrt{KR}} + \frac{\tau B^2}{\sqrt{R}}  + \frac{\zeta_\star B}{\sqrt{R}}.
		\end{align*}
	\end{theorem} 
	
	Note that when $\tau$ and $\zeta_\star$ are zero, the above theorem recovers the convergence rate of single-machine SGD, highlighting that the last two terms indeed capture the cost of having different data distributions across machines. The proof of the above theorem relies on a fundamentally different one-step recursion (see \Cref{lem:descent_Bt_Dt} in \Cref{app:new_technical_lemmas}), but also uses a self-bounding argument that couples consensus error with the following alternative iterate error (Lemmas~\ref{lem:bound_C_max_through_B_max} and \ref{lem:bound_B_max_V_max_C_max_through_B2} in \Cref{app:new_technical_lemmas}),
	$$ 
	B_{\max} \coloneqq \max_{0\le s\le T}\frac{1}{M}\sum_{m\in[M]}\ee\sb{\norm{x_s^m - x^\star}^2} \ge A_{\max} .
	$$
	We defer the proof to the appendix, as it is qualitatively similar to the proof of \Cref{thm:lsgd_consensus_informal}. While \Cref{thm:lsgd_singlemachine_informal} is crucial to characterize the tight rates for local SGD (see \Cref{sec:lower}), it does not characterize a wider regime than \Cref{thm:lsgd_consensus_informal} where local SGD beats mini-batch SGD. 

    \subsection{An Improved Analysis under \Cref{ass:zeta_everywhere}}
    
	Finally, we also prove an improved upper bound (in \Cref{app:upper_bound_zeta}), for the class $\ppp_{\zeta}^{H, B, \sigma}$ which relies on perturbing against the single machine SGD analysis, as in previous subsection, and thus compliments existing analyses under \Cref{ass:zeta_everywhere} (see \Cref{tab:rates}). 
	
	\begin{theorem}[Informal version of \Cref{thm:lsgd_zeta}]\label{thm:lsgd_zeta_informal}
		For problems in the class $\ppp_{\zeta}^{H, B, \sigma}$, using an appropriate step-size, the Local SGD output $\hat x = \frac{1}{T}\sum_{t=0}^{T-1}\bar{x}_t$ satisfies,
		\begin{align*}
			\ee\sb{F(\hat x) - F(x^\star)}\lesssim \frac{HB^2}{KR} + \frac{\sigma B}{\sqrt{KR}} + \frac{\zeta B}{\sqrt{R}}  .
		\end{align*}
	\end{theorem} 

	Comparing the above rate with the prior guarantee of
	\citet{woodworth2020minibatch}, together with the refinement of
	\citet{luo2025revisiting} under \Cref{ass:tau}
	(see \Cref{tab:rates}), observe that
	\[
	\frac{\zeta B}{\sqrt{R}}
	\lesssim
	\frac{(H\zeta^2B^4)^{1/3}}{R^{2/3}}
	\qquad\text{whenever}\qquad
	\zeta \lesssim \frac{HB}{\sqrt{R}} ,
	\]
	and
	\[
	\frac{\sigma B}{\sqrt{KR}}
	\lesssim
	\frac{(H\sigma^2B^4)^{1/3}}{K^{1/3}R^{2/3}}
	\qquad\text{whenever}\qquad
	\sigma \lesssim HB\sqrt{\frac{K}{R}} .
	\]
	Consequently, in this regime, our upper bound is no larger, up to universal numerical constants, than the prior upper bound over $\ppp_{\zeta}^{H,B,\sigma}$, and gives a sharper dependence on the heterogeneity and stochastic-noise parameters. Comparing \Cref{thm:lsgd_zeta_informal} with \Cref{thm:lsgd_singlemachine_informal}, we note that the latter guarantee can essentially be obtained by replacing $\zeta$ in the former guarantee by $\zeta_\star + \tau B$, which mirrors the relationship between \Cref{thm:lsgd_consensus_informal} and the exiting convergence analysis under \Cref{ass:zeta_everywhere} (see the rate due to \citet{woodworth2020minibatch,luo2025revisiting} in \Cref{tab:rates}).

	\section{A New Lower Bound and the Tightness of Our Results}\label{sec:lower}
	In this section, we prove a new lower bound for Local SGD, which builds on existing lower bounds by \citet{glasgow2022sharp} and \citet{patel2025revisiting}. The lower bound is crucial for decoupling the effects of first- and second-order heterogeneity and relies on two orthogonal mechanisms of client drift.

	\begin{theorem}[Informal version of Theorems \ref{thm:tau_sqrtR_lower_bound} and \ref{thm:tau-zero-lower}]\label{thm:lsgd_lb_informal}
		If $K, M\geq 2$, then there exists a problem in the class $\ppp_{\zeta_\star, \tau}^{H, B, \sigma}$, such that for any step-size $\eta$, the final Local SGD iterate\footnote{We state the lower bound for the final Local SGD iterate, while the upper bound is for the averaged iterate. This gap between the upper and lower bounds is common in this literature~\citep{woodworth2020local,woodworth2020minibatch,glasgow2022sharp,patel2022towards,patel2024limits,patel2025revisiting}. The lower bound can be repeatedly applied across all time steps, with additional logarithmic factors, but we omit these technical details for brevity.} $\bar x_{KR}$ must have,
		\begin{equation}
			\ee\sb{F(\bar x_{KR})}- F(x^\star) \gtrsim \min \left\{
			\frac{HB^2}{R},
			\;
			\frac{\tau B^2}{\sqrt R}
			\right\}
			+ 
			\min\left\{HB^2,
			\frac{\zeta_\star^2}{H},
			\frac{(H\zeta_\star^2B^4)^{1/3}}{R^{2/3}}
			\right\} .\nonumber
		\end{equation}
	\end{theorem} 
	To obtain the final lower bound in \Cref{tab:rates}, we combine our construction with the homogeneous lower bound of \citet{glasgow2022sharp} by placing their hard instance on disjoint coordinates. Since this instance has $\zeta_\star=\tau=0$, it satisfies all heterogeneity assumptions and does not restrict our parameters. We will also use two distinct hard instances to obtain the two terms in the lower bound. Note that we lose only numerical constants by using extra coordinates for different hard instances. We next discuss the key ideas in proving both components of the lower bound.

	\paragraph{Second-order lower bound.} Our construction (see \Cref{thm:tau_sqrtR_lower_bound}) builds on \citet[Theorem 1]{patel2025revisiting}: two clients share an optimizer but recover it only on their respective subspaces in isolation, while together they recover it globally. By tuning the angle between these subspaces, we induce drift across rounds so that the joint dynamics reduce to $R$ steps of gradient descent on the average objective with an effective step size depending on $K$ and $\eta$, yielding a lower bound that does not improve with $K$. The key improvement in our analysis is that we decouple the ``scale'' of the client objectives from the size of their second-order heterogeneity, which allows us to obtain a sharper tuning of our hard instance, leading to the improved lower bound versus the $\frac{\tau B^2}{R}$ lower bound of \citet{patel2025revisiting}.

	\paragraph{First-order lower bound.} This construction (see \Cref{thm:tau-zero-lower}) uses three coordinates to separate optimization and heterogeneity effects. On the active coordinate, we define a convex $H$-smooth function with a curvature change at the origin,
	\[
	g(s)=\mathbf{1}_{s<0}\cdot \tfrac{H}{4}s^2 + \mathbf{1}_{s\ge 0}\cdot \tfrac{H}{2}s^2 ,
	\]
	and add opposite linear shifts $\pm \zeta_\star s$ across clients. This yields $\zeta_\star$-heterogeneity at the optimum while keeping $\tau=0$. Local updates then place clients in different curvature regimes, inducing a persistent bias of order $\zeta_\star/H$ per round, which leads to the $R^{-2/3}$ rate after balancing step-size effects. Crucially, this mechanism relies on the
	fact that $g$ is \emph{not twice differentiable} at the origin: if the function
	were $C^2$ with identical Hessians across clients, then the local dynamics would linearize globally, and this asymmetric drift would disappear, collapsing to a homogeneous-like behavior. The other two coordinates control extreme step sizes and prevent trivial convergence.

	\paragraph{Min-max optimal analysis of Local SGD.} 
	We can now compare our lower bound in \Cref{thm:lsgd_lb_informal} along with the homogeneous lower bound of \citet{glasgow2022sharp} to the best of three upper bounds: (i) from \citet{koloskova2020unified} (see \Cref{tab:rates}); (ii) our consensus-error based upper bound in \Cref{thm:lsgd_consensus_informal}; and (iii) our single-machine mimicking upper bound in \Cref{thm:lsgd_singlemachine_informal}. For simplicity, we look at the problem class $\ppp_{\zeta_\star, \tau}^{H,B,0}$ in the noiseless setting. We summarize the conclusion of this exercise (see calculations in \Cref{app:phase}) in \Cref{fig:lower_vs_upper} and note that, in most regimes, we characterize the min-max optimal convergence guarantee for Local SGD.\footnote{We compare previous upper bound of \citet{koloskova2020unified} and the lower bound of \citet{patel2025revisiting} in \Cref{fig:lower_patel_vs_upper_koloskova} and show that no regime exists in the $\zeta_\star$–$\tau$ plane where they match, highlighting the necessity of a more refined analysis.} We believe new ideas are needed to close the gap in the red regime of \Cref{fig:lower_vs_upper}.

    \paragraph{Comparison to Existing Upper and Lower Bounds.} In \Cref{fig:lower_patel_vs_upper_koloskova}, we provide illustrative phase diagrams of both the lower bound of \citet{patel2025revisiting} and the upper bound of \citet{koloskova2020unified}. We observe that there is no regime where the rates match, demonstrating the looseness of prior analyses for Local SGD over the class $\ppp_{\zeta_\star, \tau}^{H, B, \sigma}$. We derive both of these phase diagrams in \Cref{app:phase_diagrams_prior_works}.
	
	\begin{figure}[!ht]
		\centering
		\begin{tabular}{cc}
			\includegraphics[width=0.45\linewidth]{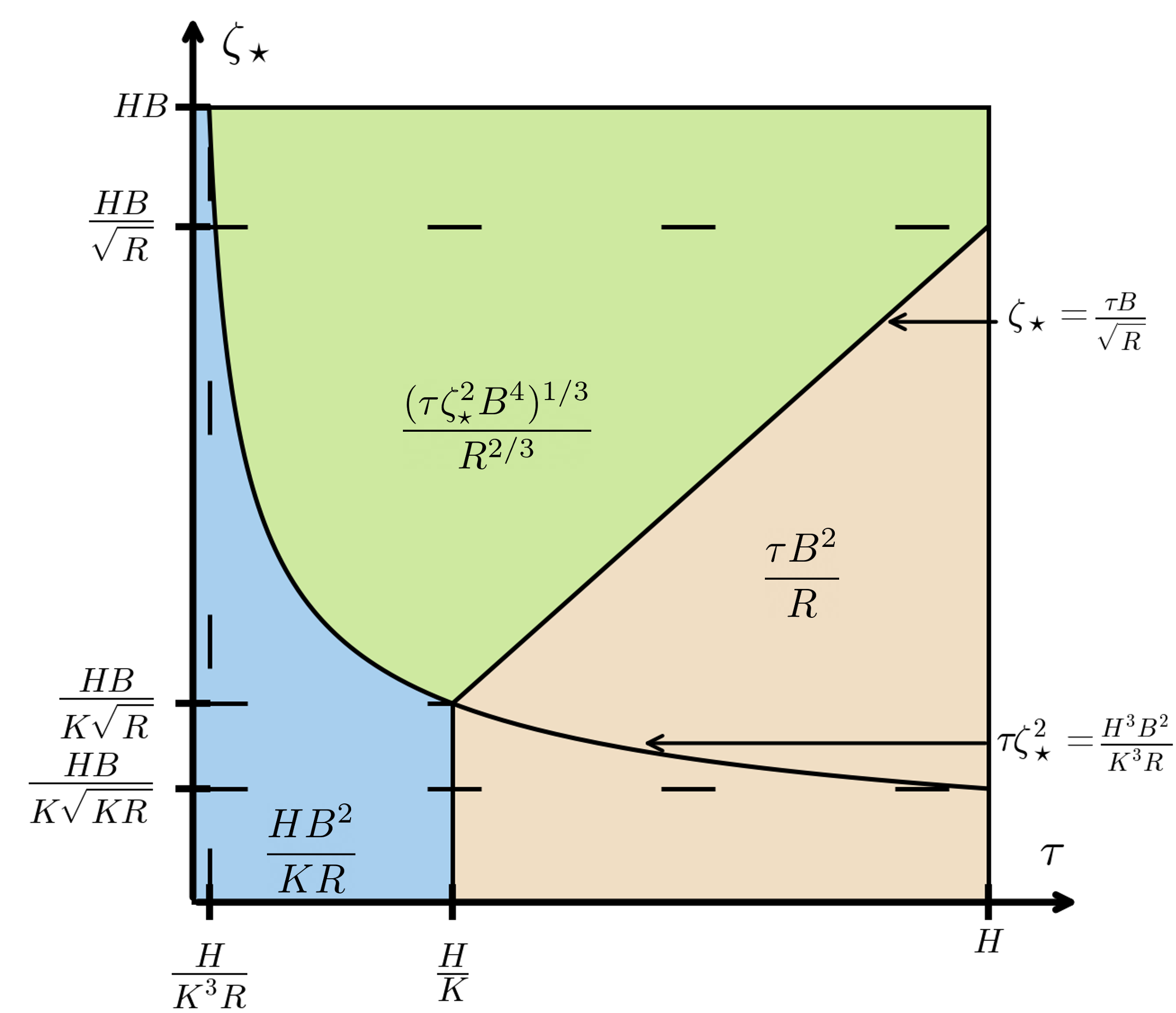}
			
			&
			
			\includegraphics[width=0.40\linewidth]{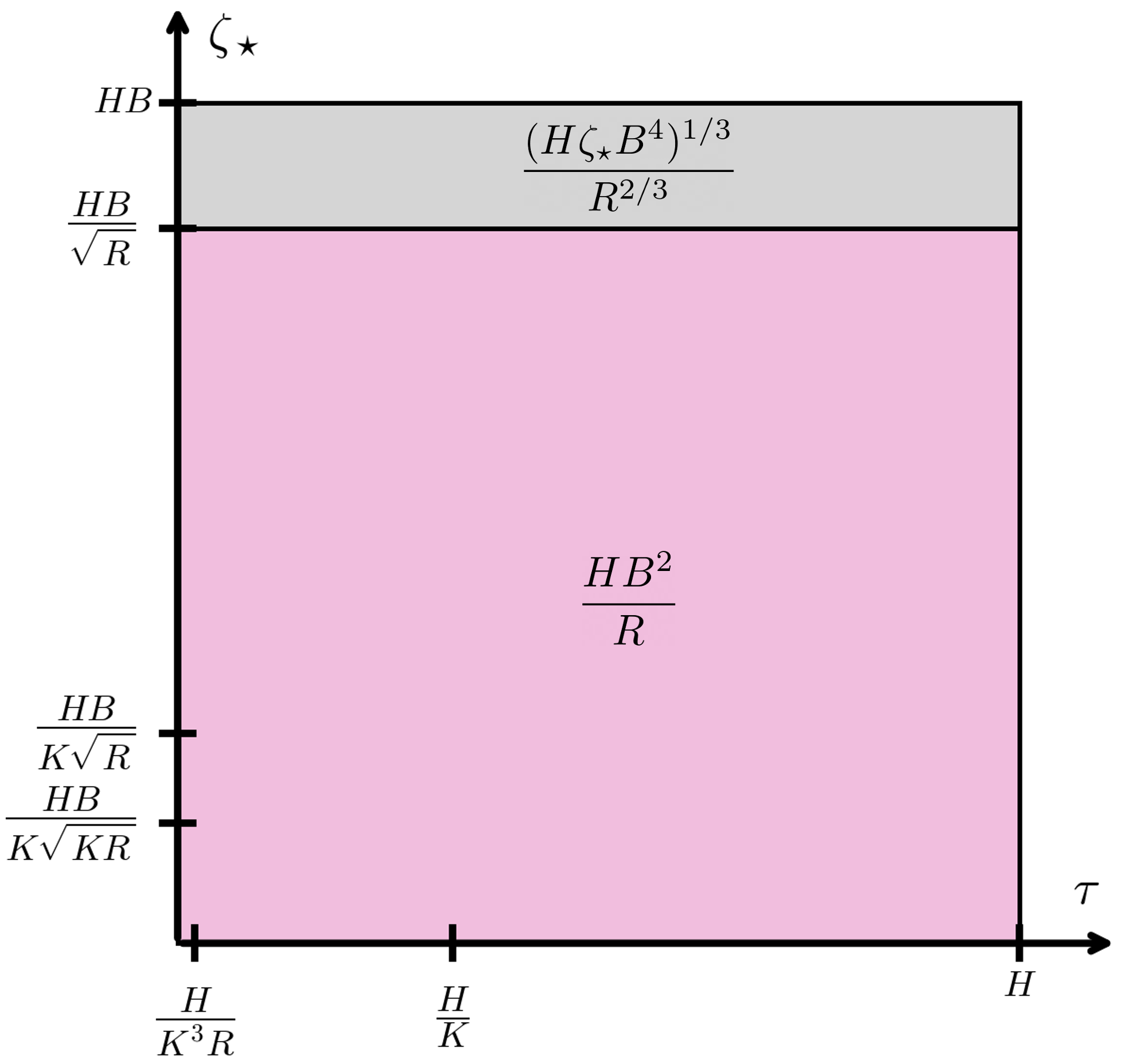}
			
		\end{tabular}
		\caption{Phase diagrams of the lower bound of \citet{patel2025revisiting} (left) and the upper bound of \citet{koloskova2020unified} (right): see \Cref{tab:rates} for the exact rates and \Cref{sec:upper_bound_koloskova_diagram} and \ref{sec:lower_bound_patel_diagram} for derivations of the diagrams.
		}
		\label{fig:lower_patel_vs_upper_koloskova}
	\end{figure}
	

\subsection{A New Lower Bound for $\ppp_{\zeta}^{H,B,\sigma}$}

Although \Cref{thm:lsgd_lb_informal} is stated for the class
$\ppp_{\zeta_\star,\tau}^{H,B,\sigma}$, it also implies a lower bound for the
more restrictive uniformly bounded first-order heterogeneity class
$\ppp_{\zeta}^{H,B,\sigma}$. The key observation is the inclusion
\[
    \ppp^{H,B,\sigma}_{\zeta_\star=\zeta,\tau=0}
    \subseteq
    \ppp^{H,B,\sigma}_{\zeta}\enspace,
\]
which follows from the calculation in \Cref{rem:interpolation}.

Therefore,
specializing the hard instance in \Cref{thm:lsgd_lb_informal} to $\tau=0$ and
$\zeta_\star=\zeta$ gives the following corollary.

\begin{corollary}
\label{cor:zeta_lower_bound}
If $K,M\geq 2$, then there exists a problem in the class
$\ppp_{\zeta}^{H,B,\sigma}$ such that, for any constant step-size $\eta>0$, the
final Local SGD iterate $\bar x_{KR}$ satisfies
\[
    \ee\!\left[F(\bar x_{KR})-F(x^\star)\right]
    \gtrsim
    \min\left\{
        HB^2,\,
        \frac{\zeta^2}{H},\,
        \frac{(H\zeta^2B^4)^{1/3}}{R^{2/3}}
    \right\}.
\]
\end{corollary}

To our knowledge, \Cref{cor:zeta_lower_bound} is the first lower bound for
Local SGD that makes the dependence on the averaged uniform heterogeneity
parameter $\zeta$ explicit.\footnote{The homogeneous lower bound of
\citet{glasgow2022sharp} also applies to the class
$\ppp^{H,B,\sigma}_{\zeta}$, but it corresponds to the case $\zeta=0$.}
A related algorithm-independent lower bound was proved by
\citet{hou2021fedchain}, but under a localized clientwise heterogeneity condition
rather than the global averaged condition in
\Cref{ass:zeta_everywhere}.\footnote{\citet{hou2021fedchain} bound
$\sup_{x\in A,m\in[M]}\norm{\nabla F_m(x)-\nabla F(x)}$ over an
algorithm-dependent region $A$. If one instead compares the global clientwise
condition
$\sup_{x\in\rr^d,m\in[M]}\norm{\nabla F_m(x)-\nabla F(x)}\leq \bar\zeta$
with \Cref{ass:zeta_everywhere}, then the former implies
\Cref{ass:zeta_everywhere} with $\zeta\leq \bar\zeta$, while the reverse
implication can lose a factor of order $\sqrt M$ in the heterogeneity parameter.}
The main difficulty in proving lower bounds under \Cref{ass:zeta_everywhere} is
that this assumption rules out many standard quadratic hard instances with
different client Hessians: for such instances, the average gradient disagreement
typically grows unboundedly with $\norm{x}$. Our first-order block avoids this
obstruction. It has $\tau=0$, and the clients differ only through opposite
linear perturbations. Hence the average gradient disagreement is constant over
the entire domain, so the bound at $x^\star$ becomes a uniform bound over all
$x$ (see \Cref{rem:interpolation}). Applying \Cref{thm:lsgd_lb_informal} with
$\tau=0$ and $\zeta_\star=\zeta$ proves the corollary.

\section{Multi-pass SGD with Replacement under Second-Order Heterogeneity}
\label{sec:fresh_sgd_soh}

The preceding sections focused on Local SGD, where second-order heterogeneity controls the drift created by multiple local updates between communication rounds. We now ask whether the same second-order viewpoint is useful even in an algorithmic setting with no local-updates. To this end, we study serial fresh-sample SGD with replacement. This algorithm does not create local client trajectories, but it still samples gradients from heterogeneous clients, and so second-order heterogeneity can affect how the client smoothness parameters enter the convergence rate. To capture this, unlike the rest of the paper, we do not assume a uniform smoothness parameter across clients. Instead, each client objective \(F_m\) is \(H_m\)-smooth, and we define
\[
    \hat H := \max_{m\in[M]} H_m\enspace.
\]
We also assume that the average objective \(F\) is \(H\)-smooth. Note that \(H\le \hat H\), but the ratio \(\hat H/H\) can be as large as \(M\). All other assumptions remain unchanged, including convexity, bounded optimum, and Assumptions~\ref{ass:zeta_optimal} and \ref{ass:tau}. We denote this problem class by
\(
    \ppp_{\zeta_\star,\tau}^{H,\hat H,B,\sigma}\enspace.
\)

Specifically, we consider the multi-pass SGD with replacement algorithm for \(N:=TM\) iterations, where
\(T\) denotes the number of ``effective epochs''. At each iteration
\(t\in[0,N-1]\), the algorithm samples \(m_t\sim\mathrm{Unif}([M])\), draws
\(z_t\sim\mathcal D_{m_t}\), and performs
\[
    x_{t+1} = x_t-\eta\nabla f(x_t;z_t)\enspace,
\]
with initialization \(x_0=0\). The output is the average iterate
\(
    \bar x_N\coloneqq\frac1N\sum_{t=0}^{N-1}x_t\enspace.
\)
The sampled gradient is unbiased for \(\nabla F(x_t)\). Moreover, by Assumption~\ref{ass:zeta_optimal} and the stochastic-gradient assumption in \eqref{ass:stoch_first_order}, we have
\[
    \ee\left[\norm{\nabla f(x^\star;z_t)}^2\right]
    \le
    \sigma^2+\zeta_\star^2\enspace.
\]
Thus, standard SGD analyses for convex smooth functions give,
\begin{equation}
\label{eq:old_sgd_rate}
    \ee\bigl[F(\bar x_N)-F(x^\star)\bigr]
    \lesssim
    \frac{\hat H B^2}{TM} + \frac{(\sigma+\zeta_\star)B}{\sqrt{TM}}\enspace.
\end{equation}

A natural question is whether Assumption~\ref{ass:tau} allows one to improve
the optimization term in \eqref{eq:old_sgd_rate} by replacing the worst-client smoothness \(\hat H\) with the smoothness \(H\) of the average objective, paying only an additional term depending on \(\tau\). In the worst case, the answer is no. Indeed, the class contains homogeneous instances with \(\tau=0\), where all clients coincide and hence \(H=\hat H\). On such
instances, fresh-sample SGD reduces to \(N\) iterations of SGD on an \(\hat H\)-smooth convex objective, and the standard lower bound \(\Omega(\hat H B^2/(TM))\) applies.

This homogeneous construction, however, does not explain what happens on genuinely heterogeneous instances for which \(H\llll\hat H\). In such problems, the average objective may be well conditioned, while some clients have much
larger curvature. The homogeneous construction only certifies an \(HB^2/(TM)\) lower bound in this regime, which can be much smaller than the \(\hat H\)-dependent upper bound in \eqref{eq:old_sgd_rate}.

The following lower bound shows that the \(\hat H B^2/(TM)\) optimization term remains tight even in this heterogeneous regime. The key point is that a large \(\hat H\) can arise from rare high-curvature directions, with
\[
    \hat H
    \cong
    \sqrt M\,\tau\enspace,
\]
rather than from uniformly large curvature across all clients.

\begin{theorem}[Informal version of \Cref{thm:fresh_sgd_heterogeneous_lb_formal}]
\label{thm:fresh_sgd_heterogeneous_lb_informal}
Fix \(M,T\ge1\), \(B>0\), and parameters satisfying
\[
    \tau\cong H
    \cong \frac{\hat H}{\sqrt M} \enspace.
\]
Then there exists an instance in
\(\ppp_{\zeta_\star,\tau}^{H,\hat H,B,\sigma}\)  such that for every constant stepsize \(\eta>0\), fresh-sample SGD with replacement satisfies
\[
    \ee\bigl[F(\bar x_{TM})-F(x^\star)\bigr]
    \gtrsim
    \frac{\hat H B^2}{TM}
    +
    \frac{(\sigma+\zeta_\star)B}{\sqrt{TM}}\enspace.
\]
\end{theorem}

The construction, given formally in
\Cref{app:fresh_sample_sgd_shell_lower_bound}, follows the same guiding principle as our second-order-heterogeneity lower bound for Local SGD: it uses a rare-curvature direction carried by heterogeneous clients. In particular, the construction has clients whose curvature is much larger than that of the average objective, yielding the scaling
\(\tau \cong H \cong \hat H/\sqrt M\).

The key difference from Local SGD is that, for serial fresh-sample SGD, this rare curvature does not create a new local-drift effect. Instead, it realizes the usual worst-client smoothness barrier. An orthogonal stochastic and first-order-heterogeneity block gives the usual noise lower bound.

The lower bound should therefore be interpreted as follows. The standard \(\hat H B^2/(TM)\) barrier is not merely an artifact of the homogeneous \(\hat H\)-smooth hard instance. Even in a genuinely heterogeneous regime where \(H\llll \hat H\), serial fresh-sample SGD can still suffer the usual \(\hat H\)-dependent rate. In this sense, for serial SGD with replacement, second-order heterogeneity explains the source of worst-client smoothness:
\(
    \hat H
    \cong
    \sqrt M\,\tau
\).
This contrasts with the Local SGD results in the main part of the paper, where bounded second-order heterogeneity leads to genuinely new local-drift and communication effects that cannot be reduced to the standard serial \(\hat H\)-smooth analysis.

    \section{Discussion}
	\label{sec:discussion}
	
	Our results sharpen the theoretical picture of Local SGD under realistic data heterogeneity. In particular, they show that the benefits of local computation do not require the more restrictive uniform first-order heterogeneity assumptions used in much of the prior literature. Instead, bounded second-order heterogeneity is enough to obtain provable improvements over Mini-batch SGD in the general convex setting. This confirms the conjectural picture of \citet{patel2024limits} beyond strongly-convex objectives: what matters for controlling client drift is not that local gradients remain uniformly close everywhere, but that local objectives have sufficiently aligned curvature along the trajectory.

	There remain several important open questions. First, there is still a gap in the red regime of \Cref{fig:lower_vs_upper}. Closing this gap likely requires either a sharper analysis that goes beyond the current consensus-error framework, or a new lower-bound construction that improves both the first- and second-order effects in our analysis. We conjecture that the upper bounds presented in our paper are actually tight. Second, it remains unclear whether the lack of second-order differentiability is indeed necessary for our lower bound. Third, it might be possible to get the best of the three upper bounds for Local SGD in a single bound using inner and outer step-sizes. 
	
	Our analysis of fresh-sample SGD with replacement suggests that second-order heterogeneity can play a role in understanding multi-pass finite-sum methods. An intriguing next step is to identify other variants of second-order heterogeneity that would allow fresh-sample SGD to avoid relying on the worst-case smoothness conditions. Additionally, it would be even more fascinating to develop a matching theory for random reshuffling in the context of second-order heterogeneity.

    \section*{Acknowledgments}

    Rustem Islamov and Aurelien Lucchi acknowledge the financial support of the Swiss National Science Foundation, SNSF grant No 207392.

	\newpage
	\appendix

	\counterwithin{figure}{section}
	
	\vbox{
		{\hrule height 2pt \vskip 0.15in \vskip -\parskip}
		\centering
		{\LARGE\bf Appendix\par}
		{\vskip 0.2in \vskip -\parskip \hrule height 0.5pt \vskip 0.09in}
	}

    \etocdepthtag.toc{appendix}

    \begingroup
    \renewcommand{\contentsname}{Appendix Contents}
    \setcounter{tocdepth}{2}
    \etocsettagdepth{main}{none}
    \etocsettagdepth{appendix}{subsection}
    \tableofcontents
    \endgroup
	
	\crefalias{section}{appendix}

	\section{Additional Notation for the Appendix}
	
	Throughout the appendix, we analyze the vanilla variant of Local SGD, which uses the outer step-size $\beta=1$ (see \Cref{alg:local-sgd}). In particular, for brevity, we will use the notation in \Cref{eq:local_updates}, along with the following \textit{``ghost iterate''} which might not be computed at each time step $t\in[0,KR]$,
	\begin{align*}
		\bar x_t = \frac{1}{M}\sum_{m\in[M]}x_t^m .
	\end{align*}
	It would also be useful to define the following sequences, which appear many times throughout our analyses, for all $t\in[0,KR]$
	\begin{align*}
		A(t) &\coloneqq  \ee\sb{\norm{\bar x_t - x^\star}^2} ,\qquad\text{ (Expected ghost iterate error)}\\
		B(t) &\coloneqq  \frac{1}{M}\sum_{m\in[M]}\ee\sb{\norm{x_t^m - x^\star}^2} \geq A(t) ,\qquad\text{ (Expected average iterate error)}\\
		\Xi_t &\coloneqq  \frac{1}{M}\sum_{m\in[M]}\norm{x_t^m - \bar x_t}^2 ,\qquad\text{ (Consensus error)}\\
		C(t) &\coloneqq  \ee\sb{\Xi_t}= \frac{1}{M}\sum_{m\in[M]}\ee\sb{\norm{x_t^m - \bar x_t}^2}  ,\qquad \text{ (Expected consensus error)}\\
		D(t) &\coloneqq  \ee\sb{F(\bar x_t) - F(x^\star)}\geq 0 ,\qquad  \text{ (Expected function sub-optimality)}\\
		E(t) &\coloneqq  \frac{1}{M}\sum_{m\in[M]}\ee\sb{F(x_t^m) - F(x^\star)} ,\qquad \text{ (Expected average function sub-optimality)}\\
		V(t) &\coloneqq  \frac{1}{M}\sum_{m\in[M]}\ee\sb{\norm{\nabla F_m(\bar x_t) - \nabla F(\bar x_t)}^2} ,\qquad  \text{ (Expected gradient heterogeneity)}\\
		R(t) & \coloneqq \frac{1}{M}\sum_{m\in[M]}\ee [\mathcal{D}(x_t^m, \bar{x}_t)],\qquad \text{ (Expected average Bregman divergence)}\\
        &\text{where }\mathcal{D}_{F_m}(x_t^m, \bar{x}_t) \coloneqq F_m(x_t^m) - F_m(\bar{x}_t) - \langle\nabla F_m(\bar{x}_t), x_t^m - \bar{x}_t\rangle\\
		Q(t) &\coloneqq \frac{1}{M}\sum_{m\in[M]}\ee [\langle\nabla F_m(\bar{x}_t) - \nabla F(\bar{x}_t),x_t^m - \bar{x}_t\rangle],\\
		S(t) &\coloneqq \frac{1}{M} \sum_{m\in[M]} \ee[F_m(x_t^m) - F_m(x^\star)],\qquad \text{ (Expected local function sub-optimality)}.
	\end{align*}
	where throughout we fix $x^\star\in \arg\min_{x\in\rr^d} F(x)$ to be the optimizer of the average objective which satisfies \Cref{ass:zeta_optimal}. We will also denote the maximum of the first two sequences by,
	\begin{align*}
		A_{\max} &\coloneqq  \max_{t\in[0, KR]} A(t) ,\\
		B_{\max} &\coloneqq  \max_{t\in[0, KR]} B(t) ,\\
		C_{\max} &\coloneqq  \max_{t\in[0, KR]} C(t) ,\\
		V_{\max} &\coloneqq  \max_{t\in[0, KR]} V(t).
	\end{align*}
	
	It would also be useful to define the following filtration for all $t\in[KR]$, which allows easily dealing with the stochastic noise on each client,
	\begin{align*}
		\hhh_t \coloneqq  \sigma\rb{z_0^1,\dots,z_0^M, \dots, z_{t-1}^1, \dots, z_{t-1}^M} ,
	\end{align*}
	where $\sigma\rb{S}$ represents the sigma algebra\footnote{We apologize for the abuse of notation w.r.t. \Cref{ass:stoch_first_order}.} generated by a set $S\in \zzz^\star$. We write $x \in \hhh$ to mean that $x$ is measurable under the sigma algebra $\hhh$. It would be useful to note that,
	\begin{align*}
		x_t^m, \bar x_t, \Xi_t \in \hhh_t ,
	\end{align*}
	for all $m\in[M]$ and $t\in[KR]$.

	\section{Useful Prior Results and Technical Context}
	\subsection{A Canonical Descent Lemma and the Upper Bound of \citet{woodworth2020minibatch}}
	Many of our proofs rely on the following standard lemma, which, for Local SGD, imitates the one-step recursion usual in SGD analysis by introducing an expected consensus error term. 
	
	\begin{lemma}[Lemma 7, \citet{woodworth2020minibatch}]\label{lem:lemma7_woodworth} Consider problems in $\ppp^{H, B, \sigma}$ then for $\eta \le \frac{1}{2H}$, the Local SGD ghost iterate satisfies for all $t\in[0, T-1]$
		\begin{equation}
			\ee[F(\bar{x}_t) - F^\star] \le 
			\frac{1}{\eta}\ee[\|\bar{x}_t-x^*\|^2_2]
			- \frac{1}{\eta}\ee[\|\bar{x}_{t+1}-x^\star\|^2_2] 
			+ \frac{3\eta\sigma^2}{M} 
			+ \frac{2H}{M}\sum_{m=1}^M\ee\left[\|\bar{x}_t - x_t^m\|^2_2\right] .
		\end{equation}
	\end{lemma}
	In the standard analysis due to \citet{woodworth2020minibatch}, \Cref{ass:zeta_everywhere} allows bounding the consensus error as follows,
	\begin{align}\label{eq:consensus_woodworth}
		C(t) \leq 3\eta^2 K\sigma^2 + 6 \eta^2 K^2\zeta^2 .
	\end{align}
	Plugging this upper bound, unrolling the recursion, telescoping appropriately, and tuning the step-size results in the upper bound stated in \Cref{tab:rates} for problem class $\ppp_{\zeta}^{H, B, \sigma}$. Our analysis in this paper effectively replaces the $\zeta^2$ term in the above upper bound by $\zeta_\star^2 + \tau^2B^2$ up to numerical constants, thus also allowing us to replace the dependence on \Cref{ass:zeta_everywhere} by Assumptions \ref{ass:zeta_optimal} and \ref{ass:tau}, i.e., extending their paper to $\ppp_{\zeta_\star, \tau}^{H, B, \sigma}$. Our strategy builds on the following observation due to \citet{patel2024limits}, which bounds the uniform gradient heterogeneity in terms of the existing iterate error.
	\begin{lemma}
		Consider problems in $\ppp^{H, B, \sigma}_{\zeta_\star, \tau}$ which also satisfy twice differentiability on each client, then for all $m\in[M]$,
		\begin{align*}
			\sup_{x\in \bb_2(D), m\in[M]}\norm{\nabla F_m(x) - \nabla F(x)}^2 \leq (\zeta_\star + \tau D)^2 . 
		\end{align*}
	\end{lemma}
	In \Cref{lem:vmax_ub}, we refine the above lemma in the following ways: 
	\begin{enumerate}
		\item we require it to only hold in expectation over the trajectory of Local SGD, i.e., we only want to upper bound $V_ {\max}$;
		\item we remove the twice-differentiability requirement;
		\item up to numerical constants, we replace the $D$ in the upper bound by $B$ using a  careful analysis of the Local SGD iterates, and bounding $A_ {\max}$; and
		\item we replace the supremum by an average over $m\in[M]$.
	\end{enumerate}
	
	\subsection{The Empirical Variance Trick of \citet{luo2025revisiting}}
	The consensus error upper bound of \citet{woodworth2020minibatch} mentioned in \Cref{eq:consensus_woodworth} proceeds by applying Jensen's inequality as follows,
	\begin{align*}
		C(t) &= \frac{1}{M}\sum_{m\in[M]}\ee\sb{\norm{x_t^m - \bar x_t}^2} ,\\
		&= \frac{1}{M}\sum_{m\in[M]}\ee\sb{\norm{x_t^m - \frac{1}{M}\sum_{n\in[M]} x_t^n}^2} ,\\
		&\leq \frac{1}{M^2}\sum_{m,n\in[M]}\ee\sb{\norm{x_t^m - x_t^n}^2} ,
	\end{align*}
	and then upper bound $\ee\sb{\norm{x_t^m - x_t^n}^2}$ for every pair of machines $m,n\in[M]$. This is also why they need the stronger version of \Cref{ass:zeta_everywhere} which takes a supremum over $m\in[M]$, instead of an average. \citet{luo2025revisiting} instead use the following series of calculations, 
	\begin{align}\label{eq:luo_trick}
		\Xi_t &= \frac{1}{M}\sum_{m\in[M]}\norm{x_t^m - \bar x_t}^2 ,\nonumber\\
		&= \min_{y\in \rr^d}\frac{1}{M}\sum_{m\in[M]}\norm{x_t^m - y}^2 ,\nonumber\\
		&\leq \frac{1}{M}\sum_{m\in[M]}\norm{x_t^m - (\bar x_{t-1} - \eta \nabla F(\bar x_{t-1}))}^2 ,
	\end{align}
	which is simply a restatement of the fact that the empirical mean of $M$ vectors minimizes their empirical variance. While quite simple, this trick is quite useful for our analyses, as the reader will note in the next section.
	
	\subsection{A Co-coercivity Property on Each Machine}
	
	We note the following property on each machine, which is often used in Local SGD analysis to bound the consensus error~\cite{woodworth2020local,woodworth2020minibatch,luo2025revisiting}, and more generally relies on the folklore co-coercivity result of Baillon-Haddad~\cite{baillon1977quelques,bauschke2009baillon}.
	
	\begin{lemma}\label{lem:cocoercive_machine}
		For problems in the class $\ppp^{H, B, \sigma}$ and with $\eta \leq \frac{2}{H}$ each machine $m\in[M]$ satisfies for all $x,y\in\rr^d$,
		\begin{align*}
			\norm{x-\eta\nabla F_m(x) - \rb{y - \eta \nabla F_m(y)}}^2 \leq \norm{x-y}^2 .
		\end{align*}
	\end{lemma}
	\begin{proof}
		We first prove the following auxiliary inequality: for any $x,y\in\rr^d$,
		\begin{align}
			\frac{1}{2H}\norm{\nabla F_m(x)-\nabla F_m(y)}^2
			\leq
			F_m(x)-F_m(y)-\inner{\nabla F_m(y)}{x-y}.
			\label{eq:onesided-cocoercive}
		\end{align}
		
		To show this, define
		\begin{align*}
			g(z) \coloneqq  F_m(z) - \inner{\nabla F_m(y)}{z} .
		\end{align*}
		Then $g$ is convex and $H$-smooth, and
		\begin{align*}
			\nabla g(z)=\nabla F_m(z)-\nabla F_m(y) ,\qquad \nabla g(y)=0 .
		\end{align*}
		Since $g$ is $H$-smooth, for any $u,v\in\rr^d$,
		\begin{align*}
			g(u)\leq g(v)+\inner{\nabla g(v)}{u-v}+\frac{H}{2}\norm{u-v}^2 .
		\end{align*}
		Apply this with $u = x - \frac{1}{H}\nabla g(x)$ and $v=x$ to get
		\begin{align*}
			g\rb{x-\frac{1}{H}\nabla g(x)}
			&\leq
			g(x)
			+ \inner{\nabla g(x)}{-\frac{1}{H}\nabla g(x)}
			+ \frac{H}{2}\norm{\frac{1}{H}\nabla g(x)}^2\\
			&=
			g(x)-\frac{1}{2H}\norm{\nabla g(x)}^2 .
		\end{align*}
		Since $g$ is convex and $\nabla g(y)=0$, the point $y$ is a global minimizer of $g$, hence
		\begin{align*}
			g(y)\leq g\rb{x-\frac{1}{H}\nabla g(x)} .
		\end{align*}
		Combining the last two displays yields
		\begin{align*}
			g(x)-g(y)\geq \frac{1}{2H}\norm{\nabla g(x)}^2 .
		\end{align*}
		Substituting back the definition of $g$ gives \eqref{eq:onesided-cocoercive}.
		
		Now apply \eqref{eq:onesided-cocoercive} twice: once as written, and once with $x$ and $y$ swapped. This gives
		\begin{align*}
			\frac{1}{2H}\norm{\nabla F_m(x)-\nabla F_m(y)}^2
			&\leq
			F_m(x)-F_m(y)-\inner{\nabla F_m(y)}{x-y},\\
			\frac{1}{2H}\norm{\nabla F_m(x)-\nabla F_m(y)}^2
			&\leq
			F_m(y)-F_m(x)-\inner{\nabla F_m(x)}{y-x}.
		\end{align*}
		Adding these inequalities, the function-value terms cancel, and we obtain
		\begin{align}\label{eq:cocoercivity_condition}
			\frac{1}{H}\norm{\nabla F_m(x)-\nabla F_m(y)}^2
			\leq
			\inner{\nabla F_m(x)-\nabla F_m(y)}{x-y} .
		\end{align}
		
		Now note that
		\begin{align*}
			&\norm{x-\eta\nabla F_m(x) - \rb{y - \eta \nabla F_m(y)}}^2\\
			&\qquad= \norm{x- y -\eta\rb{\nabla F_m(x) - \nabla F_m(y)}}^2 ,\\
			&\qquad= \norm{x-y}^2 + \eta^2\norm{\nabla F_m(x) - \nabla F_m(y)}^2 - 2\eta \inner{x-y}{\nabla F_m(x) - \nabla F_m(y)} .
		\end{align*}
		Using the bound just proved,
		\begin{align*}
			-2\eta \inner{x-y}{\nabla F_m(x) - \nabla F_m(y)}
			\leq
			-\frac{2\eta}{H}\norm{\nabla F_m(x)-\nabla F_m(y)}^2 .
		\end{align*}
		Therefore,
		\begin{align*}
			&\norm{x-\eta\nabla F_m(x) - \rb{y - \eta \nabla F_m(y)}}^2\\
			&\qquad\leq \norm{x-y}^2
			+ \eta^2\norm{\nabla F_m(x)-\nabla F_m(y)}^2
			- \frac{2\eta}{H}\norm{\nabla F_m(x)-\nabla F_m(y)}^2\\
			&\qquad=
			\norm{x-y}^2
			+ \eta\rb{\eta-\frac{2}{H}}
			\norm{\nabla F_m(x)-\nabla F_m(y)}^2\\
			&\qquad\leq \norm{x-y}^2 ,
		\end{align*}
		where the last step uses $\eta\leq \frac{2}{H}$.
	\end{proof}

	\subsection{Local SGD is Never Worse than Single Machine SGD in the Homogeneous Setting}
	
	A classic result from \citet{woodworth2020local} is to show that in the homogeneous setting, that is, for problems in the problem class $\ppp_{hom}^{H, B,\sigma}$, Local SGD is no worse than single machine SGD. Specifically, they presented an alternative convergence analysis of Local SGD that is required to obtain the rate in \Cref{tab:rates}. 
	
	\begin{lemma}
		For problems in the problem class $\ppp_{hom}^{H, B,\sigma}$ choosing the step-size as $\eta~=~\min\cb{\frac{1}{2H}, \frac{B}{\sigma\sqrt{KR}}}$, the Local SGD iterate $\hat x^{\textrm{Local\text{-}SGD}} = \frac{1}{TM}\sum_{t\in[0,T-1], m\in[M]}x_t^m$ satisfies,
		\begin{align*}
			\ee\sb{F(\hat x^{\textrm{Local\text{-}SGD}}) - F(x^\star)} \leq 2\cdot\rb{\frac{HB^2}{KR} + \frac{\sigma B}{\sqrt{KR}}} .
		\end{align*}
	\end{lemma}
	
	In the regime when $\sigma$ is small, in fact \citet{woodworth2021min} showed that accelerated single machine SGD is mini-max optimal, and in that regime, using a similar argument as the above result, one can show that accelerated local SGD is optimal as well. In \Cref{thm:lsgd_singlemachine_new}, we extend the above guarantee to the heterogeneous setting, specifically to the problem class $\ppp_{\zeta_\star, \tau}^{H, B, \sigma}$.

	\subsection{Other Useful Analytical Identities}
	We will also use the following inequality several times; it is essentially a variant of the A.M.-G.M. inequality.
	\begin{lemma}\label{lem:mod_am_gm}
		For any $a,b\in\rr$ and $\gamma >0$ we have,
		\begin{align*}
			(a+b)^2 &\leq \rb{1 + \frac{1}{\gamma}}a^2 + \rb{1+\gamma}b^2 .
		\end{align*}
	\end{lemma}
	\begin{proof}
		Note the following,
		\begin{align*}
			(a+b)^2 &= a^2 + b^2 + 2ab ,\\
			&= a^2 + b^2 + 2\rb{\frac{a}{\sqrt{\gamma}}}\rb{\sqrt{\gamma} b} ,\\
			&\overset{(i)}{\leq} a^2 + b^2 + \frac{a^2}{\gamma} + \gamma b^2 ,\\
			&\leq \rb{1 + \frac{1}{\gamma}}a^2 + \rb{1 + \gamma}b^2 ,
		\end{align*}
		where $(i)$ uses A.M.-G.M. Inequality. This proves the first statement of the lemma.
	\end{proof}
	
	\section{Useful New Technical Lemmas for Our Upper Bounds}\label{app:new_technical_lemmas}
	We will prove several important technical lemmas in this section. 
	\subsection{Technical Lemmas for the Consensus Error-based Proof}
	In this sub-section, we develop lemmas that would be used in the proof of \Cref{thm:lsgd_consensus} in the next section.
	\subsubsection{A Refined Upper Bound for $C(t)$}
	We begin by proving the following refined consensus error upper bound by refining the proof of \citet{woodworth2020minibatch} based on some tricks in \citet{luo2025revisiting} and classical convex-smooth identities.
	
	\begin{lemma}\label{lem:consensus_bound_V_max}
		For problems in the class $\ppp^{H, B, \sigma}$ and for $\eta \le \frac{2}{H}$, the iterates of Local SGD satisfy for all $t\in [0, T]$, 
		\begin{align*}
			C(t) = \ee\left[\frac{1}{M}\sum_{m=1}^M\|\bar{x}_t - x_t^m\|^2_2\right] 
			\le 3(t-\delta(t)) \cdot \left[\eta^2KV_ {\max} + \eta^2\sigma^2\right] .
		\end{align*}
		In particular, this implies $C_ {\max} \leq 3\eta^2K^2V_ {\max} + 3\eta^2K\sigma^2$.
	\end{lemma}
	
	\begin{proof}
		For $t=\delta(t)$, i.e., at the time steps where communication happens, the upper bound in the lemma is trivially true as both the left and right-hand sides are zero. For the rest of the proof, assume $t\neq \delta(t)$ and note the following,
		\begin{align*}
			C(t) &= \ee\left[\Xi_{t}\right] ,\\
			&= \frac{1}{M}\sum_{m\in[M]}\ee\sb{\norm{x_t^m - \bar x_t}^2}  ,\\
			&\overset{(i)}{\le} \frac{1}{M}\sum_{m=1}^M\ee\sb{\|x_{t}^m - (\bar{x}_{t-1} - \eta \nabla F(\bar{x}_{t-1}))\|^2_2} ,\\
			&= \frac{1}{M}\sum_{m=1}^M\ee\sb{\|x_{t-1}^m - \eta \nabla f_m(x_{t-1}^m; z_{t-1}^m) - (\bar{x}_{t-1} - \eta \nabla F(\bar{x}_{t-1}))\|^2_2} ,\\
			&= \frac{1}{M}\sum_{m=1}^M\ee\sb{\|x_{t-1}^m \pm \nabla F_m(x_{t-1}^m) - \eta \nabla f_m(x_{t-1}^m; z_{t-1}^m) - (\bar{x}_{t-1} - \eta \nabla F(\bar{x}_{t-1}))\|^2_2} ,\\
			&\overset{(ii)}{\le}
			\frac{1}{M}\sum_{m=1}^M\ee\sb{\|x_{t-1}^m - \eta \nabla F_m(x_{t-1}^m) - (\bar{x}_{t-1} - \eta \nabla F(\bar{x}_{t-1}))\|^2_2} + \eta^2\sigma^2 ,
			\\
			&= \frac{1}{M}\sum_{m=1}^M\ee\sb{\|x_{t-1}^m - \eta \nabla F_m(x_{t-1}^m) - (\bar{x}_{t-1} \pm \eta \nabla F_m(\bar x_{t-1})- \eta \nabla F(\bar{x}_{t-1}))\|^2_2} + \eta^2\sigma^2 ,\\
			&\overset{(iii)}{\le}
			\frac{1}{M}\sum_{m=1}^M(1+\frac{1}{K-1})\ee\sb{\|x_{t-1}^m - \eta \nabla F_m(x_{t-1}^m) - (\bar{x}_{t-1} - \eta \nabla F_m(\bar{x}_{t-1}))\|^2_2}
			\\
			&\qquad \qquad + \frac{\eta^2}{M}\sum_{m=1}^M\rb{1 + (K-1)}\ee\sb{\|\nabla F(\bar{x}_{t-1}) - \nabla F_m(\bar{x}_{t-1})\|^2_2} + \eta^2\sigma^2 ,
			\\
			&\overset{(iv)}{\le} 
			\rb{1 + \frac{1}{K-1}}\cdot\frac{1}{M}\sum_{m=1}^M\ee\sb{\|x_{t-1}^m - \bar{x}_{t-1}\|^2_2}
			+ \eta^2KV(t-1) 
			+ \eta^2\sigma^2 ,
			\\
			&\le \rb{1 + \frac{1}{K-1}}\cdot C(t-1)
			+ \eta^2KV_ {\max} 
			+ \eta^2\sigma^2 ,
		\end{align*}
		where $(i)$ uses the variance trick of \citep{luo2025revisiting} mentioned in \Cref{eq:luo_trick}; $(ii)$ uses conditioning on $\hhh_{t-1}$, applying tower rule, noting the independence of noise across machines and then using the variance bound in \Cref{ass:stoch_first_order}; $(iii)$ uses modified AM-GM inequality, i.e., \Cref{lem:mod_am_gm} with $\gamma = \frac{1}{K-1}$; $(iv)$ uses the definition of $V(t-1)$ and \Cref{lem:cocoercive_machine}. Now we unroll the above recursion up to the last communication round,
		\begin{align*}
			C(t) &\leq \rb{1 + \frac{1}{K-1}}^{t-\delta(t)}\cdot C(\delta(t))
			+ \rb{\eta^2KV_ {\max} + \eta^2\sigma^2}\cdot\sum_{k=0}^{t-\delta(t)-1}\rb{1 + \frac{1}{K-1}}^k ,\\
			&\leq \rb{\eta^2KV_ {\max} + \eta^2\sigma^2}\cdot\sum_{k=0}^{t-\delta(t)-1}\rb{1 + \frac{1}{K-1}}^{K-1} ,\\
			&\leq \rb{\eta^2KV_ {\max} + \eta^2\sigma^2}\cdot\sum_{k=0}^{t-\delta(t)-1}e ,\\
			&\leq 3(t-\delta(t))\rb{\eta^2KV_ {\max} + \eta^2\sigma^2} .
		\end{align*}
		This proves the first statement, and for the second statement, we use $t-\delta(t)\leq K$.
	\end{proof}
	
	\subsubsection{A Refined Upper Bound for $V_ {\max}$}
	We refine \Cref{lem:lemma7_woodworth} in the following lemma and allow it to hold over $\ppp_{\zeta_\star, \tau}^{H, B, \sigma}$.
	\begin{lemma}\label{lem:vmax_ub}
		For all problems in the class $\ppp_{\zeta_\star, \tau}^{H, B, \sigma}$ we have,
		\begin{align*}
			V_ {\max}  &\le 2\tau^2A_ {\max} + 2\zeta_\star^2 . 
		\end{align*}
	\end{lemma}
	\begin{proof}
		Using the definition of $V(t)$, we have for all $t\in[0, T]$,
		\begin{align}
			V(t) &= \frac{1}{M}\sum_{m=1}^M\ee\sb{\|\nabla F_m(\bar x_t) - \nabla F(\bar x_t)\|^2_2} ,\nonumber\\
			&= \frac{1}{M}\sum_{m=1}^M\ee\sb{\|\nabla F_m(\bar x_t) - \nabla F(\bar x_t)\pm (\nabla F_m(x^\star) - \nabla F(x^\star))\|^2_2} ,\nonumber\\
			&\overset{(i)}{\le} \frac{2}{M}\sum_{m=1}^M\ee\sb{\|\nabla F_m(\bar x_t) - \nabla F(\bar x_t) - (\nabla F_m(x^\star) - \nabla F(x^\star))\|^2_2}\nonumber\\
			& \qquad \qquad \qquad + \frac{2}{M}\sum_{m=1}^M\ee\sb{\|\nabla F_m(x^\star) - \nabla F(x^\star)\|^2_2} ,\nonumber\\
			&\overset{(ii)}{\le} 2\tau^2\ee[\|\bar x_t-x^\star\|^2_2]
			+ 2\zeta_\star^2 ,\label{eq:vmax_ub_intermediate}\\
			&\leq 2\tau^2 A_ {\max} + 2\zeta_\star^2 ,\nonumber
		\end{align}
		where $(i)$ uses \Cref{lem:mod_am_gm} with $\gamma=1$; $(ii)$ uses \Cref{ass:zeta_optimal} and \Cref{ass:tau}.
		Since the bound holds for any $t,$ then we have
		\begin{align*}
			V_ {\max} = \max_{0\le s \le T}V(s) 
			&\le 2\tau^2A_ {\max} + 2\zeta_\star^2 ,
		\end{align*}
		which proves the lemma.
	\end{proof}

	\subsubsection{A New Upper Bound for $A_ {\max}$}
	\begin{lemma}\label{lem:amax_ub}
		For all problems in the problem class $\ppp_{\zeta_\star, \tau}^{H, B, \sigma}$ and with step-size $\eta~\leq~\min \cb{\frac{1}{2H}, \sqrt[3]{\frac{1}{24H\tau^2K^3R}}}$ Local SGD's iterates are such that,
		\begin{align*}
			A_ {\max} &\leq 2B^2+ \frac{6\eta^2\sigma^2 T}{M} + 24\eta^3HK^3R\zeta_\star^2 + 12\eta^3 HK^2R\sigma^2 .
		\end{align*}
	\end{lemma}
	\begin{proof}
		Starting from a re-arrangement of \Cref{lem:lemma7_woodworth} we get that for all $t\in [0, T-1]$, 
		\begin{align*}
			A(t+1) &\leq A(t) - \eta D(t) + \frac{3\eta^2\sigma^2}{M} + 2\eta HC(t) ,\\
			&\overset{(i)}{\leq} A(t) + \frac{3\eta^2\sigma^2}{M} + 2\eta HC(t) ,\\
			&\leq A(0) + (t+1)\cdot\frac{3\eta^2\sigma^2}{M} + 2\eta H\sum_{k=0}^{t}C(k) ,\\
			&\overset{(ii)}{\leq} B^2 + (t+1)\cdot\rb{\frac{3\eta^2\sigma^2}{M} + 2\eta H C_ {\max}} ,\\
			&\overset{(iii)}{\leq} B^2 + (t+1)\cdot\rb{\frac{3\eta^2\sigma^2}{M} + 2\eta H \rb{3\eta^2K^2V_ {\max} + 3\eta^2K\sigma^2}} ,\\
			&= B^2+ (t+1)\cdot\rb{\frac{3\eta^2\sigma^2}{M} +  6\eta^3HK^2V_ {\max} + 6\eta^3 HK\sigma^2} ,\\
			&\overset{(iv)}{\leq} B^2+ T\cdot\rb{\frac{3\eta^2\sigma^2}{M} +  12\eta^3H\tau^2K^2A_ {\max} + 12\eta^3HK^2\zeta_\star^2 + 6\eta^3 HK\sigma^2} ,
		\end{align*}
		where in (i) we use that the function-suboptimality is always non-negative, i.e., $D(t)\geq 0$; in (ii) we use that $\bar x_0 = 0$, $\norm{x^\star}\leq B$ and the definition of $C_ {\max}$; in (iii) we use \Cref{lem:consensus_bound_V_max}; in (iv) we use \Cref{lem:vmax_ub} and the fact that $t+1 \leq T$. Since the above bound holds for all $t\in [0, T-1]$ as well as trivially for $t=-1$ we can take the maximum over time to get,
		\begin{align*}
			A_ {\max} &= \max_{t\in [-1, T-1]} A(t+1) ,\\
			&\leq  B^2+ KR\cdot\rb{\frac{3\eta^2\sigma^2}{M} +  12\eta^3H\tau^2K^2A_ {\max} + 12\eta^3HK^2\zeta_\star^2 + 6\eta^3 HK\sigma^2} .
		\end{align*}
		Re-arranging the above we get,
		\begin{align*}
			\rb{1- 12\eta^3H\tau^2K^3R}A_ {\max} &\leq   B^2+ KR\cdot\rb{\frac{3\eta^2\sigma^2}{M} + 12\eta^3HK^2\zeta_\star^2 + 6\eta^3 HK\sigma^2}  ,\\
			\Rightarrow \frac{1}{2}A_ {\max} &\overset{(i)}{\leq}   B^2+ KR\cdot\rb{\frac{3\eta^2\sigma^2}{M} + 12\eta^3HK^2\zeta_\star^2 + 6\eta^3 HK\sigma^2} ,\\
			\Rightarrow A_ {\max} &\leq 2B^2+ \frac{6\eta^2\sigma^2 T}{M} + 24\eta^3HK^3R\zeta_\star^2 + 12\eta^3 HK^2R\sigma^2 ,
		\end{align*}
		where in (i) we used that $12\eta^3H\tau^2K^3R \leq \frac{1}{2}$ which follows from our assumption in the lemma $\eta \leq \sqrt[3]{\frac{1}{24H\tau^2K^3R}}$. 
	\end{proof}
	
	\subsubsection{An Even More Refined Upper Bound for $C_ {\max}$!}
	The observations so far allow us to prove the key technical upper-bound result of our paper, which complements \citet{patel2025revisiting} by providing a refined upper bound on the consensus error of Local SGD.
	\begin{lemma}\label{lem:cmax_ub}
		For all problems in the problem class $\ppp_{\zeta_\star, \tau}^{H, B, \sigma}$ and with step-size $\eta~\leq~\min \cb{\frac{1}{2H}, \sqrt[3]{\frac{1}{24H\tau^2K^3R}}}$ Local SGD's iterates are such that,
		\begin{align*}
			C_ {\max}\leq \frac{B^2}{\eta H KR}+ \frac{3\eta\sigma^2 }{HM} + 18\eta^2K^2\zeta_\star^2 + 9\eta^2 K\sigma^2 .
		\end{align*}
	\end{lemma}
	\begin{proof}
		This lemma follows trivially from combining \Cref{lem:consensus_bound_V_max,lem:vmax_ub,lem:amax_ub},
		\begin{align*}
			C_ {\max} &\overset{(i)}{\leq} 3\eta^2K^2V_ {\max} + 3\eta^2K\sigma^2 ,\\
			&\overset{(ii)}{\leq} 6\eta^2K^2\tau^2A_ {\max} + 6\eta^2K^2\zeta_\star^2 + 3\eta^2K\sigma^2 ,\\
			&\overset{(iii)}{\leq} \frac{A_ {\max}}{2\eta H KR}+ 6\eta^2K^2\zeta_\star^2 + 3\eta^2K\sigma^2 ,\\
			&\overset{(iv)}{\leq} \frac{B^2}{\eta H KR}+ \frac{3\eta\sigma^2 }{HM} + 12\eta^2K^2\zeta_\star^2 + 6\eta^2 K\sigma^2+ 6\eta^2K^2\zeta_\star^2 + 3\eta^2K\sigma^2 ,\\
			&= \frac{B^2}{\eta H KR}+ \frac{3\eta\sigma^2 }{HM} + 18\eta^2K^2\zeta_\star^2 + 9\eta^2 K\sigma^2 ,
		\end{align*}
		where in (i) we use \Cref{lem:consensus_bound_V_max}; in (ii) we use \Cref{lem:vmax_ub}; in (iii) we use the assumption on the step-size that $\eta \leq \sqrt[3]{\frac{1}{24H\tau^2K^3R}}$ which implies, $6\eta^2K^2\tau^2 \leq \frac{1}{2\eta H KR}$; in (iv) we use \Cref{lem:amax_ub}. 
		
		This proves the lemma. 
	\end{proof}

	\subsection{Technical Lemmas for the Single-Machine Mimicking Proof}
	
	\begin{lemma}\label{lem:bound_V_max_through_B_max} For all problems in the class 
		$\ppp_{\zeta_\star, \tau}^{H, B, \sigma}$ we have,
		\begin{align*}
			V_ {\max}  &\le 2\tau^2B_ {\max} + 2\zeta_\star^2. 
		\end{align*}
	\end{lemma}
	\begin{proof}
		We continue the inequality $V(t) \le 2\tau^2 \ee[\|\bar{x}_t - x^\star\|^2_2] + 2\zeta_\star^2$ obtained in \Cref{lem:vmax_ub} as follows
		\begin{align}
			V(t)
			&\overset{(i)}{\le} 2\tau^2A(t)
			+ 2\zeta_\star^2\notag\\
			&\overset{(ii)}{\le} 2\tau^2 B(t) + 2\zeta_\star^2\notag\\
			&\overset{(iii)}{\le} 2\tau^2 B_{\max} + 2\zeta_\star^2,
		\end{align}
		where $(i)$ is obtained in \Cref{lem:vmax_ub}, $(ii)$ follows from $B(t) \ge A(t)$, and $(iii)$ is by the definition of $B_{\max}$. Since the inequality is true for any $t\in[0,KR]$, then we have $V_{\max} \le  2\tau^2 B_{\max} + 2\zeta_\star^2$, which concludes the proof.
	\end{proof}
	
	\begin{lemma}\label{lem:bound_C_max_through_B_max}
		For all problems in the class 
		$\ppp_{\zeta_\star, \tau}^{H, B, \sigma}$ we have,
		\begin{align*}
			C_ {\max}  &\le 6\eta^2K^2\tau^2 B_{\max}
			+ 6\eta^2K^2\zeta_\star^2
			+ 3\eta^2K\sigma^2.
		\end{align*}
	\end{lemma}
	
	\begin{proof}
		We need to use the statement of \Cref{lem:bound_V_max_through_B_max} in \Cref{lem:consensus_bound_V_max}:
		\begin{align*}
			C(t) &\overset{(i)}{\le} 3\eta^2K^2V_{\max} + 3\eta^2K\sigma^2\\
			&\overset{(ii)}{\le} 6\eta^2K^2\tau^2B_{\max} + 6\eta^2K^2\zeta_\star^2 + 3\eta^2K\sigma^2,
		\end{align*}
		where $(i)$ is obtained in \Cref{lem:consensus_bound_V_max}, $(ii)$ uses \Cref{lem:bound_V_max_through_B_max}.
	\end{proof}
	
	\subsubsection{An Alternative One Step Descent Lemma Using $B(t)$ and $D(t)$}
	
	\begin{lemma}\label{lem:descent_Bt_Dt} Consider problem in $\ppp_{\zeta_\star, \tau}^{H, B, \sigma}$, then for $\eta \le \frac{1}{4H}$, the Local SGD iterates satisfy for all $t\in[0,T-1]$,
		\begin{align*}
			D(t) \le \frac{B(t) - B(t+1)}{\eta} + 2\eta V(t) + \eta\sigma^2 + 2\sqrt{V(t)C(t)}.
		\end{align*}
		Moreover, 
		\begin{align*}
			B(t+1) \le B(t) + 2\eta^2V(t) + \eta^2\sigma^2 + 2\eta\sqrt{V(t)C(t)}.
		\end{align*}
		
	\end{lemma}
	\begin{proof}
		Since each $F_m$ is convex, then $R(t) \ge 0$ for all $t\in[0,T]$. Now we decompose the term $S(t) = \frac{1}{M}\sum_{m=1}^M \ee[F_m(x_t^m) - F_m(x^\star)]$ as follows 
		\begin{align*}
			S(t) &= \frac{1}{M}\sum_{m=1}^M \ee[F_m(x_t^m) - F_m(\bar{x}_t) - \langle \nabla F_m(\bar{x}_t),x_t^m - \bar{x}_t\rangle]\\
			&\qquad + \frac{1}{M}\sum_{m=1}^M \ee[F_m(\bar{x}_t) - F_m(x^\star) + \langle\nabla F_m(\bar{x}_t),x_t^m - \bar{x}_t\rangle]\\
			&= R(t) + D(t) + \frac{1}{M}\sum_{m=1}^M \ee[\langle\nabla F_m(\bar{x}_t),x_t^m - \bar{x}_t\rangle].
		\end{align*}
		Since $\frac{1}{M}\sum_{m=1}^M \ee[\langle\nabla F(\bar{x}_t),x_t^m - \bar{x}_t\rangle = 0$, then we continue as follows
		\begin{align}\label{eq:s_r_d_q}
			S(t) = \frac{1}{M}\sum_{m=1}^M \ee[F_m(x_t^m) - F_m(x^\star)] = R(t) + D(t) + Q(t).
		\end{align}
		By Cauchy-Schwarz inequality, we have 
		\begin{align*}
			|Q(t)| &= \left|\frac{1}{M}\sum_{m=1}^M\ee[\nabla F_m(\bar{x}_t) - \nabla F(\bar{x}_t), x_t^m - \bar{x}_t]\right|\\
			&\le \frac{1}{M}\sum_{m=1}^M\ee[\|\nabla F_m(\bar{x}_t) - \nabla F(\bar{x}_t)\|^2_2 \cdot \|x_t^m - \bar{x}_t\|^2_2]\\
			&\le \left(\frac{1}{M}\sum_{m=1}^M\ee[\|\nabla F_m(\bar{x}_t) - \nabla F(\bar{x}_t)\|^2_2 ]\right)^{1/2} \cdot \left(\frac{1}{M}\sum_{m=1}^M\ee[\|x_t^m - \bar{x}_t\|^2_2 ]\right)^{1/2}\\
			&= \sqrt{V(t)C(t)}.
		\end{align*}

		Fix any $m\in[M]$ and $t\in[0,T-1]$. Define the tentative local update
		\[
		\tilde x_{t+1}^m\coloneqq x_t^m-\eta \nabla f(x_t^m;z_t^m) .
		\]
		As before, if $t+1\neq \delta(t+1)$, then $\tilde x_{t+1}^m=x_{t+1}^m$, whereas if $t+1=\delta(t+1)$, then $\tilde x_{t+1}^m$ is the pre-communication iterate.
		
		We begin with the identity
		\begin{align}
			\norm{\tilde x_{t+1}^m-x^\star}^2
			&=
			\norm{x_t^m-x^\star}^2
			-2\eta\inner{x_t^m-x^\star}{\nabla f(x_t^m;z_t^m)}
			+\eta^2\norm{\nabla f(x_t^m;z_t^m)}^2 .
			\label{eq:sm_step_expand}
		\end{align}
		
		Conditioning on $\hhh_t$ and using $x_t^m\in \hhh_t$, we have
		\[
		\ee[\nabla f(x_t^m;z_t^m)\mid \hhh_t]=\nabla F_m(x_t^m) .
		\]
		Thus
		\begin{align}
			\ee\sb{\norm{\tilde x_{t+1}^m-x^\star}^2\mid \hhh_t}
			&=
			\norm{x_t^m-x^\star}^2
			-2\eta\inner{x_t^m-x^\star}{\nabla F_m(x_t^m)}\nonumber\\
			&\qquad+\eta^2 \ee\sb{\norm{\nabla f(x_t^m;z_t^m)}^2\mid \hhh_t} \nonumber\\
			&\le \norm{x_t^m-x^\star}^2
			-2\eta(F_m(x_t^m) - F_m(x^\star))\nonumber\\
			&\qquad+\eta^2 \ee\sb{\norm{\nabla f(x_t^m;z_t^m)}^2\mid \hhh_t}
			\label{eq:sm_step_conditional_new},
		\end{align}
		where the inequality follows from convexity. Rearranging gives
		\begin{align}
			&F_m(x_t^m) - F_m(x^\star) \nonumber\\
			&\qquad=
			\frac{1}{2\eta}
			\rb{
				\norm{x_t^m-x^\star}^2
				-
				\ee\sb{\norm{\tilde x_{t+1}^m-x^\star}^2\mid \hhh_t}
			}
			+
			\frac{\eta}{2}\ee\sb{\norm{\nabla f(x_t^m;z_t^m)}^2\mid \hhh_t} .
			\label{eq:sm_step_rearrange_new}
		\end{align}
		
		We next upper bound the stochastic second moment. Since
		\[
		\nabla f(x_t^m;z_t^m)
		=
		\nabla F_m(x_t^m)
		+
		\rb{\nabla f(x_t^m;z_t^m)-\nabla F_m(x_t^m)},
		\]
		expanding the square and conditioning on $\hhh_t$ gives
		\begin{align*}
			\ee\sb{\norm{\nabla f(x_t^m;z_t^m)}^2\mid \hhh_t}
			&=
			\norm{\nabla F_m(x_t^m)}^2
			+
			\ee\sb{\norm{\nabla f(x_t^m;z_t^m)-\nabla F_m(x_t^m)}^2\mid \hhh_t}\\
			&\le
			\norm{\nabla F_m(x_t^m)}^2+\sigma^2 ,
		\end{align*}
		where the cross term vanishes because the stochastic gradient is unbiased.
		
		Averaging \eqref{eq:sm_step_rearrange_new} over $m\in[M]$ and taking full expectation, we obtain
		\begin{align}
			S(t)
			&\le
			\frac{1}{2\eta}\rb{
				B(t)-\frac{1}{M}\sum_{m\in[M]}\ee\sb{\norm{\tilde x_{t+1}^m-x^\star}^2}
			}
			+
			\frac{\eta}{2M}\sum_{m\in[M]}\ee\sb{\norm{\nabla F_m(x_t^m)}^2}\nonumber\\
			&\qquad +
			\frac{\eta\sigma^2}{2} .
			\label{eq:sm_step_after_avg_new}
		\end{align}
		
		First, note that 
		\[
		\frac{1}{M}\sum_{m\in[M]}\ee\sb{\norm{\tilde x_{t+1}^m-x^\star}^2} \ge \ee\sb{\norm{\bar{x}_{t+1} - x^\star}}^2 ,
		\]
		therefore from \eqref{eq:sm_step_after_avg_new} we have
			\begin{align}
			S(t)
			\le
			\frac{1}{2\eta}\rb{
				B(t)-B(t+1)
			}
			+
			\frac{\eta}{2M}\sum_{m\in[M]}\ee\sb{\norm{\nabla F_m(x_t^m)}^2}
			+ 
			\frac{\eta\sigma^2}{2}
			\label{eq:sm_step_after_avg_new_second}
		\end{align}
		Second, we have by Young's inequality in \Cref{lem:mod_am_gm} that
		\begin{align}
			\frac{1}{M}\sum_{m\in[M]}\ee\sb{\norm{\nabla F_m(x_t^m)}^2}
			&\overset{(i)}{\le} \frac{2}{M}\sum_{m\in[M]}\ee\sb{\norm{\nabla F_m(x_t^m) - \nabla F_m(\bar{x}_t) }^2}\nonumber\\
			&\qquad + \frac{2}{M}\sum_{m\in[M]}\ee\sb{\norm{\nabla F_m(\bar{x}_t )}^2}\nonumber\\
			&\overset{(ii)}{\le} \frac{4H}{M}\sum_{m\in[M]}\ee[F_m(x_t^m) - F_m(\bar{x}_t) - \langle \nabla F_m(\bar{x}_t), x_t^m - \bar{x}_t\rangle]\nonumber\\
			&\qquad 
			+ \frac{2}{M}\sum_{m\in[M]}\ee\sb{\norm{\nabla F(\bar{x}_t )}^2}\nonumber\\
			&\qquad 
			+ \frac{2}{M}\sum_{m\in[M]}\ee\sb{\norm{\nabla F_m(\bar{x}_t ) - \nabla F(\bar{x}_t) }^2}\nonumber\\
			&\overset{(iii)}{\le} 4HR(t) + 4HD(t) + 2V(t),\label{eq:bound_grad_F_m_at_x_t_m}
		\end{align}
		where $(i)$ follows from \Cref{lem:mod_am_gm}, $(ii)$ follows from \eqref{eq:onesided-cocoercive} and equality $\frac{1}{M}\sum_{m\in[M]} \|a_m\|^2_2 =\norm{\frac{1}{M} \sum_{m\in[M]} a_m}^2 + \frac{1}{M} \sum_{m\in[M]} \norm{a_m - \frac{1}{M}\sum_{m\in[M]}a_m  }^2$ for any $\{a_m\}_{m\in[M]}$, $(iii)$ from $H$-smoothness inequality $\|\nabla F(\bar{x}_t) \|^2_2 \le 2H(F(\bar{x}_t) - F(x^\star))$ and definitions of $R(t), D(t), V(t)$. We use \eqref{eq:bound_grad_F_m_at_x_t_m} and \eqref{eq:s_r_d_q} in \eqref{eq:sm_step_after_avg_new_second} and obtain
		\begin{equation}
			S(t) = R(t) + D(t) + Q(t) \le \frac{1}{2\eta}(B(t) - B(t+1)) + \frac{\eta\sigma^2}{2} + 2\eta H(R(t) + D(t)) + \eta V(t).
		\end{equation}
		This implies that 
		\begin{equation}\label{eq:bound_D_and_R}
			(1-2\eta H)(D(t) + R(t)) \le \frac{B(t) - B(t+1)}{2\eta} + \eta V(t) + \frac{\eta\sigma^2}{2} - Q(t).
		\end{equation}
		Using $\eta \le \frac{1}{4H}$ and $|Q(t)| \le \sqrt{V(t)C(t)}$, we obtain
		\begin{equation*}
			D(t) + R(t) \le \frac{B(t) - B(t+1)}{\eta} +2 \eta V(t) + \eta\sigma^2 +2\sqrt{V(t)C(t)},
		\end{equation*}
		which proves the first statement of the lemma. At the same time, from \eqref{eq:bound_D_and_R} we have 
		\begin{align*}
			B(t+1) &\le B(t) + 2\eta^2 V(t) + \eta^2\sigma^2 - 2\eta(1-2\eta H)(D(t) + R(t)) - 2\eta Q(t)\\
			&\le B(t) + 2\eta^2 V(t) + \eta^2\sigma^2  + 2\eta\sqrt{V(t)C(t)},
		\end{align*}
		which proves the second statement of the lemma.
	\end{proof}
		
		\begin{lemma}\label{lem:bound_B_max_V_max_C_max_through_B2}
			Let 
			\[
			\eta = \left\{\frac{1}{4H}, \frac{1}{10\tau K\sqrt{R}}, \frac{B}{10\zeta_\star K\sqrt{R}}, \frac{B}{10\sigma \sqrt{KR}}\right\}.
			\]
			Then,
			\begin{align*}
			&B_{\max} \le 2B^2, \quad V_{\max} \le 4\tau^2B^2 + 2\zeta_\star^2 \le 4(\tau^2B^2 + \zeta_\star^2), \\
			&C_{\max} \le 12\eta^2K^2\tau^2B^2 + 6\eta^2K^2\zeta_\star^2 + 3\eta^2K\sigma^2.
			\end{align*}
		\end{lemma}
		
		\begin{proof}
		Let $U^2\coloneqq \tau^2B_{\max} + \zeta_\star^2$. From \Cref{lem:bound_V_max_through_B_max}, we have $V(t) \le 2U^2$, and from \Cref{lem:bound_C_max_through_B_max} we have $C_{\max} \le 6\eta^2K^2U^2 + 3\eta^2K\sigma^2$. This implies that 
		\[
		\sqrt{C_{\max}} \le \sqrt{6}\eta KU + \sqrt{3}\eta\sqrt{K}\sigma.
		\]
		From \Cref{lem:descent_Bt_Dt} we obtain
		\begin{align}
			B(t+1) &\le B(t) + 2\eta^2V(t) + \eta^2\sigma^2 + 2\eta\sqrt{V(t)C(t)} \nonumber\\
			&\le B(t) + 4\eta^2U^2 + \eta^2\sigma^2 + 2\eta\sqrt{2}U\cdot \sqrt{C_{\max}} \nonumber\\
			&\le B(t) + 4\eta^2U^2  + \eta^2\sigma^2 + 2\sqrt{2}\eta U(\sqrt{6}\eta KU + \sqrt{3}\eta\sqrt{K}\sigma)\nonumber\\
			&= B(t) + 4\eta^2U^2 + \eta^2\sigma^2 + 4\sqrt{3}\eta^2KU^2 + 2\sqrt{6}\eta^2U\sqrt{K}\sigma\nonumber\\
			&\le B(t) + 4\eta^2U^2 + \eta^2\sigma^2 + 4\sqrt{3}\eta^2KU^2 + \sqrt{6}\eta^2U^2K + \sqrt{6}\eta^2\sigma^2\nonumber\\
			&= B(t) + (4 + 4\sqrt{3}K + \sqrt{6}K) U^2\eta^2 + \eta^2\sigma^2(1+\sqrt{6})\nonumber\\
			&\le B(t) + 14K\eta^2U^2 + 4\eta^2\sigma^2.
		\end{align}
		Using the definition of $U^2$, we then have
		\[
		B(t+1) \le B(t) + 14\eta^2K\tau^2B_{\max} + 14\eta^2K\zeta_\star^2 + 4\eta^2\sigma^2.
		\]
		We sum the inequality above for all iterations between $0$ and $t-1$ and obatin
		\begin{align}\label{eq:bound_Bt}
		B(t) &\le B(0) + 14\eta^2 Kt\tau^2 B_{\max} + 14\eta^2Kt\zeta_\star^2 + 4\eta^2\sigma^2t\nonumber\\
		&\le  B(0) + 14\eta^2 KT\tau^2 B_{\max} + 14\eta^2KT\zeta_\star^2 + 4\eta^2\sigma^2T\nonumber\\
		& \le B^2 + 14\eta^2 K^2R\tau^2 B_{\max} + 14\eta^2K^2R\zeta_\star^2 + 4\eta^2\sigma^2KR,
		\end{align}
		where we use $B(0) \le B^2$. Using the choice of the step size $\eta$, we get
		\begin{align*}
			14\eta^2K^2R\tau^2 &\le 14 K^2R\tau^2 \cdot \frac{1}{100\tau^2K^2R} = \frac{14}{100},\\
			14\eta^2K^2R\zeta_\star^2 &\le 14K^2R\zeta_\star^2 \cdot \frac{B^2}{100\zeta_\star^2K^2R} = \frac{14B^2}{100},\\
			4\eta^2KR\sigma^2 &\le 4KR\sigma^2 \cdot \frac{B^2}{100\sigma^2KR} = \frac{4B^2}{100}.
		\end{align*}
		From \eqref{eq:bound_Bt} we then obtain
		\[
		B_{\max} \le B^2 + \frac{14B_{\max}}{100} + \frac{18B^2}{100} \Rightarrow B_{\max} \le \frac{118}{86}B^2 < 2B^2.
		\]
		\end{proof}

	\section{Proofs of Our Upper Bounds for $\ppp_{\zeta_\star, \tau}^{H,B,\sigma}$ 
    }\label{app:upper_bounds}
	In this section, we will prove our main upper bounds. 
	
	The first upper bound we prove combines the canonical one-step result in \Cref{lem:lemma7_woodworth}, with our refined consensus error upper bound in \Cref{lem:cmax_ub} followed by carefully tuning the step-size. 
	
	\begin{theorem}[An Upper Bound Based on Consensus Error; Full Statement of \Cref{thm:lsgd_consensus_informal}]\label{thm:lsgd_consensus}
		For problems in the class $\ppp_{\zeta_\star, \tau}^{H, B, \sigma}$, assuming the step-size is chosen as follows,
		\begin{align*}
			\eta = \min\cb{\frac{1}{2H}, \sqrt[3]{\frac{1}{24H\tau^2K^3R}}, \sqrt{\frac{B^2 M}{3\sigma^2 KR}}, \sqrt[3]{\frac{B^2}{12H\zeta_\star^2 K^3R}}, \sqrt[3]{\frac{B^2}{6H\sigma^2K^2R}}} ,
		\end{align*}
		Local SGD generates the output $\hat x^{\textrm{Local\text{-}SGD}} = \frac{1}{TM}\sum_{t\in[0, T-1], m\in[M]}x_t^m$ which satisfies,
		\begin{align*}
			&\ee\sb{F(\hat x^{\textrm{Local\text{-}SGD}}) - F(x^\star)}\\ 
			&\qquad\qquad\qquad\leq 18\cdot \rb{\frac{HB^2}{KR} + \frac{(H\tau^2)^{1/3}B^2}{R^{2/3}} + \frac{\sigma B}{\sqrt{MKR}} + \frac{(H\zeta_\star^2 B^4)^{1/3}}{R^{2/3}} + \frac{(H\sigma^2 B^4)^{1/3}}{K^{1/3}R^{2/3}}} .
		\end{align*}
	\end{theorem}
	
	The second upper bound we prove follows the observation of \citet{woodworth2020local} that Local SGD is never worse than single-machine SGD in the homogeneous setting, extending it over the problem class $\ppp_{\zeta_\star, \tau}^{H, B, \sigma}$ using the tools we developed in \Cref{app:new_technical_lemmas}.
	
	\begin{theorem}[An Upper Bound Based on Comparison to Single Machine SGD; Full Statement of \Cref{thm:lsgd_singlemachine_new}]\label{thm:lsgd_singlemachine_new}
		For problems in the class $\ppp_{\zeta_\star, \tau}^{H, B, \sigma}$, assuming the step-size is chosen as follows, 
		$$\eta = \min\cb{\frac{1}{4H}, \frac{1}{10 \tau K\sqrt{R}}, \frac{B}{10\zeta_\star K\sqrt{R}}, \frac{B}{10\sigma\sqrt{KR}}} ,$$
		Local SGD generates the output $\hat x^{\textrm{Local\text{-}SGD}} = \frac{1}{TM}\sum_{t\in[0, T-1], m\in[M]}x_t^m$ which satisfies,
		\[
		\ee[F(\hat x)-F(x^\star)]
		\le
		13\cdot\left(
		\frac{HB^2}{KR}
		+
		\frac{\sigma B}{\sqrt{KR}}
		+
		\frac{\zeta_\star B}{\sqrt R}
		+
		\frac{\tau B^2}{\sqrt R}
		\right).
		\]
	\end{theorem}

	\subsection{Proof of \Cref{thm:lsgd_consensus}}
	
	\begin{proof}
		We begin by re-stating the one-step recursion from \Cref{lem:lemma7_woodworth} for all $t\in [0, T-1]$
		\begin{align*}
			D(t) &\le 
			\frac{A(t) - A(t+1)}{\eta} + \frac{3\eta\sigma^2}{M} + 2HC(t) ,\\
			&\leq \frac{A(t) - A(t+1)}{\eta} + \frac{3\eta\sigma^2}{M} + 2HC_ {\max}
		\end{align*}
		Averaging this over time steps $t\in [0, T-1]$ and applying Jensen's inequality we get,
		\begin{align*}
			\ee\sb{F(\hat x^{\textrm{Local\text{-}SGD}}) - F(x^\star)} &\leq \frac{1}{T}\sum_{t=0}^{T-1}D(t) ,\\
			&\leq \frac{A(0) - A(T)}{\eta T} + \frac{3\eta\sigma^2}{M} + 2HC_ {\max} ,\\
			&\overset{(i)}{\leq} \frac{B^2}{\eta T} + \frac{3\eta\sigma^2}{M} + \frac{2B^2}{\eta T}+ \frac{6\eta\sigma^2 }{M} + 36\eta^2HK^2\zeta_\star^2 + 18\eta^2H K\sigma^2 ,\\
			&\leq \frac{3B^2}{\eta T} + \frac{9\eta\sigma^2}{M}+ 36\eta^2HK^2\zeta_\star^2 + 18\eta^2H K\sigma^2 ,
		\end{align*}
		where in (i) we use \Cref{lem:cmax_ub}. Now to get the final bound we will balance the second, third, and fourth term above with the first-term by setting the step-size $\eta$. In particular, combining this with the constraints on the step-size from \Cref{lem:cmax_ub} we get, 
		\begin{align*}
			\eta = \min\cb{\frac{1}{2H}, \sqrt[3]{\frac{1}{24H\tau^2K^3R}}, \sqrt{\frac{B^2 M}{3\sigma^2 KR}}, \sqrt[3]{\frac{B^2}{12H\zeta_\star^2 K^3R}}, \sqrt[3]{\frac{B^2}{6H\sigma^2K^2R}}} .
		\end{align*}
		With this choice of step-size we get, 
		\begin{align*}
			&\ee\sb{F(\hat x^{\textrm{Local\text{-}SGD}}) - F(x^\star)}\\
			&\qquad\leq\frac{6 HB^2}{KR} + \frac{6\sqrt[3]{24}(H\tau^2)^{1/3}B^2}{R^{2/3}} + \frac{6\sigma B}{\sqrt{MKR}} + \frac{6\sqrt[3]{12} (H\zeta_\star^2 B^4)^{1/3}}{R^{2/3}} + \frac{6\sqrt[3]{6} (H\sigma^2 B^4)^{1/3}}{K^{1/3}R^{2/3}} ,\\
			&\qquad\leq \frac{6 HB^2}{KR} + \frac{18(H\tau^2)^{1/3}B^2}{R^{2/3}} + \frac{6\sigma B}{\sqrt{MKR}} + \frac{18 (H\zeta_\star^2 B^4)^{1/3}}{R^{2/3}} + \frac{12 (H\sigma^2 B^4)^{1/3}}{K^{1/3}R^{2/3}} ,\\
			&\qquad\leq 18\cdot \rb{\frac{HB^2}{KR} + \frac{(H\tau^2)^{1/3}B^2}{R^{2/3}} + \frac{\sigma B}{\sqrt{MKR}} + \frac{(H\zeta_\star^2 B^4)^{1/3}}{R^{2/3}} + \frac{(H\sigma^2 B^4)^{1/3}}{K^{1/3}R^{2/3}}} ,
		\end{align*}
		which proves the theorem. 
	\end{proof}
	
	\subsection{Proof of \Cref{thm:lsgd_singlemachine_new}}
	
	\begin{proof}
		Let $U_0^2 = \tau^2B^2 + \zeta_\star^2$. Then from  \Cref{lem:bound_B_max_V_max_C_max_through_B2} we have 
		\[
		V_{\max} \le 4U_0^2, \quad C_{\max} \le 12\eta^2K^2\tau^2B^2 + 6\eta^2K^2\zeta_\star^2 + 3\eta^2K\sigma^2 \le 12\eta^2K^2U_0^2 + 3\eta^2K\sigma^2.
		\]
		This implies that $\sqrt{C_{\max}} \le \sqrt{12}\eta KU_0 + \sqrt{3}\eta\sqrt{K}\sigma \le 4\eta KU_0 + 2\eta\sqrt{K}\sigma$. From \Cref{lem:descent_Bt_Dt} we then obtain
		\begin{align*}
			D(t) &\le \frac{B(t) - B(t+1)}{\eta} + \eta\sigma^2 + 2\eta V(t) + 2\sqrt{V(t)C(t)}\\
			&\le \frac{B(t) - B(t+1)}{\eta} + \eta\sigma^2 + 8\eta U_0^2 + 2\cdot 2U_0(4\eta KU_0 + 2\eta\sqrt{K}\sigma)\\
			&\le \frac{B(t) - B(t+1)}{\eta} + 5\eta\sigma^2 + 28\eta KU_0^2\\
			&= \frac{B(t) - B(t+1)}{\eta} + 5\eta\sigma^2 + 28\eta(\tau^2B^2 + \zeta_\star^2).
		\end{align*}
		Averaging this inequality over all iterations $t\in[0,T-1]$, we obtain
		\begin{align}\label{eq:bound_Dt}
		\frac{1}{T}\sum_{t=0}^{T-1} D(t) &\le \frac{B^2}{\eta T} + 5\eta\sigma^2 + 28\eta K(\tau^2 B^2 + \zeta_\star^2)\nonumber\\
		&=  \frac{B^2}{\eta KR} + 5\eta\sigma^2 + 28\eta K(\tau^2 B^2 + \zeta_\star^2).
		\end{align}
		From the choice of the step size, we have
		\begin{align*}
			\frac{1}{\eta} &\le 4H + 10\tau K\sqrt{R} + \frac{10\zeta_\star K\sqrt{R}}{B} + \frac{10\sigma\sqrt{KR}}{B}.
		\end{align*}
		Hence,
		\begin{align}\label{eq:stepsize_1}
				\frac{B^2}{\eta KR} &\le \frac{4HB^2}{KR} + \frac{10\tau B^2}{\sqrt{R}} + \frac{10\zeta_\star B}{\sqrt{R}} + \frac{10\sigma B}{\sqrt{KR}}.
		\end{align}
		Moreover, we have
		\begin{align}\label{eq:stepsize_2}
			28\eta K\tau^2B^2 \le 28 K\tau^2B^2 \cdot \frac{1}{10\tau K\sqrt{R}} = 2.8\frac{\tau B^2}{\sqrt{R}},
		\end{align}
		and
		\begin{align}\label{eq:stepsize_3}
			28\eta K\zeta_\star^2 \le 28K\zeta_\star^2\cdot \frac{B}{10\zeta_\star K\sqrt{R}} = 2.8\frac{\zeta_\star B}{\sqrt{R}}.
		\end{align}
		Finally, we also have 
		\begin{align}\label{eq:stepsize_4}
			5\eta\sigma^2 \le 5\sigma^2 \frac{B}{10\sigma\sqrt{KR}} \le 0.5\frac{\sigma B}{\sqrt{KR}}.
		\end{align}
		Combining \eqref{eq:stepsize_1}, \eqref{eq:stepsize_2}, \eqref{eq:stepsize_3}, and \eqref{eq:stepsize_4} in \eqref{eq:bound_Dt}, we obtain 
		\begin{align*}
				\frac{1}{T}\sum_{t=0}^{T-1} D(t) &\le\frac{4HB^2}{KR} + \frac{12.8\tau B^2}{\sqrt{R}} + \frac{12.8\zeta_\star B}{\sqrt{R}} + \frac{10.5\sigma B}{\sqrt{KR}}\\
				&\le 13\cdot\left(
					\frac{HB^2}{KR} + \frac{\tau B^2}{\sqrt{R}} + \frac{\zeta_\star B}{\sqrt{R}} + \frac{\sigma B}{\sqrt{KR}}
					\right).
		\end{align*}
		It remains to use Jensen's inequality to obtain the bound for $\hat x^{\textrm{Local\text{-}SGD}}$.
	\end{proof}

	\section{A Single-Machine-Style Upper Bound for $\ppp_{\zeta}^{H,B,\sigma}$}
	\label{app:upper_bound_zeta}
	
	In this section, we give a second upper bound for Local SGD under the uniformly bounded first-order heterogeneity condition in \Cref{ass:zeta_everywhere}. The guarantee is complementary to the bound of \citet{woodworth2020minibatch} reported in \Cref{tab:rates}. Rather than comparing the ghost iterate to dense mini-batch SGD, we use the average client distance $B(t)$ as the potential, as in a single-machine SGD analysis, while directly controlling the global function sub-optimality $D(t)$. A local Bregman-divergence term then absorbs the part of the stochastic gradient second moment caused by evaluating the client gradients at different local iterates.
	
	We begin by stating the following resulting upper bound, which is the formal version of \Cref{thm:lsgd_zeta_informal}.
	
	\begin{theorem}[A Single-Machine-Style Upper Bound under Uniform First-Order Heterogeneity; Full Statement of \Cref{thm:lsgd_zeta_informal}]
		\label{thm:lsgd_zeta}
		For problems in the class $\ppp_{\zeta}^{H,B,\sigma}$, assume the step-size is chosen as follows,
		\begin{align*}
			\eta
			=
			\min\cb{
				\frac{1}{4H},
				\frac{B}{\sigma\sqrt{KR}},
				\frac{B}{\zeta K\sqrt{R}}
			}
			 ,
		\end{align*}
		Local SGD generates the output
		\begin{align*}
			\hat x^{\textrm{Local\text{-}SGD}}
			=
			\frac{1}{TM}\sum_{t\in[0,T-1],\,m\in[M]}x_t^m
			=
			\frac{1}{T}\sum_{t=0}^{T-1}\bar x_t
		\end{align*}
		which satisfies
		\begin{align*}
			\ee\sb{
				F\rb{\hat x^{\textrm{Local\text{-}SGD}}}
				-F(x^\star)
			}
			\leq
			9\cdot\rb{
				\frac{HB^2}{KR}
				+
				\frac{\sigma B}{\sqrt{KR}}
				+
				\frac{\zeta B}{\sqrt R}
			}
			 .
		\end{align*}
	\end{theorem}
	
	\subsection{A One-Step Recursion using $B(t)$ and $D(t)$}
	
	The following lemma is the main new ingredient. It is stated for the base class $\ppp^{H,B,\sigma}$ in terms of the trajectory-dependent gradient heterogeneity $V(t)$. The uniform bound from \Cref{ass:zeta_everywhere} will only be used afterward.
	
	\begin{lemma}
		\label{lem:lsgd_zeta_single_machine_step}
		For problems in the class $\ppp^{H,B,\sigma}$ and with $\eta\leq \frac{1}{4H}$, the iterates of Local SGD satisfy, for all $t\in[0,T-1]$,
		\begin{align*}
			D(t)
			\leq
			\frac{B(t)-B(t+1)}{\eta}
			+
			\eta\sigma^2
			+
			2\eta V(t)
			+
			2\sqrt{V(t)C(t)}
			 .
		\end{align*}
	\end{lemma}
	
	\begin{proof}
		Fix any $t\in[0,T-1]$. For brevity, define
		\begin{align*}
			\Delta_m(x)
			\coloneqq
			\nabla F_m(x)-\nabla F(x)
			 ,
		\end{align*}
		and recall that we denote the average local Bregman divergence between $x_t^m$ and $\bar x_t$ by
		\begin{align*}
			R(t)
			=
			\frac{1}{M}\sum_{m\in[M]}
			\ee\sb{
				F_m(x_t^m)-F_m(\bar x_t)
				-
				\inner{\nabla F_m(\bar x_t)}{x_t^m-\bar x_t}
			}
			\geq 0
			 ,
		\end{align*}
		where the inequality follows from the convexity of each $F_m$. Expanding every $F_m(x_t^m)$ around the common point $\bar x_t$ gives
		\begin{align}
			&\frac{1}{M}\sum_{m\in[M]}
			\ee\sb{F_m(x_t^m)-F_m(x^\star)}
			\nonumber\\
			&\qquad=
			D(t)
			+
			\mathcal R(t)
			+
			\frac{1}{M}\sum_{m\in[M]}
			\ee\sb{
				\inner{\Delta_m(\bar x_t)}{x_t^m-\bar x_t}
			}
			 .
			\label{eq:zeta_local_bregman_identity}
		\end{align}
		Indeed, the contribution of the global gradient vanishes because
		\begin{align*}
			\frac{1}{M}\sum_{m\in[M]}
			\inner{\nabla F(\bar x_t)}{x_t^m-\bar x_t}
			=0
			 .
		\end{align*}
		Moreover, Cauchy--Schwarz on the product space consisting of the random trajectory and a uniformly sampled machine gives
		\begin{align}
			\left|
			\frac{1}{M}\sum_{m\in[M]}
			\ee\sb{
				\inner{\Delta_m(\bar x_t)}{x_t^m-\bar x_t}
			}
			\right|
			\leq
			\sqrt{V(t)C(t)}
			 .
			\label{eq:zeta_bregman_cross_term}
		\end{align}
		
		We next relate the left-hand side of \eqref{eq:zeta_local_bregman_identity} to the potential $B(t)$. For every $m\in[M]$, define the tentative pre-synchronization update
		\begin{align*}
			\widetilde x_{t+1}^m
			\coloneqq
			x_t^m-\eta\nabla f(x_t^m;z_t^m)
			 .
		\end{align*}
		Conditioning on $\hhh_t$, expanding the squared distance to $x^\star$, and using the unbiasedness condition in \Cref{ass:stoch_first_order}, we obtain
		\begin{align*}
			&\ee\sb{
				\norm{\widetilde x_{t+1}^m-x^\star}^2
				\mid \hhh_t
			}\\
			&\qquad=
			\norm{x_t^m-x^\star}^2
			-
			2\eta\inner{x_t^m-x^\star}{\nabla F_m(x_t^m)}
			+
			\eta^2
			\ee\sb{
				\norm{\nabla f(x_t^m;z_t^m)}^2
				\mid \hhh_t
			}
			 .
		\end{align*}
		By convexity of $F_m$,
		\begin{align*}
			F_m(x_t^m)-F_m(x^\star)
			\leq
			\inner{x_t^m-x^\star}{\nabla F_m(x_t^m)}
			 .
		\end{align*}
		Therefore, averaging over $m\in[M]$ and taking full expectation yields
		\begin{align}
			&\frac{1}{M}\sum_{m\in[M]}
			\ee\sb{F_m(x_t^m)-F_m(x^\star)}
			\nonumber\\
			&\qquad\leq
			\frac{1}{2\eta}
			\left(
			B(t)
			-
			\frac{1}{M}\sum_{m\in[M]}
			\ee\sb{\norm{\widetilde x_{t+1}^m-x^\star}^2}
			\right)
			+
			\frac{\eta}{2M}\sum_{m\in[M]}
			\ee\sb{\norm{\nabla f(x_t^m;z_t^m)}^2}
			 .
			\label{eq:zeta_local_objective_distance}
		\end{align}
		The stochastic-gradient variance bound gives
		\begin{align}
			\frac{1}{M}\sum_{m\in[M]}
			\ee\sb{\norm{\nabla f(x_t^m;z_t^m)}^2}
			\leq
			\frac{1}{M}\sum_{m\in[M]}
			\ee\sb{\norm{\nabla F_m(x_t^m)}^2}
			+
			\sigma^2
			 .
			\label{eq:zeta_stochastic_second_moment}
		\end{align}
		
		It remains to control the deterministic-gradient second moment while retaining the non-negative term $R(t)$. We have
		\begin{align}
			&\frac{1}{M}\sum_{m\in[M]}
			\ee\sb{\norm{\nabla F_m(x_t^m)}^2}
			\nonumber\\
			&\qquad\leq
			\frac{2}{M}\sum_{m\in[M]}
			\ee\sb{
				\norm{\nabla F_m(x_t^m)-\nabla F_m(\bar x_t)}^2
			}
			+
			\frac{2}{M}\sum_{m\in[M]}
			\ee\sb{\norm{\nabla F_m(\bar x_t)}^2}
			\nonumber\\
			&\qquad\overset{(i)}{\leq}
			4H R(t)
			+
			2\ee\sb{\norm{\nabla F(\bar x_t)}^2}
			+
			2V(t)
			\nonumber\\
			&\qquad\overset{(ii)}{\leq}
			4H R(t)
			+
			4H D(t)
			+
			2V(t)
			 .
			\label{eq:zeta_deterministic_second_moment}
		\end{align}
		Here, $(i)$ uses the one-sided co-coercivity inequality \eqref{eq:onesided-cocoercive} and the variance identity
		\begin{align*}
			\frac{1}{M}\sum_{m\in[M]}
			\norm{\nabla F_m(\bar x_t)}^2
			=
			\norm{\nabla F(\bar x_t)}^2
			+
			\frac{1}{M}\sum_{m\in[M]}
			\norm{\nabla F_m(\bar x_t)-\nabla F(\bar x_t)}^2
			 .
		\end{align*}
		For $(ii)$, note that $F$ is also convex and $H$-smooth, and $\nabla F(x^\star)=0$. Applying the same one-sided co-coercivity inequality as in \eqref{eq:onesided-cocoercive} to $F$, with $x=\bar x_t$ and $y=x^\star$, gives
		\begin{align*}
			\ee\sb{\norm{\nabla F(\bar x_t)}^2}
			\leq
			2H D(t)
			 .
		\end{align*}
		
		Substituting \eqref{eq:zeta_stochastic_second_moment} and \eqref{eq:zeta_deterministic_second_moment} into \eqref{eq:zeta_local_objective_distance}, and using the fact that synchronization can only decrease the average squared distance to $x^\star$, gives
		\begin{align}
			\frac{1}{M}\sum_{m\in[M]}
			\ee\sb{F_m(x_t^m)-F_m(x^\star)}
			&\leq
			\frac{B(t)-B(t+1)}{2\eta}
			+
			2H\eta\rb{D(t)+R(t)}
			\\
			&\qquad+
			\eta V(t)
			+
			\frac{\eta\sigma^2}{2}
			 .
			\label{eq:zeta_local_objective_recursion}
		\end{align}
		For completeness, at a communication time Jensen's inequality gives
		\begin{align*}
			B(t+1)
			\leq
			\frac{1}{M}\sum_{m\in[M]}
			\ee\sb{\norm{\widetilde x_{t+1}^m-x^\star}^2}
			 ,
		\end{align*}
		and away from communication times the two sides are equal.
		
		Combining \eqref{eq:zeta_local_bregman_identity} and \eqref{eq:zeta_local_objective_recursion}, and then using \eqref{eq:zeta_bregman_cross_term}, yields
		\begin{align*}
			\rb{1-2H\eta}\rb{D(t)+ R(t)}
			&\leq
			\frac{B(t)-B(t+1)}{2\eta}
			+
			\eta V(t)
			+
			\frac{\eta\sigma^2}{2}
			+
			\sqrt{V(t)C(t)}
			 .
		\end{align*}
		Since $R(t)\geq 0$ and $\eta\leq \frac{1}{4H}$, we have $1-2H\eta\geq \frac{1}{2}$. Consequently,
		\begin{align*}
			D(t)
			\leq
			\frac{B(t)-B(t+1)}{\eta}
			+
			2\eta V(t)
			+
			\eta\sigma^2
			+
			2\sqrt{V(t)C(t)}
			 ,
		\end{align*}
		which proves the lemma.
	\end{proof}
	
	\subsection{Proof of \Cref{thm:lsgd_zeta}}
	
	\begin{proof}
		Under \Cref{ass:zeta_everywhere}, the definition of $V(t)$ immediately gives, for every $t\in[0,T]$,
		\begin{align*}
			V(t)
			=
			\frac{1}{M}\sum_{m\in[M]}
			\ee\sb{
				\norm{\nabla F_m(\bar x_t)-\nabla F(\bar x_t)}^2
			}
			\leq
			\zeta^2
			 .
		\end{align*}
		In particular, $V_ {\max}\leq \zeta^2$. Since the step-size in the theorem satisfies $\eta\leq \frac{1}{4H}\leq \frac{2}{H}$, \Cref{lem:consensus_bound_V_max} implies
		\begin{align}
			C_ {\max}
			\leq
			3\eta^2K^2\zeta^2
			+
			3\eta^2K\sigma^2
			 .
			\label{eq:zeta_consensus_from_existing_lemma}
		\end{align}
		Consequently,
		\begin{align*}
			\sqrt{C_ {\max}}
			\leq
			\sqrt{3}\eta K\zeta
			+
			\sqrt{3}\eta\sqrt K\sigma
			 .
		\end{align*}
		Applying \Cref{lem:lsgd_zeta_single_machine_step}, and using $V(t)\leq\zeta^2$ and $C(t)\leq C_ {\max}$, gives
		\begin{align*}
			D(t)
			&\leq
			\frac{B(t)-B(t+1)}{\eta}
			+
			\eta\sigma^2
			+
			2\eta\zeta^2
			+
			2\zeta\sqrt{C_ {\max}}
			\\
			&\leq
			\frac{B(t)-B(t+1)}{\eta}
			+
			\eta\sigma^2
			+
			2\eta\zeta^2
			+
			2\sqrt{3}\eta K\zeta^2
			+
			2\sqrt{3}\eta\sqrt K\zeta\sigma
			 .
		\end{align*}
		Using the A.M.-G.M. inequality,
		\begin{align*}
			2\sqrt{3}\eta\sqrt K\zeta\sigma
			\leq
			\sqrt{3}\eta\rb{K\zeta^2+\sigma^2}
			 .
		\end{align*}
		Since $K\geq 1$, the preceding bounds imply
		\begin{align}
			D(t)
			\leq
			\frac{B(t)-B(t+1)}{\eta}
			+
			8\eta\rb{K\zeta^2+\sigma^2}
			 .
			\label{eq:zeta_single_machine_master_recursion}
		\end{align}
		Summing \eqref{eq:zeta_single_machine_master_recursion} over $t\in[0,T-1]$, telescoping the potential, and using $B(0)=\norm{x^\star}^2\leq B^2$ and $B(T)\geq 0$, we obtain
		\begin{align}
			\frac{1}{T}\sum_{t=0}^{T-1}D(t)
			\leq
			\frac{B^2}{\eta T}
			+
			8\eta\rb{K\zeta^2+\sigma^2}
			 .
			\label{eq:zeta_single_machine_averaged_recursion}
		\end{align}
		By convexity of $F$,
		\begin{align*}
			\ee\sb{
				F\rb{\hat x^{\textrm{Local\text{-}SGD}}}
				-F(x^\star)
			}
			\leq
			\frac{1}{T}\sum_{t=0}^{T-1}D(t)
			 .
		\end{align*}
		It remains to substitute the step-size. Since $T=KR$,
		\begin{align*}
			\frac{B^2}{\eta T}
			&\leq
			4\frac{HB^2}{KR}
			+
			\frac{\sigma B}{\sqrt{KR}}
			+
			\frac{\zeta B}{\sqrt R}
			 ,\\
			8\eta\sigma^2
			&\leq
			8\frac{\sigma B}{\sqrt{KR}}
			 ,\\
			8\eta K\zeta^2
			&\leq
			8\frac{\zeta B}{\sqrt R}
			 .
		\end{align*}
		Combining these inequalities with \eqref{eq:zeta_single_machine_averaged_recursion} proves
		\begin{align*}
			\ee\sb{
				F\rb{\hat x^{\textrm{Local\text{-}SGD}}}
				-F(x^\star)
			}
			\leq
			4\frac{HB^2}{KR}
			+
			9\frac{\sigma B}{\sqrt{KR}}
			+
			9\frac{\zeta B}{\sqrt R}
			 ,
		\end{align*}
		and hence the claimed result.
	\end{proof}
	
	The bound in \Cref{thm:lsgd_zeta} is complementary to the result of \citet{woodworth2020minibatch} in \Cref{tab:rates}. In particular, the heterogeneity contribution $\zeta B/\sqrt R$ is no larger than $(H\zeta^2B^4)^{1/3}/R^{2/3}$ whenever $\zeta\leq HB/\sqrt R$. Similarly, $\sigma B/\sqrt{KR}$ is no larger than $(H\sigma^2B^4)^{1/3}/(K^{1/3}R^{2/3})$ whenever $\sigma\leq HB\sqrt{K/R}$. Moreover, in the limit $\zeta=0$, the theorem recovers the single-machine SGD term $\sigma B/\sqrt{KR}$ without requiring any second-order heterogeneity assumption.

	\section{Proofs of Our Lower Bounds}\label{app:lb}
	\paragraph{Key Simplifications.} In this section, we will prove the lower bound mentioned in \Cref{thm:lsgd_lb_informal}. We will make some standard simplifying assumptions in the proof, which do not affect the lower bound beyond numerical constants:
	\begin{itemize}
		\item Our constructions are all noiseless, i.e., $\sigma=0$ as we do not need to rely on the stochastic gradient noise to prove any of the lower bounds. 
		\item We will consider the $\tau = 0$ and $\zeta_\star = 0$ settings separately, i.e., we will provide two distinct constructions for these settings. This is without loss of generality, because both of our constructions lie in the class $\ppp^{H,B,0}_{\zeta_\star, \tau}$ and can be combined by putting the hard functions on disjoint co-ordinates. Since both constructions isolate one form of data heterogeneity, the second-order construction does not constrain the choice of $\zeta_\star$ and vice-versa. The reason for this is that the second-order construction ensures a shared optimizer across clients, so it does not restrict the choice of $\zeta_\star$ in the first-order construction for the other co-ordinates. Similarly, the first-order construction ensures that the clients have the same second-order derivative wherever it exists, which means it does not restrict $\tau$ in any way for the second-order construction. Combining the constructions leads to the desired lower bound. This idea of using disjoint co-ordinates is also how we combine the lower bound in \Cref{thm:lsgd_lb_informal} with the homogeneous lower bound of \citet{glasgow2022sharp} to get the final lower bound in \Cref{tab:rates}. For more context on using disjoint co-ordinates to prove distributed optimization lower bounds, see \citet{woodworth2021minimax} and \citet{patel2025makes}.
		\item As the reader would notice, both the constructions only require $M=2$, i.e., two machines. This is also without loss of generality. If there are an even number of machines, we can duplicate these two machines and generate the same lower bound. If there is an odd number of machines, we add a dummy machine with the average objective of the other two machines. This only makes the lower bound worse by a factor of $\frac{M-1}{M}$, which we do not care about, as we anyway look at the setting $M\geq 2$, and ignore numerical constants. 
	\end{itemize}
	All of the above simplifications are usual in this literature~\citep{woodworth2020local,woodworth2020minibatch,glasgow2022sharp,patel2022towards,patel2023still,patel2024limits,patel2025revisiting}. With these caveats in mind, we are ready to state our lower bounds. 
	
	\subsection{A Sharper Lower Bound When $\zeta_\star=0$}
	In this section, we prove a lower bound that refines the basic
	\(\Omega(\tau B^2/R)\) construction due to \citet{patel2025revisiting}. The key point is to separate the hard block's curvature scale from the instance's actual second-order heterogeneity.
	
	\begin{theorem}\label{thm:tau_sqrtR_lower_bound}
		There exist two deterministic convex quadratic client objectives
		\(
		F_1,F_2:\mathbb{R}^2\to\mathbb{R}
		\)
		in the problem class $\ppp^{H,B, 0}_{0, \tau}$, such that for Local SGD (see update~\eqref{eq:local_updates}) initialized at \(x_0=0\), with any number of local steps \(K\ge 1\), any step-size \(\eta>0\), the final iterate \(\bar x_R\) after \(R\) communication rounds satisfies
		\[
		F(\bar x_R)-F(x^\star)
		\ge
		c\cdot\,
		\min\left\{
		\frac{H B^2}{R},
		\frac{\tau B^2}{\sqrt R}
		\right\},
		\]
		where \(c>0\) is a universal numerical constant.
	\end{theorem}
	
	\begin{proof}
		Let
		\[
		\kappa \coloneqq  12R ,
		\qquad\text{and}\qquad
		\alpha \coloneqq  \frac{\kappa-1}{\kappa+1} .
		\]
		Then
		\[
		\frac{1+\alpha}{1-\alpha}=\kappa .
		\]
		Define for some $\alpha\in[0,1]$ to be determined later,
		\[
		e_1 \coloneqq  (1,0)^\top ,
		\qquad\text{and}\qquad
		v \coloneqq  \bigl(\alpha,\sqrt{1-\alpha^2}\bigr)^\top.
		\]
		These vectors give the following rank-one matrices,
		\[
		\hat A_1 \coloneqq  e_1e_1^\top ,
		\qquad\text{and}\qquad
		\hat A_2 \coloneqq  vv^\top .
		\]
		We define the hard-block curvature scale as follows
		\[
		\lambda \coloneqq  \min\{H,\sqrt{3R}\,\tau\} 
		\]
		and use it to define the actual Hessians on both our clients as
		\[
		A_1 \coloneqq  \lambda \hat A_1,
		\qquad\text{and}\qquad
		A_2 \coloneqq  \lambda \hat A_2.
		\]
		Specifically, our two client objectives are
		\[
		F_m(x)\coloneqq \frac{1}{2}(x-x^\star)^\top A_m(x-x^\star) ,
		\qquad \text{for }m\in\{1,2\} ,
		\]
		where \(x^\star\) will be chosen below.
		
		Since \(\|\hat A_m\|_2=1\), we have
		\[
		\|A_m\|_2=\lambda\le H .
		\]
		Thus, both the clients are convex and \(H\)-smooth.
		
		We next compute the second-order heterogeneity. Since
		\[
		\hat A_1-\hat A_2
		=
		\begin{bmatrix}
			1-\alpha^2 & -\alpha\sqrt{1-\alpha^2}\\
			-\alpha\sqrt{1-\alpha^2} & -(1-\alpha^2)
		\end{bmatrix} ,
		\]
		the eigenvalues of \(\hat A_1-\hat A_2\) are
		\[
		\pm \sqrt{1-\alpha^2} .
		\]
		Therefore
		\[
		\|A_1-A_2\|_2
		=
		\lambda\sqrt{1-\alpha^2} .
		\]
		Now
		\[
		1-\alpha^2
		=
		(1-\alpha)(1+\alpha)
		\le
		\frac{4}{\kappa} ,
		\]
		so
		\[
		\sqrt{1-\alpha^2}\le \frac{2}{\sqrt\kappa}
		=
		\frac{1}{\sqrt{3R}}  .
		\]
		By the choice of \(\lambda\),
		\[
		\|A_1-A_2\|
		\le
		\frac{\lambda}{\sqrt{3R}}
		\le
		\tau .
		\]
		Thus, the second-order heterogeneity (the stronger \Cref{ass:tau}) is at most \(\tau\).
		
		Moreover, the average Hessian is
		\[
		A\coloneqq \frac{A_1+A_2}{2}
		=
		\frac{\lambda}{2}(\hat A_1+\hat A_2) .
		\]
		The two nonzero eigenvalues of \(A\) are
		\[
		\Lambda_{\max}=\frac{\lambda}{2}(1+\alpha) ,
		\qquad\text{and}\qquad
		\Lambda_{\min}=\frac{\lambda}{2}(1-\alpha) .
		\]
		Hence
		\[
		\frac{\Lambda_{\max}}{\Lambda_{\min}}
		=
		\frac{1+\alpha}{1-\alpha}
		=
		\kappa
		=
		12R .
		\]
		Let \(u_{\max}\) and \(u_{\min}\) be unit eigenvectors of \(A\) associated
		with \(\Lambda_{\max}\) and \(\Lambda_{\min}\), respectively. Set
		\[
		x^\star\coloneqq -\frac{B}{\sqrt 2}(u_{\max}+u_{\min}) .
		\]
		Then
		\[
		\|x^\star\|_2=B .
		\]
		
		We now analyze the Local SGD dynamics in error coordinates, where $r\in[0,T]$
		\[
		\tilde x_t \coloneqq  \bar x_t-x^\star .
		\]
		Thus \(\tilde x_{qK}\) denotes the synchronized iterate after \(q\) communication rounds.
		
		For client \(m\), starting from the synchronized iterate \(\tilde x_{qK}\), the next \(K\) local gradient steps give
		\[
		\tilde x_{qK+K}^m
		=
		(I-\eta A_m)^K \tilde x_{qK}.
		\]
		Since
		\[
		A_m=\lambda p_mp_m^\top,
		\qquad
		p_1=e_1,\qquad p_2=v,
		\]
		is rank one and satisfies \(p_m^\top p_m=1\), we have
		\[
		A_m^2=\lambda A_m.
		\]
		Therefore,
		\[
		(I-\eta A_m)^K
		=
		I-\frac{1-(1-\eta\lambda)^K}{\lambda}A_m.
		\]
		Define the effective local step-size
		\[
		\tilde\eta
		\coloneqq 
		\frac{1-(1-\eta\lambda)^K}{\lambda}.
		\]
		Then
		\[
		\tilde x_{qK+K}^m
		=
		(I-\tilde\eta A_m)\tilde x_{qK}.
		\]
		
		Since we take the standard outer step-size \(\beta=1\), the server simply averages the two client endpoints:
		\[
		\tilde x_{qK+K}
		=
		\frac12\left(\tilde x_{qK+K}^1+\tilde x_{qK+K}^2\right).
		\]
		Substituting the local-update formula gives
		\[
		\begin{aligned}
			\tilde x_{qK+K}
			&=
			\frac12
			\left[
			(I-\tilde\eta A_1)\tilde x_{qK}
			+
			(I-\tilde\eta A_2)\tilde x_{qK}
			\right]  \\
			&=
			\left(I-\tilde\eta\frac{A_1+A_2}{2}\right)\tilde x_{qK}.
		\end{aligned}
		\]
		Let
		\[
		A\coloneqq \frac{A_1+A_2}{2}.
		\]
		Then the communication-round dynamics is exactly
		\[
		\tilde x_{qK+K}
		=
		(I-\tilde\eta A)\tilde x_{qK}.
		\]
		Iterating over \(R=T/K\) communication rounds,
		\[
		\tilde x_T
		=
		(I-\tilde\eta A)^R \tilde x_0.
		\]
		Thus Local SGD on this hard instance is exactly gradient descent for \(R\) iterations on the average quadratic with Hessian \(A\), using the effective step-size \(\tilde\eta\).
		
		We now lower bound this gradient descent trajectory. Recall that \(A\) has eigenvalues
		\[
		\Lambda_{\max}=\frac{\lambda}{2}(1+\alpha),
		\qquad
		\Lambda_{\min}=\frac{\lambda}{2}(1-\alpha),
		\]
		and condition number
		\[
		\kappa
		\coloneqq 
		\frac{\Lambda_{\max}}{\Lambda_{\min}}
		=
		\frac{1+\alpha}{1-\alpha}
		=
		12R .
		\]
		Thus, the eigenvalues are
		\[
		\boxed{
			\Lambda_{\max} = \lambda \cdot \frac{12R}{12R+1} ,
			\qquad\text{and}\qquad
			\Lambda_{\min} = \frac{\lambda}{12R+1} .
		}
		\]
		
		Let \(u_{\max}\) and \(u_{\min}\) be corresponding unit eigenvectors. We choose
		\[
		x^\star
		\coloneqq 
		-\frac{B}{\sqrt2}(u_{\max}+u_{\min}),
		\]
		so that \(\|x^\star\|=B\). Since the algorithm starts from \(\bar x_0=0\), we have
		\[
		\tilde x_0=-x^\star
		=
		\frac{B}{\sqrt2}(u_{\max}+u_{\min}).
		\]
		Hence
		\[
		\tilde x_T
		=
		\frac{B}{\sqrt2}
		(1-\tilde\eta\Lambda_{\max})^R u_{\max}
		+
		\frac{B}{\sqrt2}
		(1-\tilde\eta\Lambda_{\min})^R u_{\min}.
		\]
		
		If
		\[
		\tilde\eta\ge \frac{3}{\Lambda_{\max}},
		\]
		then
		\[
		|1-\tilde\eta\Lambda_{\max}|\ge 2,
		\]
		and so the high-curvature component gives
		\[
		\begin{aligned}
			F(\bar x_T)-F(x^\star)
			&=
			\frac12\tilde x_T^\top A\tilde x_T ,  \\
			&\ge
			\frac12\Lambda_{\max}
			\left(
			\frac{B}{\sqrt2}
			|1-\tilde\eta\Lambda_{\max}|^R
			\right)^2 ,  \\
			&\ge
			\frac{\Lambda_{\max}B^22^R}{4} ,\\
			&= \frac{3R^22^R}{12R+1}\cdot\frac{\lambda B^2}{R} ,\\
			&\overset{(R\geq 1)}{\geq} 
			\frac{1}{4}\cdot\frac{\lambda B^2}{R} .
		\end{aligned}
		\]

		It remains to now consider the more interesting case
		\[
		\tilde\eta<\frac{3}{\Lambda_{\max}}.
		\]
		Then
		\[
		\tilde\eta\Lambda_{\min}
		\le
		\frac{3\Lambda_{\min}}{\Lambda_{\max}}
		=
		\frac{3}{\kappa}.
		\]
		Therefore
		\[
		|1-\tilde\eta\Lambda_{\min}|
		\ge
		1-\frac{3}{\kappa}.
		\]
		The low-curvature component gives
		\[
		\begin{aligned}
			F(\bar x_T)-F(x^\star)
			&\ge
			\frac12\Lambda_{\min}
			\left(
			\frac{B}{\sqrt2}
			\left(1-\frac{3}{\kappa}\right)^R
			\right)^2 ,  \\
			&=
			\frac{\Lambda_{\min}B^2}{4}
			\left(1-\frac{3}{\kappa}\right)^{2R} .
		\end{aligned}
		\]
		Since \(\kappa=12R\),
		\[
		\left(1-\frac{3}{\kappa}\right)^{2R}
		=
		\left(1-\frac{1}{4R}\right)^{2R}
		\overset{(R\geq 1)}{\ge} \frac{9}{16} .
		\]
		Hence
		\begin{align*}
			F(\bar x_T)-F(x^\star)
			&\ge
			\frac{9 R}{64(12R+1)}\cdot\frac{\lambda B^2}{R} ,\\
			&\overset{(R\geq 1)}{\geq} \frac{3}{256} \cdot\frac{\lambda B^2}{R} .
		\end{align*}
		
		Combining the two cases,
		\[
		F(\bar x_T)-F(x^\star)
		\ge
		\frac{3}{256}\cdot\frac{\lambda B^2}{R} ,
		\]
		
		Finally, since
		\[
		\lambda\coloneqq \min\{H,\sqrt{3R}\tau\},
		\]
		we have
		\[
		\frac{\lambda B^2}{R}
		=
		\Theta\left(
		\min\left\{
		\frac{HB^2}{R},
		\frac{\tau B^2}{\sqrt R}
		\right\}
		\right).
		\]
		This proves the claim.
	\end{proof}
	
	\subsection{A Sharper Lower Bound When $\tau=0$}
	
	We first isolate three elementary lemmas that will be used in the construction.
	
	\begin{lemma}[Active-coordinate lower bound]
		\label{lem:tau-zero-active}
		Fix $L>0$, $\zeta_\star>0$, $R\ge 1$, an integer $K\ge 2$, and recall that $T=KR$. Define
		\[
		g(s)=
		\begin{cases}
			\frac{L}{4}s^2, & s<0,\\[3pt]
			\frac{L}{2}s^2, & s\ge 0,
		\end{cases}
		\qquad
		G_1(s)=g(s)-\zeta_\star s,
		\qquad
		G_2(s)=g(s)+\zeta_\star s.
		\]
		Run vanilla Local SGD on $G_1,G_2$ with $M=2$, step-size
		$\eta\le 1/L$, and initialization
		\[
		s_0^1=s_0^2=0.
		\]
		Let
		\[
		\bar s_t\coloneqq \frac{s_t^1+s_t^2}{2}.
		\]
		Then
		\[
		g(\bar s_T)-g(0)
		\gtrsim
		\min\left\{
		\frac{\zeta_\star^2}{L},\;
		L K^2\eta^2\zeta_\star^2
		\min\{(RK\eta L)^2,1\}
		\right\}.
		\]
	\end{lemma}
	
	\begin{proof}
		We compare the above dynamics to a simpler quadratic system. Define
		\[
		\widetilde G_1(s)\coloneqq \frac{L}{2}s^2-\zeta_\star s,
		\qquad
		\widetilde G_2(s)\coloneqq \frac{L}{4}s^2+\zeta_\star s.
		\]
		Let $\widetilde s_t^1,\widetilde s_t^2$ denote vanilla Local SGD on
		$\widetilde G_1,\widetilde G_2$, with the same $K,\eta$ and initialization
		\[
		\widetilde s_0^1=\widetilde s_0^2=0.
		\]
		Let
		\[
		\widetilde{\bar s}_t\coloneqq \frac{\widetilde s_t^1+\widetilde s_t^2}{2}.
		\]
		
		We first show that, at all communication times,
		\begin{align}
			\bar s_{rK}\le \widetilde{\bar s}_{rK}\le 0,
			\qquad r=0,1,\dots,R.\label{eq:comm_inequality}    
		\end{align}
		
		For the true dynamics, the one-step maps are
		\[
		T_1(s)\coloneqq s-\eta(G_1'(s))
		=s-\eta(g'(s)-\zeta_\star),
		\]
		and
		\[
		T_2(s)\coloneqq s-\eta(G_2'(s))
		=s-\eta(g'(s)+\zeta_\star).
		\]
		For the comparison dynamics, the one-step maps are
		\[
		\widetilde T_1(s)\coloneqq s-\eta(\widetilde G_1'(s))
		=(1-\eta L)s+\eta\zeta_\star,
		\]
		and
		\[
		\widetilde T_2(s)\coloneqq s-\eta(\widetilde G_2'(s))
		=\left(1-\frac{\eta L}{2}\right)s-\eta\zeta_\star.
		\]
		Since $\eta\le 1/L$, all four maps are nondecreasing.
		
		We now compare them pointwise. For client $1$, if $s<0$, then
		\[
		T_1(s)
		=
		\left(1-\frac{\eta L}{2}\right)s+\eta\zeta_\star
		\le
		(1-\eta L)s+\eta\zeta_\star
		=
		\widetilde T_1(s),
		\]
		because $s<0$. If $s\ge 0$, then $T_1(s)=\widetilde T_1(s)$. Hence
		\[
		T_1(s)\le \widetilde T_1(s)
		\qquad \forall s\in\mathbb R.
		\]
		For client $2$, if $s<0$, then $T_2(s)=\widetilde T_2(s)$. If $s\ge 0$, then
		\[
		T_2(s)
		=
		(1-\eta L)s-\eta\zeta_\star
		\le
		\left(1-\frac{\eta L}{2}\right)s-\eta\zeta_\star
		=
		\widetilde T_2(s).
		\]
		Therefore
		\[
		T_2(s)\le \widetilde T_2(s)
		\qquad \forall s\in\mathbb R.
		\]
		
		We next prove the comparison at communication times by induction. At $t=0$,
		the two systems agree. Suppose
		\[
		\bar s_{rK}\le \widetilde{\bar s}_{rK}.
		\]
		At time $rK$, all clients are synchronized, so
		\[
		s_{rK}^1=s_{rK}^2=\bar s_{rK},
		\qquad
		\widetilde s_{rK}^1=\widetilde s_{rK}^2=\widetilde{\bar s}_{rK}.
		\]
		Because $T_m$ and $\widetilde T_m$ are nondecreasing and
		$T_m\le \widetilde T_m$ pointwise, induction over the $K$ local steps gives
		\[
		s_{rK+k}^m\le \widetilde s_{rK+k}^m,
		\qquad
		m\in\{1,2\},\quad k=0,1,\dots,K.
		\]
		Averaging at the communication step gives
		\[
		\bar s_{(r+1)K}
		=
		\frac{s_{(r+1)K}^1+s_{(r+1)K}^2}{2}
		\le
		\frac{\widetilde s_{(r+1)K}^1+\widetilde s_{(r+1)K}^2}{2}
		=
		\widetilde{\bar s}_{(r+1)K}.
		\]
		This proves the first inequality \eqref{eq:comm_inequality} by induction.
		
		We now compute the comparison recursion. Let
		\[
		a\coloneqq 1-\frac{\eta L}{2},
		\qquad
		b\coloneqq 1-\eta L.
		\]
		Then $a,b\in[0,1]$. Starting from a synchronized value $s$, comparison
		client $1$ satisfies
		\[
		\widetilde s_{k+1}^1=b\,\widetilde s_k^1+\eta\zeta_\star.
		\]
		Thus after $K$ local steps,
		\[
		\widetilde s_K^1
		=
		b^K s+\eta\zeta_\star\sum_{j=0}^{K-1}b^j
		=
		b^K s+\frac{\zeta_\star}{L}(1-b^K).
		\]
		Similarly, comparison client $2$ satisfies
		\[
		\widetilde s_{k+1}^2=a\,\widetilde s_k^2-\eta\zeta_\star,
		\]
		and hence
		\[
		\widetilde s_K^2
		=
		a^K s-\eta\zeta_\star\sum_{j=0}^{K-1}a^j
		=
		a^K s-\frac{2\zeta_\star}{L}(1-a^K).
		\]
		Therefore the synchronized comparison iterate obeys
		\[
		\widetilde{\bar s}_{(r+1)K}
		=
		A_K\widetilde{\bar s}_{rK}-\Delta_K,
		\]
		where
		\[
		A_K\coloneqq \frac{a^K+b^K}{2},
		\qquad
		\Delta_K\coloneqq 
		\frac{\zeta_\star}{2L}\left(1+b^K-2a^K\right).
		\]
		Since $\widetilde{\bar s}_0=0$, iterating the affine recursion gives
		\[
		\widetilde{\bar s}_{rK}
		=
		-\Delta_K\sum_{j=0}^{r-1}A_K^j
		\qquad \forall r\in[0,R].
		\]

		We also have $A_K\in[0,1]$ because $a,b\in[0,1]$. Moreover,
		\[
		\Delta_K
		=
		\frac{\zeta_\star}{2L}\left(1+b^K-2a^K\right)\ge 0.
		\]
		To see this, define
		\[
		\phi(x)\coloneqq 1+(1-x)^K-2(1-x/2)^K,\qquad x\in[0,1].
		\]
		Then
		\[
		1+b^K-2a^K=\phi(\eta L).
		\]
		Since
		\[
		\phi(0)=0,
		\qquad
		\phi'(x)
		=
		K\left[(1-x/2)^{K-1}-(1-x)^{K-1}\right]\ge 0,
		\]
		we have $\phi(x)\ge 0$ for all $x\in[0,1]$. Hence $\Delta_K\ge 0$.
		
		Therefore
		\[
		\widetilde{\bar s}_{T}
		=
		-\Delta_K\sum_{r=0}^{R-1}A_K^r
		\le 0.
		\]
		By the comparison inequality at communication times, already proved above,
		\[
		\bar s_T\le \widetilde{\bar s}_T\le 0.
		\]
		
		It remains to lower bound the right-hand side. Define
		\[
		\theta\coloneqq K\eta L.
		\]
		We first show
		\[
		1+b^K-2a^K\gtrsim \min\{\theta^2,1\}.
		\]
		Let
		\[
		f(x)\coloneqq (1-x)^K.
		\]
		Then
		\[
		1+b^K-2a^K
		=
		f(0)+f(\eta L)-2f(\eta L/2).
		\]
		\textbf{Case 1.} If $\boldsymbol{\theta\le 1}$, then $\eta L\le 1/K$. Since
		\[
		f''(x)=K(K-1)(1-x)^{K-2},
		\]
		for all $x\in[0,\eta L]$ we have
		\[
		f''(x)
		\ge
		K(K-1)\left(1-\frac1K\right)^{K-2}
		\gtrsim K^2,
		\]
		where the last inequality holds for all $K\ge 2$. Now, let $h:=\eta L$ and $m:=\inf_{x\in[0,h]}f''(x)$. Since $f$ is
		$m$-strongly convex on $[0,h]$,
		\[
		f(0)+f(h)-2f(h/2)\ge \frac{m h^2}{4}.
		\]
		
		Using the above lower bound on $m$ we get,
		\[
		f(0)+f(\eta L)-2f(\eta L/2)
		\gtrsim
		K^2\eta^2L^2
		=
		\theta^2.
		\]
		\textbf{Case 2.} If $\boldsymbol{\theta\ge 1}$, define
		\[
		\phi(x)\coloneqq 1+(1-x)^K-2(1-x/2)^K.
		\]
		Then
		\[
		\phi'(x)
		=
		K\left[
		(1-x/2)^{K-1}-(1-x)^{K-1}
		\right]\ge 0
		\qquad x\in[0,1].
		\]
		Since $\eta L\ge 1/K$,
		\[
		1+b^K-2a^K
		=
		\phi(\eta L)
		\ge
		\phi(1/K)
		\gtrsim 1,
		\]
		where the last inequality follows from the previous case applied at
		$\theta=1$. 
		
		Therefore, combining the two cases above, we get,
		\[
		\boxed{
			\Delta_K
			\gtrsim
			\frac{\zeta_\star}{L}\min\{\theta^2,1\} .
		}
		\]

		Next, we lower bound the geometric sum. Since $A_K\in[0,1]$,
		\[
		\sum_{r=0}^{R-1}A_K^r
		=
		1+A_K+\cdots+A_K^{R-1}.
		\]
		We next lower bound the geometric sum. Recall that
		\[
		A_K=\frac{a^K+b^K}{2},
		\qquad
		a=1-\frac{\eta L}{2},
		\qquad
		b=1-\eta L.
		\]
		Since $\eta\le 1/L$, we have $a,b\in[0,1]$, and therefore $A_K\in[0,1]$.
		
		We claim that
		\[
		\sum_{r=0}^{R-1}A_K^r
		\gtrsim
		\begin{cases}
			\min\{R,1/\theta\}, & \theta\le 1,\\
			1, & \theta\ge 1.
		\end{cases}
		\]
		
		First suppose $\theta\ge 1$. Since the first term of the sum is $A_K^0=1$,
		we immediately have
		\[
		\sum_{r=0}^{R-1}A_K^r\ge 1.
		\]
		
		Now suppose $\theta\le 1$. We first upper bound $1-A_K$. Since
		\[
		1-A_K
		=
		1-\frac{a^K+b^K}{2}
		=
		\frac{1-a^K}{2}+\frac{1-b^K}{2},
		\]
		and since $1-x^K\le K(1-x)$ for every $x\in[0,1]$,
		which follows from
		\[
		1-x^K=(1-x)(1+x+\cdots+x^{K-1})\le K(1-x),
		\]
		we get
		\[
		1-A_K
		\le
		\frac{K(1-a)}{2}+\frac{K(1-b)}{2}.
		\]
		Substituting $1-a=\eta L/2$ and $1-b=\eta L$ gives
		\[
		1-A_K
		\le
		\frac{K\eta L}{4}+\frac{K\eta L}{2}
		=
		\frac{3}{4}K\eta L
		=
		\frac{3}{4}\theta
		\le
		\theta.
		\]
		
		We now use this to show that the first $\asymp \min\{R,1/\theta\}$ terms of
		the geometric sum are bounded below by a constant. If $\theta=0$, then
		$A_K=1$ and the sum is exactly $R$, so the claim is trivial. Hence assume
		$\theta>0$.
		
		Let
		\[
		N:=
		\min\left\{R,\left\lfloor \frac{1}{2\theta}\right\rfloor\right\}.
		\]
		If $N=0$, then $\theta>1/2$, and since $\theta\le 1$, the desired lower bound
		is only a universal constant:
		\[
		\min\{R,1/\theta\}\le 2.
		\]
		In that case,
		\[
		\sum_{r=0}^{R-1}A_K^r\ge 1\gtrsim \min\{R,1/\theta\}.
		\]
		So assume $N\ge 1$.
		
		For every $0\le r\le N-1$, Bernoulli's inequality gives
		\[
		A_K^r
		=
		(1-(1-A_K))^r
		\ge
		1-r(1-A_K).
		\]
		Using $1-A_K\le\theta$ and $r\le N-1\le 1/(2\theta)$, we get
		\[
		A_K^r
		\ge
		1-r\theta
		\ge
		\frac12.
		\]
		Therefore
		\[
		\sum_{r=0}^{R-1}A_K^r
		\ge
		\sum_{r=0}^{N-1}A_K^r
		\ge
		\frac{N}{2}.
		\]
		
		Finally, $N$ is within a universal constant of $\min\{R,1/\theta\}$, except
		in the already handled case $N=0$. More explicitly, when
		$\min\{R,1/(2\theta)\}\ge 1$, we have
		\[
		\left\lfloor \frac{1}{2\theta}\right\rfloor
		\ge
		\frac{1}{4\theta},
		\]
		and hence
		\[
		N
		=
		\min\left\{R,\left\lfloor \frac{1}{2\theta}\right\rfloor\right\}
		\gtrsim
		\min\{R,1/\theta\}.
		\]
		Thus
		\[
		\sum_{r=0}^{R-1}A_K^r
		\gtrsim
		\min\{R,1/\theta\}.
		\]
		
		Combining the drift and geometric-sum bounds gives
		\[
		-\widetilde{\bar s}_T
		=
		\Delta_K\sum_{r=0}^{R-1}A_K^r.
		\]
		We analyze this expression by splitting into four regimes depending on
		$\theta=K\eta L$ and $R\theta$.
		
		\paragraph{Case A: $\theta\le 1$ and $R\theta\le 1$.}
		In this regime,
		\[
		\Delta_K \gtrsim \frac{\zeta_\star}{L}\theta^2,
		\qquad
		\sum_{r=0}^{R-1}A_K^r \gtrsim R.
		\]
		Therefore
		\[
		-\widetilde{\bar s}_T
		\gtrsim
		\frac{\zeta_\star}{L}\,R\theta^2.
		\]
		
		\paragraph{Case B: $\theta\le 1$ and $R\theta\ge 1$.}
		In this regime,
		\[
		\Delta_K \gtrsim \frac{\zeta_\star}{L}\theta^2,
		\qquad
		\sum_{r=0}^{R-1}A_K^r \gtrsim \frac{1}{\theta}.
		\]
		Thus
		\[
		-\widetilde{\bar s}_T
		\gtrsim
		\frac{\zeta_\star}{L}\,\theta.
		\]
		
		\paragraph{Case C: $\theta\ge 1$.}
		Here we use the crude bounds
		\[
		\Delta_K \gtrsim \frac{\zeta_\star}{L},
		\qquad
		\sum_{r=0}^{R-1}A_K^r \ge 1,
		\]
		which yield
		\[
		-\widetilde{\bar s}_T
		\gtrsim
		\frac{\zeta_\star}{L}.
		\]
		
		\paragraph{Conclusion.}
		Combining the three regimes, we obtain
		\[
		-\widetilde{\bar s}_T
		\gtrsim
		\frac{\zeta_\star}{L}
		\min\{1,\theta,R\theta^2\}.
		\]
		Since $\bar s_T\le \widetilde{\bar s}_T\le 0$,
		\[
		|\bar s_T|\ge |\widetilde{\bar s}_T|
		\gtrsim
		\frac{\zeta_\star}{L}
		\min\{1,\theta,R\theta^2\}.
		\]
		Finally, because $\bar s_T\le 0$,
		\[
		g(\bar s_T)-g(0)
		=
		\frac{L}{4}\bar s_T^2.
		\]
		Therefore
		\[
		g(\bar s_T)-g(0)
		\gtrsim
		\frac{\zeta_\star^2}{L}
		\min\{1,\theta^2,R^2\theta^4\}.
		\]
		Using $\theta=K\eta L$, this is exactly
		\[
		g(\bar s_T)-g(0)
		\gtrsim
		\min\left\{
		\frac{\zeta_\star^2}{L},\;
		L K^2\eta^2\zeta_\star^2
		\min\{(RK\eta L)^2,1\}
		\right\}.
		\]
	\end{proof}
	
	\begin{lemma}[Auxiliary-coordinate lower bound]
		\label{lem:tau-zero-aux}
		Fix $\mu>0$ and define
		\[
		q(u)=\frac{\mu}{2}\left(u+\frac{B}{3}\right)^2.
		\]
		Suppose Local SGD is run on this coordinate with identical client objectives,
		initialized at $u_0^1=u_0^2=0$. If
		\[
		\eta\le \frac{1}{2\mu KR},
		\]
		then at time $T=KR$,
		\[
		q(\bar u_T)-\min_u q(u)\gtrsim \mu B^2.
		\]
	\end{lemma}
	
	\begin{proof}
		Since both clients have the same deterministic quadratic objective and start
		from the same point, their iterates agree at all times. Let
		\[
		e_t\coloneqq u_t+\frac{B}{3}.
		\]
		Then
		\[
		e_{t+1}=(1-\eta\mu)e_t,
		\qquad
		e_0=\frac{B}{3}.
		\]
		The assumption $\eta\le 1/(2\mu KR)$ gives
		\[
		0\le \eta\mu\le \frac{1}{2KR}.
		\]
		Hence
		\[
		|e_T|
		=
		(1-\eta\mu)^{KR}\frac{B}{3}
		\ge
		\left(1-\frac{1}{2KR}\right)^{KR}\frac{B}{3}
		\gtrsim B.
		\]
		Therefore
		\[
		q(\bar u_T)-\min_u q(u)
		=
		\frac{\mu}{2}e_T^2
		\gtrsim
		\mu B^2.
		\]
	\end{proof}
	
	\begin{lemma}[Guard-coordinate lower bound]
		\label{lem:tau-zero-guard}
		Define
		\[
		h(v)=\frac{H}{2}\left(v+\frac{B}{3}\right)^2.
		\]
		Suppose Local SGD is run on this coordinate with identical client objectives,
		initialized at $v_0^1=v_0^2=0$. If
		\[
		\eta>\frac{2}{H},
		\]
		then at time $T=KR$,
		\[
		h(\bar v_T)-\min_v h(v)\gtrsim HB^2.
		\]
	\end{lemma}
	
	\begin{proof}
		Since both clients have the same deterministic quadratic objective and start
		from the same point, their iterates agree at all times. Let
		\[
		w_t\coloneqq v_t+\frac{B}{3}.
		\]
		Then
		\[
		w_{t+1}=(1-\eta H)w_t,
		\qquad
		w_0=\frac{B}{3}.
		\]
		If $\eta>2/H$, then $|1-\eta H|>1$. Hence
		\[
		|w_T|\ge |w_0|=\frac{B}{3}.
		\]
		Therefore
		\[
		h(\bar v_T)-\min_v h(v)
		=
		\frac{H}{2}w_T^2
		\gtrsim
		HB^2.
		\]
	\end{proof}
	
	\begin{theorem}[$\tau=0$ lower bound for Local SGD]
		\label{thm:tau-zero-lower}
		There exist two deterministic convex client objectives
		\(
		F_1,F_2:\mathbb{R}^3\to\mathbb{R}
		\)
		in the class $\ppp^{H,B,0}_{\zeta_\star,0,0}$, such that vanilla Local SGD with $M=2$ clients, $K\geq 2$ local steps per round, arbitrary step-size $\eta>0$ and initialized at $x_0^1=x_0^2=0$, must have its final averaged iterate
		\[
		\bar x_T\coloneqq \frac12(x_T^1+x_T^2)
		\]
		such that it satisfies
		\[
		F(\bar x_T)-F(x^\star)
		\gtrsim
		\min\left\{
		HB^2,\;
		\frac{\zeta_\star^2}{H},\;
		\frac{(H\zeta_\star^2B^4)^{1/3}}{R^{2/3}}
		\right\}.
		\]
	\end{theorem}
	
	\begin{proof}
		Set
		\[
		L\coloneqq \frac{H}{2}.
		\]
		We construct a deterministic instance on $\mathbb R^3$, writing
		\[
		x=(s,u,v).
		\]
		Use the active objectives $G_1,G_2$ from Lemma~\ref{lem:tau-zero-active} with
		this value of $L$. Let $\mu>0$ be chosen later, and define
		\[
		q(u)=\frac{\mu}{2}\left(u+\frac{B}{3}\right)^2,
		\qquad
		h(v)=\frac{H}{2}\left(v+\frac{B}{3}\right)^2.
		\]
		Define the two client objectives
		\[
		F_1(s,u,v)=G_1(s)+q(u)+h(v),
		\]
		and
		\[
		F_2(s,u,v)=G_2(s)+q(u)+h(v).
		\]
		Then
		\[
		F(s,u,v)\coloneqq \frac{F_1(s,u,v)+F_2(s,u,v)}2
		=
		g(s)+q(u)+h(v).
		\]
		The minimizer is
		\[
		x^\star=
		\left(0,-\frac{B}{3},-\frac{B}{3}\right),
		\]
		so
		\[
		\norm{x^\star}
		=
		\frac{\sqrt 2 B}{3}
		\le B.
		\]
		The initialization is
		\[
		x_0^1=x_0^2=(0,0,0).
		\]
		
		The instance is deterministic, so $\sigma=0$. The smoothness of the active
		coordinate is $L=H/2$, the smoothness of the auxiliary coordinate is $\mu$,
		and the smoothness of the guard coordinate is $H$. We will choose $\mu\le H$,
		so each client objective is $H$-smooth, convex, and satisfies the bounded optima assumption in \eqref{ass:bounded_optima}.
		
		At $x^\star$,
		\[
		\nabla F_1(x^\star)=(-\zeta_\star,0,0),
		\qquad
		\nabla F_2(x^\star)=(\zeta_\star,0,0).
		\]
		Therefore
		\[
		\frac12\sum_{m=1}^2 \norm{\nabla F_m(x^\star)}^2
		=
		\zeta_\star^2.
		\]
		Moreover, the two clients differ only by the linear terms
		$\mp \zeta_\star s$. Hence, for all $x,y$,
		\[
		\nabla F_1(x)-\nabla F_1(y)
		=
		\nabla F_2(x)-\nabla F_2(y)
		=
		\nabla F(x)-\nabla F(y).
		\]
		Thus Assumptions~\ref{ass:tau} holds with $\tau=0$.
		
		Now coming back to the choice of $\mu$, define
		\[
		\mathcal T\coloneqq 
		\min\left\{
		HB^2,\;
		\frac{\zeta_\star^2}{H},\;
		\frac{(H\zeta_\star^2B^4)^{1/3}}{R^{2/3}}
		\right\},
		\qquad
		\mu\coloneqq \frac{\mathcal T}{4B^2}.
		\]
		Since $\mathcal T\le HB^2$, we have
		\[
		\mu\le \frac{H}{4}\le H.
		\]
		Thus, all of this shows that the construction lies in $\ppp^{H,B,0}_{\zeta_\star,0,0}$.
		
		We now split into three step-size regimes.
		
		\medskip
		\noindent
		\textbf{Small step-sizes.}
		If
		\[
		\eta\le \frac{1}{2\mu KR},
		\]
		then Lemma~\ref{lem:tau-zero-aux} gives
		\[
		F(\bar x_T)-F(x^\star)
		\ge
		q(\bar u_T)-\min_u q(u)
		\gtrsim
		\mu B^2
		=
		\frac{\mathcal T}{4}
		\gtrsim
		\mathcal T.
		\]
		
		\medskip
		\noindent
		\textbf{Large step-sizes.}
		If
		\[
		\eta>\frac{2}{H},
		\]
		then Lemma~\ref{lem:tau-zero-guard} gives
		\[
		F(\bar x_T)-F(x^\star)
		\ge
		h(\bar v_T)-\min_v h(v)
		\gtrsim
		HB^2
		\ge
		\mathcal T.
		\]
		
		\medskip
		\noindent
		\textbf{Intermediate step-sizes.}
		It remains to consider
		\[
		\frac{1}{2\mu KR}<\eta\le \frac{2}{H}.
		\]
		Since $L=H/2$, this is
		\[
		\frac{1}{2\mu KR}<\eta\le \frac1L.
		\]
		Let
		\[
		\eta_0\coloneqq \frac{1}{2\mu KR}.
		\]
		The active-coordinate lower bound in Lemma~\ref{lem:tau-zero-active} is
		nondecreasing in $\eta$ on $[0,1/L]$. Since $\eta>\eta_0$, it suffices to
		evaluate the bound at $\eta_0$.
		
		We have
		\[
		RK\eta_0 L=\frac{L}{2\mu}.
		\]
		Because $\mu\le H/4=L/2$, we have
		\[
		RK\eta_0 L\ge 1.
		\]
		Therefore Lemma~\ref{lem:tau-zero-active} gives
		\[
		g(\bar s_T)-g(0)
		\gtrsim
		\min\left\{
		\frac{\zeta_\star^2}{L},\;
		L K^2\eta_0^2\zeta_\star^2
		\right\}.
		\]
		Substituting $\eta_0=1/(2\mu KR)$,
		\[
		L K^2\eta_0^2\zeta_\star^2
		=
		\frac{L\zeta_\star^2}{4\mu^2R^2}.
		\]
		Thus
		\[
		g(\bar s_T)-g(0)
		\gtrsim
		\min\left\{
		\frac{\zeta_\star^2}{L},\;
		\frac{L\zeta_\star^2}{\mu^2R^2}
		\right\}.
		\]
		The first term satisfies
		\[
		\frac{\zeta_\star^2}{L}
		=
		\frac{2\zeta_\star^2}{H}
		\ge 2\mathcal T.
		\]
		For the second term, using $\mu=\mathcal T/(4B^2)$ and $L=H/2$,
		\[
		\frac{L\zeta_\star^2}{\mu^2R^2}
		=
		\frac{16L\zeta_\star^2B^4}{\mathcal T^2R^2}
		=
		\frac{8H\zeta_\star^2B^4}{\mathcal T^2R^2}.
		\]
		Since
		\[
		\mathcal T
		\le
		\frac{(H\zeta_\star^2B^4)^{1/3}}{R^{2/3}},
		\]
		we have
		\[
		\mathcal T^3
		\le
		\frac{H\zeta_\star^2B^4}{R^2}.
		\]
		Therefore
		\[
		\frac{8H\zeta_\star^2B^4}{\mathcal T^2R^2}
		\ge
		8\mathcal T.
		\]
		Thus the intermediate regime also gives
		\[
		F(\bar x_T)-F(x^\star)
		\ge
		g(\bar s_T)-g(0)
		\gtrsim
		\mathcal T.
		\]
		
		Combining the three regimes gives
		\[
		F(\bar x_T)-F(x^\star)\gtrsim \mathcal T,
		\]
		which proves the theorem.
	\end{proof}
	
	\section{Phase Diagram: Comparison of New Lower and Upper Bounds}\label{app:phase}
	
	In this section, we calculate in which regimes in the $\zeta_\star$-$\tau$ diagram our new upper and lower bounds match, and when a gap remains.
	
	To facilitate comparison, we divide the $\zeta_\star$-$\tau$ diagram into multiple segments and identify the dominant terms in the upper and lower bounds. In this section, we consider only a practical setting $\zeta_\star\le HB$ in the noiseless regime $\sigma=0$.
	
	In the aforementioned setting, our upper and lower bounds can be written as follows
	\begin{align*}
		\ee[F(\hat x_t) - F^\star] &\lesssim \min\left\{R_1(\tau,\zeta_\star), R_2(\tau,\zeta_\star), R_3(\tau,\zeta_\star)\right\}, \quad \text{where}\\
		R_1(\tau,\zeta_\star) &= \frac{HB^2}{KR} + \frac{\tau B^2}{\sqrt{R}} + \frac{\zeta_\star B}{\sqrt{R}},
		\qquad
		R_2(\tau,\zeta_\star) = \frac{HB^2}{KR} + \frac{(H\tau^2)^{1/3}B^2}{R^{2/3}} + \frac{(H\zeta_\star^2B^4)^{1/3}}{R^{2/3}}, 
		\\
		R_3(\tau,\zeta_\star) &= \frac{HB^2}{R} + \frac{(H\zeta_\star^2B^4)^{1/3}}{R^{2/3}},\quad \text{and}\\
		\ee[F(\hat x_t) - F^\star] &\gtrsim \frac{HB^2}{KR} + \min\left\{\frac{HB^2}{R}, \frac{\tau B^2}{\sqrt{R}}\right\} + \min\left\{HB^2,\frac{\zeta_\star^2}{H}, \frac{(H\zeta_\star^2B^4)^{1/3}}{R^{2/3}}\right\},
	\end{align*}
	for some point $\hat{x}_t$ (different for lower and upper bounds).
	
	We use the following notation $\alpha=\frac{HB^2}{R}, x_0=\frac{\tau\sqrt{R}}{H},$ and $y_0 = \frac{\zeta_\star\sqrt{R}}{HB}$, then the bounds can be written as follows
	\begin{align*}
		U(x_0,y_0) &= \min\left\{
		\frac{\alpha}{K} + \alpha x_0 + \alpha y_0, 
		\frac{\alpha}{K}+\alpha x_0^{2/3} + \alpha y_0^{2/3},   \alpha + \alpha y_0^{2/3}\right\},\\
		L(x_0,y_0) &= \frac{\alpha}{K} + \min\left\{\alpha, \alpha x_0\right\} + \min\left\{\alpha R, \alpha y_0^2, \alpha y_0^{2/3}\right\},
	\end{align*}
	where $U(x_0,y_0)$ is the upper bound and $L(x_0,y_0)$ is the lower bound.
	
	\subsection{Analysis of Upper Bound's Dominant Term}\label{sec:upper_bound_diagram}
	
	\paragraph{Case 1: $0\le x_0 \le \frac{1}{K}, 0 \le y_0 \le \frac{1}{K}$.}
	In this case, we have $R_1, R_2 \ge \frac{\alpha}{K}$ and $R_3 \ge \alpha$. So $U(x_0,y_0) \ge \frac{\alpha}{K}$. Moreover, $R_1 = \frac{\alpha}{K} + \alpha x_0 + \alpha y_0 \le \frac{3\alpha}{K},$ so in this regime $U(x_0,y_0) \sim \frac{\alpha}{K}$.
	
	\paragraph{Case 2: $\frac{1}{K} \le x_0 \le 1, y_0 \le x_0$.} 
	In this case, we have $$R_1 = \frac{\alpha}{K} + \alpha x_0 + \alpha y_0 \ge \alpha x_0, \quad R_2 = \frac{\alpha}{K} + \alpha x_0^{2/3} + \alpha y_0^{2/3} \ge \alpha x_0^{2/3} \ge \alpha x_0,$$
	and
	$$R_3 = \alpha + \alpha y_0^{2/3} \ge \alpha \ge \alpha x_0.$$
	In other words, each term is at least $\alpha x_0,$ so $Y(x_0,y_0) \ge \alpha x_0$. At the same time, $R_1 \le 3\alpha x_0$. Therefore, in this regime $U(x_0,y_0) \sim \alpha x_0$.
	
	\paragraph{Case 3: $\frac{1}{K} \le y_0 \le 1, x_0 \le y_0$.}
	
	This case is symmetric to Case 2, when swapping $x_0$ and $y_0$. Therefore, $U(x_0,y_0) \sim \alpha y_0$.

	\paragraph{Case 4: $x_0\ge 1, y_0 \le 1$.}
	In this case, each of the rates $R_1,R_2,R_3$ is at least $\alpha$. Indeed, we have
	$$
	R_1 = \frac{\alpha}{K} + \alpha x_0 + \alpha y_0 \ge \alpha x_0 \ge \alpha, 
	\qquad
	R_2 = \frac{\alpha}{K} + \alpha x_0^{2/3} + \alpha y_0^{2/3} \ge \alpha x_0^{2/3} \ge \alpha,
	$$
	and 
	$$
	R_3 = \alpha + \alpha y_0^{2/3} \ge \alpha.
	$$
	Therefore, $U(x_0,y_0) \ge \alpha$ as well. At the same time, $R_3 \le \alpha + \alpha = 2\alpha$. Thus, we obtain that in this regime $U(x_0,y_0) \sim \alpha$.
	
	\paragraph{Case 5: $1 \le y_0 \le \sqrt{R}$.}
	
	Finally, in this regime, we have
	$$
	R_1 = \frac{\alpha}{K} + \alpha x_0 + \alpha y_0 \ge \alpha y_0 \ge \alpha y_0^{2/3}, \qquad
	R_2 = \frac{\alpha}{K} + \alpha x_0^{2/3} + \alpha y_0^{2/3} \ge \alpha y_0^{2/3},
	$$
	and
	$$R_3 = \alpha + \alpha y_0^{2/3} \ge \alpha y_0^{2/3}.$$
	At the same time, $R_3 \le 2\alpha y_0^{2/3}$. Thus, $U(x_0,y_0) \sim \alpha y_0^{2/3}.$
	
	\paragraph{Final split.}
	
	Summarizing the above 5 cases together, the upper bound has the following phase diagram
	$$
	U(x_0,y_0) = \begin{cases}
		\frac{\alpha}{K}, &\text{if } 0\le x_0\le \frac{1}{K}, 0\le y_0 \le \frac{1}{K},\\
		\alpha x_0, & \text{if } \frac{1}{K}\le x_0 \le 1, y_0 \le x_0,\\
		\alpha y_0, & \text{if } \frac{1}{K} \le y_0 \le 1, x_0\le y_0,\\
		\alpha, &\text{if } x_0\ge 1, y_0 \le 1,\\
		\alpha y_0^{2/3}, &\text{if } 1\le y_0 \le \sqrt{R}.
	\end{cases}
	$$
	In the original notation, the phase diagram of the upper bound has the following form
	$$
	U(\tau,\zeta_\star) = \begin{cases}
		\frac{HB^2}{KR}, &\text{if } 0\le \tau\le \frac{H}{K\sqrt{R}}, 0\le \zeta_\star \le \frac{HB}{K\sqrt{R}},\\
		\frac{\tau B^2}{\sqrt{R}}, & \text{if } \frac{H}{K\sqrt{R}}\le\tau \le \frac{H}{\sqrt{R}}, \zeta_\star \le \tau B,\\
		\frac{\zeta_\star B}{\sqrt{R}}, & \text{if } \frac{HB}{K\sqrt{R}} \le\zeta_\star \le \frac{HB}{\sqrt{R}}, \tau B\le\zeta_\star,\\
		\frac{HB^2}{R}, &\text{if } \tau \ge \frac{H}{\sqrt{R}}, \zeta_\star \le \frac{HB}{\sqrt{R}},\\
		\frac{(H\zeta_\star B^4)^{1/3}}{R^{2/3}}, &\text{if } \frac{HB}{\sqrt{R}}\le\zeta_\star \le HB.
	\end{cases}
	$$
	
	\subsection{Analysis of Lower Bound's Dominant Term}\label{sec:lower_bound_diagram}

	Note that the lower bound has the following form:
	$$
	L(x_0,y_0) = \frac{\alpha}{K} + \min\left\{\alpha, \alpha x_0\right\} + \min\left\{\alpha R, \alpha y_0^2, \alpha y_0^{2/3}\right\}.
	$$
	
	\paragraph{Case 1: $0 \le x_0 \le \frac{1}{K}, 0 \le y_0 \le \frac{1}{\sqrt{K}}$.}
	
	In this case, we have $x_0 \le \frac{1}{K}$ and $y_0^2 \le \frac{1}{K}$. Therefore, $y_0^2 \le y_0^{2/3}$, and the lower bound is 
	$$
	L(x_0,y_0) = \frac{\alpha}{K} + \alpha x_0 + \alpha y_0^2.
	$$
	This implies that 
	$$
	\frac{\alpha}{K} \le L(x_0,y_0) \le \frac{3\alpha}{K}
	$$
	in this regime, i.e., $L(x_0,y_0) \sim \frac{\alpha}{K}$.
	
	\paragraph{Case 2: $\frac{1}{K} \le x_0 \le 1, y_0\le \sqrt{x_0}$.}
	Note that in this regime we have $y_0^2 \le x_0\le 1$ and $y_0^2 \le y_0^{2/3}$. Thus, the lower bound is 
	$$
	L(x_0,y_0) = \frac{\alpha}{K} + \alpha x_0 + \alpha y_0^2.
	$$
	This implies that 
	$$
	\alpha x_0 \le L(x_0,y_0) \le 3\alpha x_0
	$$
	in this regime, i.e., $L(x_0,y_0) \sim \alpha x_0$.
	
	\paragraph{Case 3: $\frac{1}{\sqrt{K}} \le y_0 \le 1, x_0 \le y_0^2$.}
	
	In this regime, we have $\frac{1}{K} \le y_0^2 \le y_0^{2/3} \ge 1$ and $x_0 \le y_0^2\ge1$. This implies that the lower bound is 
	$$
	L(x_0,y_0) = \frac{\alpha}{K} + \alpha x_0 + \alpha y_0^2.
	$$
	Therefore, we have 
	$$
	\alpha y_0^2 \le L(x_0,y_0) \le 3\alpha y_0^2,
	$$
	i.e., $L(x_0,y_0) \sim \alpha y_0^2$.
	
	\paragraph{Case 4: $1 \le x_0, y_0 \le 1$.}
	
	Since $y_0^2\le y_0^{2/3} \le 1$ in this regime, then the lower bound is 
	$$
	L(x_0,y_0) = \frac{\alpha}{K} + \alpha + \alpha y_0^2.
	$$
	This implies that 
	$$
	\alpha \le L(x_0,y_0) \le 3\alpha,
	$$
	i.e., $L(x_0,y_0) \sim \alpha$.
	
	\paragraph{Case 5: $1 \le y_0 \le \sqrt{R}$.}
	Note that now $y_0^{2/3} \le y_0^2$. Therefore, in this regime, we have that the lower bound is 
	$$
	L(x_0,y_0) = \frac{\alpha}{K} + \min\{\alpha,\alpha x_0\} + \alpha y_0^{2/3}.
	$$
	Since $\min \{\alpha,\alpha x_0\} \le \alpha \le \alpha y_0^{2/3}$ and $\frac{\alpha}{K}\le \alpha \le \alpha y_0^{2/3}$, then we have 
	$$
	\alpha y_0^{2/3} \le L(x_0,y_0) \le 3\alpha y_0^{2/3},
	$$
	i.e., $L(x_0,y_0) \sim 3\alpha y_0^{2/3}$ in this regime.
	
	\paragraph{Final split.}
	Summarizing the above 5 cases together, the lower bound has the following phase diagram
	$$
	L(x_0,y_0) = \begin{cases}
		\frac{\alpha}{K}, &\text{if } 0\le x_0 \le \frac{1}{K}, 0 \le y_0 \le \frac{1}{\sqrt{K}},\\
		\alpha x_0, &\text{if } \frac{1}{K} \le x_0 \le 1, y_0 \le \sqrt{x_0},\\
		\alpha y_0^2, &\text{if } \frac{1}{\sqrt{K}} \le y_0 \le 1, x_0 \le y_0^2,\\
		\alpha, &\text{if } 1\le x_0, y_0 \le 1,\\
		\alpha y_0^{2/3}, &\text{if } 1 \le y_0 \le \sqrt{R}.
	\end{cases}
	$$
	In the original notation, the phase diagram for the lower bound has the following form
	$$
	L(\tau,\zeta_\star) = \begin{cases}
		\frac{HB^2}{KR}, &\text{if } 0\le \tau \le \frac{H}{K\sqrt{R}}, 0 \le \zeta_\star \le \frac{HB}{\sqrt{KR}},\\
		\frac{\tau B^2}{\sqrt{R}}, &\text{if } \frac{H}{K\sqrt{R}} \le \tau \le \frac{H}{\sqrt{R}}, \zeta_\star^2\le \frac{HB^2\tau}{\sqrt{R}},\\
		\frac{\zeta_\star^2}{H}, &\text{if } \frac{HB}{\sqrt{RK}} \le \zeta_\star \le \frac{HB}{\sqrt{R}}, \frac{HB^2\tau}{\sqrt{R}} \le \zeta_\star^2,\\
		\frac{HB^2}{R}, &\text{if } \frac{H}{\sqrt{R}}\le\tau, \zeta_\star \le \frac{HB}{\sqrt{R}},\\
		\frac{(H\zeta_\star^2B^4)^{1/3}}{R^{2/3}}, &\text{if } \frac{HB}{\sqrt{R}} \le \zeta_\star \le HB.
	\end{cases}
	$$

	\subsection{Combined Phase Diagram}
	
	Knowing the phase diagram for both the lower and upper bounds, we can merge them into one to show the areas where the lower and upper bounds match (up to numerical constants) or do not. We have
	$$
	(U(x_0,y_0), L(x_0,y_0)) = \begin{cases}
		(\frac{\alpha}{K},\frac{\alpha}{K}), &\text{if } 0\le x_0 \le \frac{1}{K}, 0 \le y_0 \le \frac{1}{K},\\
		(\alpha x_0,\alpha x_0) &\text{if } \frac{1}{K} \le x_0 \le 1, y_0 \le x_0,\\
		(\alpha,\alpha), &\text{if } 1 \le x_0 \le \sqrt{R}, 0 \le y_0 \le 1,\\
		(\alpha y_0^{2/3},\alpha y_0^{2/3}), &\text{if } 1 \le y_0 \le \sqrt{R},\\
		(\alpha y_0, \max\{\frac{\alpha}{K},\alpha y_0^2,\alpha x_0^2\}), &\text{if } x_0 \le y_0, \frac{1}{K}\le y_0 \le 1.
	\end{cases}
	$$
	In the notation $\tau,\zeta_\star$ the bounds become 
	$$
	(U(\tau,\zeta_\star), L(\tau,\zeta_\star)) = \begin{cases}
		(\frac{HB^2}{KR},\frac{HB^2}{KR}), &\text{if } 0\le \tau \le \frac{H}{K\sqrt{R}}, 0 \le y_0 \le \frac{HB}{K\sqrt{R}},\\
		(\frac{\tau B^2}{\sqrt{R}}, \frac{\tau B^2}{\sqrt{R}}) &\text{if } \frac{H}{K\sqrt{R}} \le \tau \le \frac{H}{\sqrt{R}}, \zeta_\star \le \tau B,\\
		(\frac{HB^2}{R},\frac{HB^2}{R}), &\text{if } \frac{H}{\sqrt{R}} \le \tau \le H, 0 \le \zeta_\star \le \frac{HB}{\sqrt{R}},\\
		(\frac{(H\zeta_\star^2B^4)^{1/3}}{R^{2/3}},\frac{(H\zeta_\star^2B^4)^{1/3}}{R^{2/3}}), &\text{if } \frac{HB}{\sqrt{R}} \le \zeta_\star \le HB,\\
		(\frac{\zeta_\star B}{\sqrt{R}}, \max\{\frac{HB^2}{KR},\frac{(H\zeta_\star^2B^4)^{1/3}}{R^{2/3}},\frac{(H\tau^2)^{1/3}B^2}{R^{2/3}}\}), &\text{if } \tau B \le \zeta_\star, \frac{HB}{K\sqrt{R}}\le\zeta_\star \le \frac{HB}{\sqrt{R}}.
	\end{cases}
	$$
	As we can see, the only regime when our lower and upper bounds do not match is $\tau B \le \zeta_\star, \frac{HB}{K\sqrt{R}}\le\zeta_\star \le \frac{HB}{\sqrt{R}}$.

	\begin{figure}[!ht]
		\centering
		\begin{tabular}{cc}
			\includegraphics[width=0.45\linewidth]{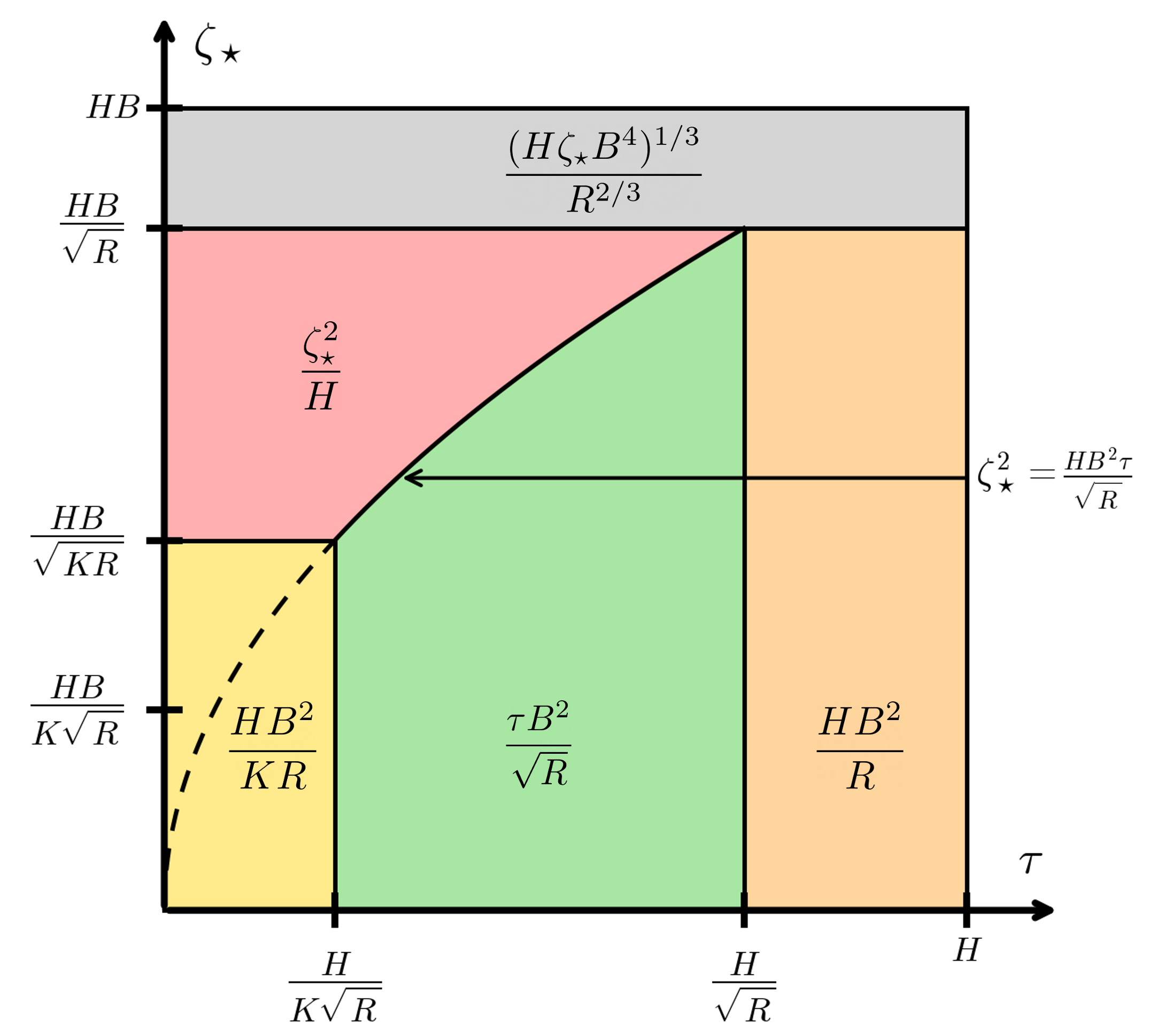}
			
			&
			
			\includegraphics[width=0.44\linewidth]{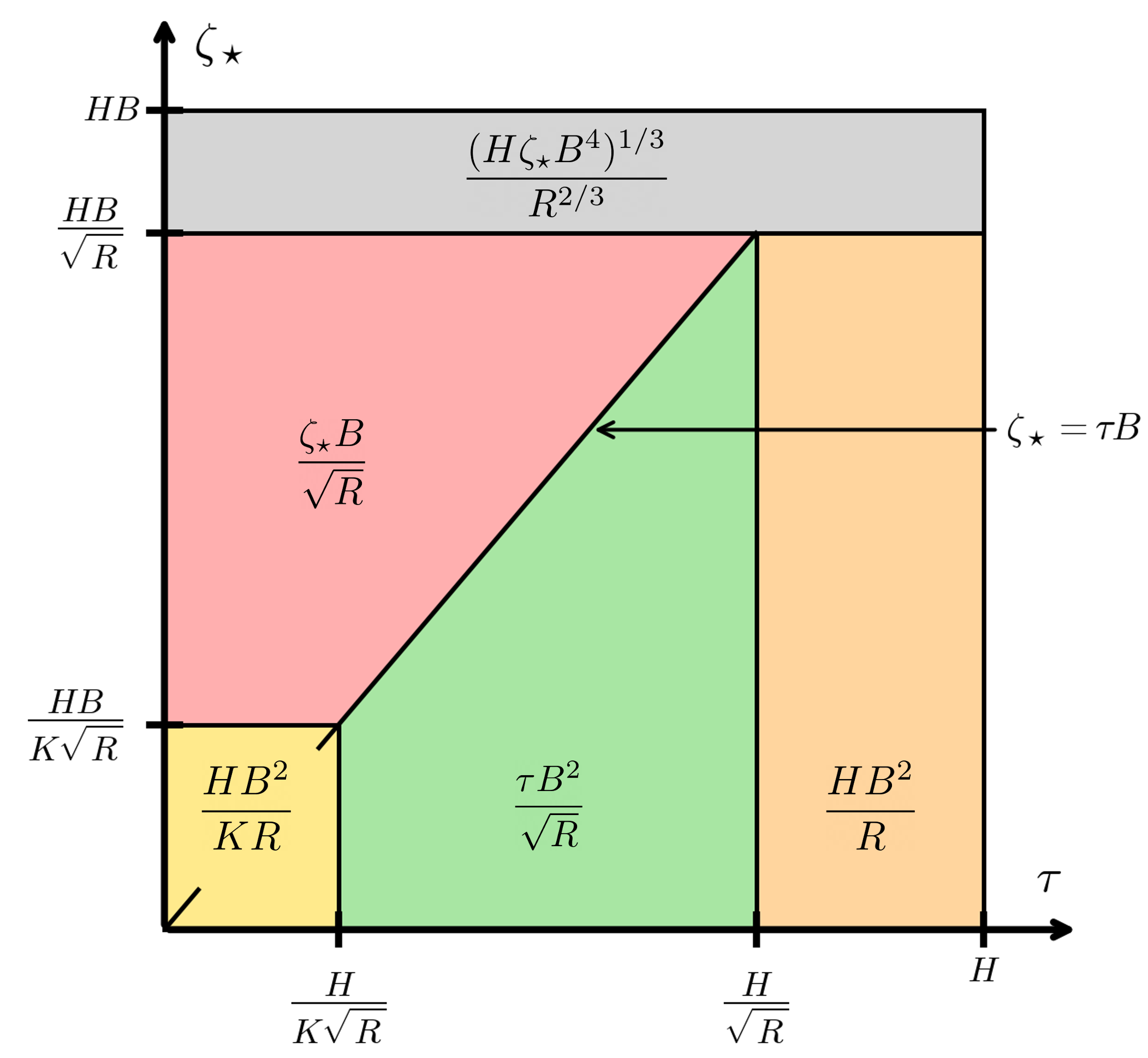}
			
		\end{tabular}
		\caption{Phase diagrams of the lower (left) and upper (right) bounds provided in this work: see \Cref{tab:rates} for the exact rates and \Cref{sec:upper_bound_diagram} and \ref{sec:lower_bound_diagram} for derivations of the diagrams. The merged phase diagram is reported in \Cref{fig:lower_vs_upper}.
		}
		\label{fig:lower_bound_upper_bound}
	\end{figure}

	\section{Phase Diagram: Comparison of Prior Lower and Upper Bounds}\label{app:phase_diagrams_prior_works}
	
	In this section, we provide phase diagrams in $\zeta_\star$-$\tau$ diagram for the previous upper bound of \citet{koloskova2020unified} and lower bound of \citet{patel2025revisiting}. Again, we consider the noiseless regime $\sigma=0$, and make a restriction $\zeta_\star\le HB$.
	
	In the aforementioned setting, the upper bound of \citet{koloskova2020unified} is 
	\[
	\ee[F(\hat{x}_t) -F^\star] \lesssim \frac{HB^2}{R} + \frac{(H\zeta_\star^2B^4)^{1/3}}{R^{2/3}},
	\]
	and the lower bound of \citet{patel2025revisiting} is 
	\[
	\ee[F(\hat{x}_t) - F^\star] \gtrsim \frac{HB^2}{KR}
	+ \frac{\tau B^2}{R} + \frac{(\tau \zeta_\star^2 B^4)^{1/3}}{R^{2/3}},
	\]
	for some $\hat{x}_t$.
	
	We use the following notation $x_0=\frac{\tau}{H},y_0 = \frac{\zeta_\star\sqrt{R}}{HB}$, and $\alpha=\frac{HB^2}{R}$, then the bounds become 
	\[
	U_{K} = \alpha(1+y_0^{2/3}), \quad \text{and} 
	\quad 
	L_{P} = \alpha(1/K + x_0 + x_0^{1/3}y_0^{2/3}).
	\]
	
	\subsection{Analysis of Dominant Term of the Upper Bound in \citet{koloskova2020unified}}\label{sec:upper_bound_koloskova_diagram}
	
	The phase diagram for the upper bound is simple since it consists of only two terms, and it does not depend on $\tau$ at all. The switching point between the two terms is when they are equal: 
	\[
	\alpha =\alpha y_0^{2/3}\Longleftrightarrow y_0 = 1.
	\]
	Therefore, we have the following phase diagram for the upper bound
	\[
	U_K(x_0,y_0) = \begin{cases}
		\alpha, &\text{if } y_0\le 1,\\
		\alpha y_0^{2/3}, &\text{if } y_0 \ge 1.
	\end{cases}
	\]
	In the standard notation $\zeta_\star$-$\tau$, we have 
	\[
	U_K(\tau,\zeta_\star) = \begin{cases}
		\frac{HB^2}{R}, &\text{if } \zeta_\star \le \frac{HB}{\sqrt{R}},\\
		\frac{(H\zeta_\star^2B^4)^{1/3}}{R^{2/3}}, &\text{if } \zeta_\star \ge \frac{HB}{\sqrt{R}}.
	\end{cases}
	\]
	
	\subsection{Analysis of Dominant Term of the Lower Bound in \citet{patel2025revisiting}}\label{sec:lower_bound_patel_diagram}
	
	\paragraph{Case 1: $0\le x_0 \le \frac{1}{K}, x_0y_0^2\le \frac{1}{K^3}$.}
	
	In this case, we have $x_0 \le \frac{1}{K}$ and $x_0^{1/3}y_0^{2/3} = (x_0y_0^2)^{1/3} \le \frac{1}{K}$. Therefore, we have that the lower bound satisfies
	\[
	\frac{\alpha}{K} \le L_P(x_0,y_0) = \frac{\alpha}{K} + \alpha x_0 + \alpha x_0^{1/3}y_0^{2/3} \le \frac{3\alpha}{K},
	\]
	i.e., $L_P \sim \frac{\alpha}{K}$ in this regime.
	
	\paragraph{Case 2: $\frac{1}{K}\le x_0, y_0\le x_0$.}
	
	In this regime, we have $x_0 \ge \frac{1}{K}$ and $x_0^{1/3}y_0^{2/3} \le x_0^{1/3} x_0^{2/3} = x_0$. Therefore, we have that the lower bound satisfies
	\[
	\alpha x_0 \le L_P(x_0,y_0) = \frac{\alpha}{K} + \alpha x_0 + \alpha x_0^{1/3}y_0^{2/3} \le 3\alpha x_0,
	\]
	i.e., $L_P \sim \alpha x_0$ in this regime.
	
	\paragraph{Case 3: $x_0 \le y_0, \frac{1}{K^3} \le x_0y_0^2$.}
	
	In this regime, we have that $x_0 \le x_0^{1/3}y_0^{2/3}$ and $\frac{1}{K} \le (x_0y_0^2)^{1/3} =x_0^{1/3}y_0^{2/3}$. Therefore, we have that the lower bound satisfies
	\[
	\alpha x_0y_0^{2/3} \le L_P(x_0,y_0) = \frac{\alpha}{K} + \alpha x_0 + \alpha x_0^{1/3}y_0^{2/3} \le 3\alpha x_0^{1/3}y_0^{2/3},
	\]
	i.e., $L_P \sim x_0^{1/3}y_0^{2/3}$. 
	
	\paragraph{Final split.}
	
	Summarizing the above 3 cases together, the lower bound has the following phase diagram
	\[
	L_P(x_0,y_0) = \begin{cases}
		\frac{\alpha}{K}, &\text{if } 0\le x_0\le \frac{1}{K}, x_0y_0^2 \le \frac{1}{K^3},\\
		\alpha x_0, &\text{if } \frac{1}{K} \le x_0, y_0 \le x_0,\\
		\alpha x_0^{1/3}y_0^{2/3}, &\text{if } x_0 \le y_0, \frac{1}{K^3} \le x_0y_0^2.
	\end{cases}
	\]
	In the original notation $\zeta_\star$-$\tau$, the diagram has the following form
	\[
	L_P(\tau,\zeta_\star) = \begin{cases}
		\frac{HB^2}{KR}, &\text{if } 0\le \tau \le \frac{H}{K}, \tau\zeta_\star^2 \le \frac{H^3B^2}{K^3R},\\
		\frac{\tau B^2}{R}, &\text{if } \frac{H}{K} \le \tau, \zeta_\star \le \frac{B\tau}{\sqrt{R}},\\
		\frac{(\tau\zeta_\star^2B^4)^{1/3}}{R^{2/3}}, &\text{if } \frac{B\tau}{\sqrt{R}} \le \zeta_\star, \frac{H^3B^2}{K^3R} \le \tau\zeta_\star^2.
	\end{cases}
	\]

\section{Fresh-sample SGD Lower Bound in the Heterogeneous Smoothness Regime}
\label{app:fresh_sample_sgd_shell_lower_bound}

In this appendix, we prove the lower bound stated informally in
\Cref{thm:fresh_sgd_heterogeneous_lb_informal}.  The proof has two independent
ingredients.  First, we construct a deterministic rare-curvature quadratic block
that gives the optimization lower bound \(\hat H B^2/(TM)\) while satisfying
\[
    H \cong \tau \cong \frac{\hat H}{\sqrt M}.
\]
Second, we add the standard stochastic and first-order heterogeneity lower-bound
block on orthogonal coordinates to obtain the noise term.

\subsection{A rare-curvature block}

We begin with the deterministic block.  Let \(M\ge2\), \(T\ge1\), and let
\(\lambda>0\).  Later we will set \(\hat H=M\lambda\), and the lower bound will
scale as \(\lambda B^2/T=\hat H B^2/(TM)\).

Let \(u,v\) be two orthonormal vectors.  Define the positive semidefinite
matrix \(A\) by
\[
    Au=\lambda u,
    \qquad
    Av=\mu v,
    \qquad
    \mu:=\frac{\lambda}{12T}.
\]
On this block, define
\[
    R_1(y):=\frac{M}{2}y^\top Ay,
    \qquad
    R_m(y):=0
    \quad\text{for }m=2,\ldots,M.
\]
Then the average objective is
\[
    R(y):=\frac1M\sum_{m=1}^M R_m(y)=\frac12 y^\top Ay.
\]
Thus \(R\) is \(\lambda\)-smooth and minimized at \(0\), while the largest
component smoothness on this block is \(M\lambda\).

The averaged second-order heterogeneity of this block is
\[
\begin{aligned}
\frac1M\sum_{m=1}^M
\norm{
    \bigl(\nabla R_m(y)-\nabla R_m(z)\bigr)
    -
    \bigl(\nabla R(y)-\nabla R(z)\bigr)
}^2
&=
(M-1)\norm{A(y-z)}^2 \\
&\le
(M-1)\lambda^2\norm{y-z}^2.
\end{aligned}
\]
Hence the block has second-order heterogeneity parameter
\[
    \tau_R=\sqrt{M-1}\,\lambda.
\]

We initialize the rare-curvature block at\footnote{Note that the assumption about bounded optima in \eqref{ass:bounded_optima} can without loss of generality be thought of as an assumption bounding $\norm{x_0-x^\star}$. This is standard in the literature.}
\[
    y_0:=\frac{B_R}{\sqrt2}(u+v),
\]
where \(B_R\) is the radius allocated to this block.  We now lower bound the
error of fresh-sample SGD with replacement on this block.

For an eigendirection of \(A\) with eigenvalue \(a\in\{\lambda,\mu\}\), define
the scalar coordinate
\[
    \theta_t:=\inner{y_t}{w},
\]
where \(w\in\{u,v\}\) is the corresponding eigenvector.  At each step, the
algorithm samples the unique nonzero component with probability \(1/M\).  If
that component is sampled, then the coordinate is multiplied by
\[
    q(a):=1-\eta M a;
\]
otherwise the coordinate is unchanged.

Let \(B_s\in\{0,1\}\) be the indicator that the nonzero component is sampled at
step \(s\), and define
\[
    S_t:=\sum_{s=0}^{t-1}B_s.
\]
Thus \(S_t\) is the number of times the nonzero component has been sampled
before iterate \(t\).  Since the samples are drawn with replacement, the
variables \(B_s\) are i.i.d. Bernoulli with mean \(1/M\).  Along an eigenvalue
\(a\), we have
\[
    \theta_t
    =
    \theta_0 q(a)^{S_t}.
\]
Therefore the coordinate of the averaged iterate
\[
    \bar y_N:=\frac1N\sum_{t=0}^{N-1}y_t,
    \qquad N:=TM,
\]
equals
\[
    \inner{\bar y_N}{w}
    =
    \theta_0 Z_N(q(a)),
    \qquad
    Z_N(q):=\frac1N\sum_{t=0}^{N-1}q^{S_t}.
\]

We need the following elementary estimate.

\begin{lemma}[Averaged multiplier lower bound]
\label{lem:fresh_replacement_multiplier_lb}
Let \(M\ge2\), \(T\ge1\), and \(N=MT\).  Let
\(B_0,\ldots,B_{N-1}\) be independent Bernoulli random variables with
\[
    \mathbb P(B_t=1)=\frac1M.
\]
Define
\[
    S_t:=\sum_{r=0}^{t-1}B_r,
    \qquad
    Z_N(q):=\frac1N\sum_{t=0}^{N-1}q^{S_t}.
\]
Then the following two claims hold.
\begin{enumerate}
    \item If $q\in[1-\nicefrac{1}{4T},1]$, then 
    \[
        \mathbb E[Z_N(q)^2]\ge c_1.
    \]
    \item If \(q\le -2\), then
    \[
        \mathbb E[Z_N(q)^2]\ge \frac{c_2}{T}.
    \]
\end{enumerate}
Here \(c_1,c_2>0\) are universal numerical constants.
\end{lemma}

\begin{proof}
We prove the two claims separately.

\paragraph{Case 1: \(q\in[1-\frac1{4T},1]\).}
Since \(q\in[0,1]\), we have
\[
    \mathbb E[q^{S_t}]
    =
    \left(1-\frac{1-q}{M}\right)^t.
\]
Moreover,
\[
    1-q\le \frac1{4T},
\]
and hence
\[
    1-\frac{1-q}{M}
    \ge
    1-\frac1{4TM}
    =
    1-\frac1{4N}.
\]
Therefore, for every \(t\in\{0,\ldots,N-1\}\),
\[
    \mathbb E[q^{S_t}]
    \ge
    \left(1-\frac1{4N}\right)^N
    \ge c
\]
for a universal constant \(c>0\).  Thus
\[
    \mathbb E[Z_N(q)]
    =
    \frac1N\sum_{t=0}^{N-1}\mathbb E[q^{S_t}]
    \ge c.
\]
By Jensen's inequality,
\[
    \mathbb E[Z_N(q)^2]
    \ge
    \bigl(\mathbb E[Z_N(q)]\bigr)^2
    \ge c^2.
\]
This proves the first claim.

\paragraph{Case 2: \(q\le -2\).}
Let
\[
    Y_t:=q^{B_t}.
\]
Then
\[
    q^{S_t}=\prod_{r=0}^{t-1}Y_r.
\]
Define
\[
    P_t:=q^{S_t}
    =
    \prod_{r=0}^{t-1}Y_r,
    \qquad
    A_N:=\sum_{t=0}^{N-1}P_t.
\]
Thus
\[
    Z_N(q)=\frac{A_N}{N}.
\]

Let
\[
    \beta:=\mathbb E[Y_t]
    =
    1+\frac{q-1}{M},
    \qquad
    \alpha:=\mathbb E[Y_t^2]
    =
    1+\frac{q^2-1}{M}.
\]
Also let
\[
    v:=\operatorname{Var}(Y_t).
\]
Since \(Y_t=1\) with probability \(1-\frac1M\) and \(Y_t=q\) with probability
\(\frac1M\), we have
\[
    v
    =
    \frac1M\left(1-\frac1M\right)(q-1)^2.
\]

Let \(\mathcal F_s:=\sigma(Y_0,\ldots,Y_{s-1})\), where
\(Y_r:=q^{B_r}\), and recall that
\[
    P_t=q^{S_t}=\prod_{r=0}^{t-1}Y_r,
    \qquad
    Z_N(q)=\frac1N\sum_{t=0}^{N-1}P_t.
\]

For \(s=0,\ldots,N-2\), set
\[
    D_s
    :=
    \mathbb E[Z_N(q)\mid\mathcal F_{s+1}]
    -
    \mathbb E[Z_N(q)\mid\mathcal F_s].
\]
We first compute \(D_s\).  The variable \(Y_s\) can affect only those products
\(P_t\) with \(t\ge s+1\).  For such \(t\),
\[
    P_t
    =
    P_sY_s\prod_{r=s+1}^{t-1}Y_r.
\]
Conditioning on \(\mathcal F_{s+1}\), the future variables
\(Y_{s+1},\ldots,Y_{t-1}\) remain independent and have mean \(\beta\), so
\[
    \mathbb E[P_t\mid\mathcal F_{s+1}]
    =
    P_sY_s\beta^{t-s-1}.
\]
Conditioning only on \(\mathcal F_s\), the variable \(Y_s\) has not yet been
revealed and has mean \(\beta\), hence
\[
    \mathbb E[P_t\mid\mathcal F_s]
    =
    P_s\beta^{t-s}.
\]
Therefore,
\[
    \mathbb E[P_t\mid\mathcal F_{s+1}]
    -
    \mathbb E[P_t\mid\mathcal F_s]
    =
    P_s(Y_s-\beta)\beta^{t-s-1}.
\]
The terms \(P_t\) with \(t\le s\) do not depend on \(Y_s\), so they cancel in
the difference of conditional expectations.  Thus
\[
\begin{aligned}
    D_s
    &=
    \frac1N\sum_{t=s+1}^{N-1}
    P_s(Y_s-\beta)\beta^{t-s-1}  \\
    &=
    \frac{P_s(Y_s-\beta)}{N}
    \sum_{r=0}^{N-s-2}\beta^r .
\end{aligned}
\]
Since \(P_s\) is \(\mathcal F_s\)-measurable and \(Y_s\) is independent of
\(\mathcal F_s\),
\[
\begin{aligned}
    \mathbb E[D_s^2]
    &=
    \frac{\mathbb E[P_s^2]\operatorname{Var}(Y_s)}{N^2}
    \left(\sum_{r=0}^{N-s-2}\beta^r\right)^2  \\
    &=
    \frac{\alpha^s v}{N^2}
    \left(\sum_{r=0}^{N-s-2}\beta^r\right)^2 .
\end{aligned}
\]

To relate these increments to the variance of \(Z_N(q)\), define
\[
    M_s:=\mathbb E[Z_N(q)\mid\mathcal F_s],
    \qquad s=0,\ldots,N.
\]
This is the Doob martingale obtained by revealing the variables
\(Y_0,Y_1,\ldots,Y_{N-1}\) one at a time.  Indeed, \(M_s\) is
\(\mathcal F_s\)-measurable, and by the tower property,
\[
    \mathbb E[M_{s+1}\mid\mathcal F_s]
    =
    \mathbb E\!\left[
        \mathbb E[Z_N(q)\mid\mathcal F_{s+1}]
        \mid \mathcal F_s
    \right]
    =
    \mathbb E[Z_N(q)\mid\mathcal F_s]
    =
    M_s.
\]
Thus \((M_s)_{s=0}^N\) is a martingale.  By the definition of \(D_s\),
\[
    D_s
    =
    \mathbb E[Z_N(q)\mid\mathcal F_{s+1}]
    -
    \mathbb E[Z_N(q)\mid\mathcal F_s]
    =
    M_{s+1}-M_s.
\]
Hence
\(\mathbb E[D_s\mid\mathcal F_s]=0\).  If \(r<s\), then \(D_r\) is
\(\mathcal F_s\)-measurable, so
\[
    \mathbb E[D_rD_s]
    =
    \mathbb E\!\left[D_r\,\mathbb E[D_s\mid\mathcal F_s]\right]
    =
    0.
\]
Thus the \(D_s\)'s are pairwise orthogonal.  Also, \(D_{N-1}=0\), since
\(Z_N(q)\) depends only on \(Y_0,\ldots,Y_{N-2}\).  Therefore
\[
    Z_N(q)-\mathbb E[Z_N(q)]
    =
    \sum_{s=0}^{N-2}D_s,
\]
and by orthogonality,
\[
    \operatorname{Var}(Z_N(q))
    =
    \sum_{s=0}^{N-2}\mathbb E[D_s^2].
\]
Combining this identity with the formula for \(\mathbb E[D_s^2]\) gives
\[
    \operatorname{Var}(Z_N(q))
    =
    \frac{v}{N^2}
    \sum_{s=0}^{N-2}
    \alpha^s
    \left(\sum_{r=0}^{N-s-2}\beta^r\right)^2.
\]
Since \(\alpha\ge1\), we can drop the factor \(\alpha^s\) and re-index
\(\ell=N-s-2\) to obtain
\[
    \operatorname{Var}(Z_N(q))
    \ge
    \frac{v}{N^2}
    \sum_{\ell=0}^{N-2}
    \left(\sum_{r=0}^{\ell}\beta^r\right)^2.
\]

We now lower bound the last sum.  Define
\[
    \delta:=1-\beta=\frac{1-q}{M}.
\]
Since \(q\le -2\), we have
\[
    \delta\ge\frac3M.
\]
Also, because \(N=MT\),
\[
    N\delta\ge 3T\ge3.
\]

We claim that
\[
    \delta^2
    \sum_{\ell=0}^{N-2}
    \left(\sum_{r=0}^{\ell}\beta^r\right)^2
    \ge cN
\]
for a universal constant \(c>0\).  To prove the claim, consider two cases.

If \(0\le\beta<1\), then
\[
    \sum_{r=0}^{\ell}\beta^r
    =
    \frac{1-\beta^{\ell+1}}{1-\beta}
    =
    \frac{1-\beta^{\ell+1}}{\delta}.
\]
For all \(\ell+1\ge1/\delta\), we have
\[
    \beta^{\ell+1}
    \le
    e^{-\delta(\ell+1)}
    \le
    e^{-1}.
\]
Thus
\[
    \delta
    \sum_{r=0}^{\ell}\beta^r
    \ge
    1-e^{-1}.
\]
Since \(\delta\ge3/M\), we have
\[
    \frac1\delta\le \frac M3\le \frac N3.
\]
Hence a constant fraction of the indices \(\ell\in\{0,\ldots,N-2\}\) satisfy
\(\ell+1\ge1/\delta\), and the claimed bound follows.

If \(\beta<0\), then for every even \(\ell\), the integer \(\ell+1\) is odd,
so \(\beta^{\ell+1}<0\).  Therefore
\[
    \sum_{r=0}^{\ell}\beta^r
    =
    \frac{1-\beta^{\ell+1}}{1-\beta}
    \ge
    \frac1\delta.
\]
At least a constant fraction of the indices
\(\ell\in\{0,\ldots,N-2\}\) are even, and again
\[
    \delta^2
    \sum_{\ell=0}^{N-2}
    \left(\sum_{r=0}^{\ell}\beta^r\right)^2
    \ge cN.
\]
This proves the claim.

It remains to lower bound \(v\).  Since
\[
    q-1=-M\delta,
\]
we have
\[
\begin{aligned}
    v
    &=
    \frac1M\left(1-\frac1M\right)(q-1)^2  \\
    &=
    (M-1)\delta^2
    \ge
    \frac{M}{2}\delta^2,
\end{aligned}
\]
where we used \(M\ge2\).  Combining this with the claim gives
\[
\begin{aligned}
    \operatorname{Var}(Z_N(q))
    &\ge
    \frac{v}{N^2}
    \sum_{\ell=0}^{N-2}
    \left(\sum_{r=0}^{\ell}\beta^r\right)^2  \\
    &\ge
    \frac{cM\delta^2}{N^2}
    \cdot
    \frac{N}{\delta^2}
    =
    \frac{cM}{N}
    =
    \frac{c}{T}.
\end{aligned}
\]
Finally,
\[
    \mathbb E[Z_N(q)^2]
    \ge
    \operatorname{Var}(Z_N(q))
    \ge
    \frac{c}{T}.
\]
This proves the second claim and completes the proof.
\end{proof}

We now prove the lower bound for the rare-curvature block.

\begin{lemma}[Rare-curvature lower bound for fresh-sample SGD]
\label{lem:fresh_replacement_rare_block_lb}
For the block \(R\) defined above, for every constant stepsize \(\eta>0\),
fresh-sample SGD with replacement satisfies
\[
    \mathbb E\bigl[R(\bar y_N)-R(0)\bigr]
    \ge
    c\frac{\lambda B_R^2}{T},
    \qquad N=TM,
\]
where \(c>0\) is a universal constant.
\end{lemma}

\begin{proof}
We split into two stepsize regimes.

First suppose that \(\eta M\lambda<3\).  Consider the low-curvature direction
\(v\), whose eigenvalue is \(\mu=\lambda/(12T)\).  Then
\[
    1-\eta M\mu
    >
    1-\frac1{4T}.
\]
Thus \(q(\mu)\in[1-\frac1{4T},1]\), and
\Cref{lem:fresh_replacement_multiplier_lb} gives
\[
    \mathbb E[Z_N(q(\mu))^2]\ge c.
\]
The initial coordinate in the \(v\)-direction is \(B_R/\sqrt2\).  Hence
\[
\begin{aligned}
    \mathbb E[R(\bar y_N)-R(0)]
    &\ge
    \frac{\mu}{2}\mathbb E[\inner{\bar y_N}{v}^2] \\
    &=
    \frac{\mu}{2}\cdot \frac{B_R^2}{2}
    \mathbb E[Z_N(q(\mu))^2] \\
    &\ge
    c\mu B_R^2
    =
    c\frac{\lambda B_R^2}{T}.
\end{aligned}
\]

Now suppose that \(\eta M\lambda\ge3\).  Consider the high-curvature direction
\(u\).  Then
\[
    q(\lambda)=1-\eta M\lambda\le -2.
\]
By \Cref{lem:fresh_replacement_multiplier_lb},
\[
    \mathbb E[Z_N(q(\lambda))^2]\ge \frac{c}{T}.
\]
Since the initial \(u\)-coordinate is \(B_R/\sqrt2\), we get
\[
\begin{aligned}
    \mathbb E[R(\bar y_N)-R(0)]
    &\ge
    \frac{\lambda}{2}\mathbb E[\inner{\bar y_N}{u}^2] \\
    &=
    \frac{\lambda}{2}\cdot \frac{B_R^2}{2}
    \mathbb E[Z_N(q(\lambda))^2] \\
    &\ge
    c\frac{\lambda B_R^2}{T}.
\end{aligned}
\]
The two regimes cover all \(\eta>0\), proving the claim.
\end{proof}

\subsection{Calibration of the parameters}

We now calibrate the block to the heterogeneous smoothness regime.  Let
\[
    \lambda:=\frac{\hat H}{M}.
\]
Then the rare block has maximum component smoothness \(\hat H\), average-objective
smoothness \(\lambda=\hat H/M\), and second-order heterogeneity
\[
    \tau_R=\sqrt{M-1}\lambda
    \cong
    \frac{\hat H}{\sqrt M}.
\]
In the regime of \Cref{thm:fresh_sgd_heterogeneous_lb_informal}, we assume
\[
    H\cong \tau\cong \frac{\hat H}{\sqrt M}.
\]
Since \(\lambda=\hat H/M\le H\), the rare block is smoother than the allowed
average-objective smoothness \(H\).

If one wants the full instance to have average-objective smoothness exactly of
order \(H\), we add an orthogonal calibration block
\[
    C_m(s):=\frac{H}{2}s^2,
    \qquad m=1,\ldots,M,
\]
and initialize this coordinate at its minimizer \(s=0\).  This block makes the
average objective \(H\)-smooth, but contributes no error because its gradient
is identically zero along the trajectory.  It also does not change
\(\tau\), since it is identical across clients.  The maximum client smoothness
of the full instance remains of order \(\hat H\).

\subsection{Adding the noise block}

We also add an orthogonal standard stochastic/first-order-heterogeneity block.
We use the following classical lower bound.

\begin{lemma}[Standard stochastic and first-order heterogeneity lower bound]
\label{lem:standard_noise_block}
For every \(N\ge1\), \(B>0\), and \(\sigma,\zeta_\star\ge0\), there exists a
convex stochastic optimization instance with optimum norm at most \(B\), within-client
noise level at most \(\sigma\), and first-order heterogeneity at the optimum
at most \(\zeta_\star\), such that any fresh-sample SGD method using \(N\)
samples and outputting an averaged iterate satisfies
\[
    \mathbb E[F_{\mathrm{noise}}(\bar x_N)-F_{\mathrm{noise}}(x^\star)]
    \ge
    c\frac{(\sigma+\zeta_\star)B}{\sqrt N}.
\]
\end{lemma}

This is the standard mean-estimation lower bound; the \(\zeta_\star\) term is
obtained by taking client-dependent linear perturbations at the optimum, and
the \(\sigma\) term is obtained by adding independent zero-mean stochastic
gradient noise.

\subsection{Proof of the theorem}

\begin{theorem}[Fresh-sample SGD lower bound in the heterogeneous smoothness regime]
\label{thm:fresh_sgd_heterogeneous_lb_formal}
Fix \(M,T\ge2\), \(B>0\), and parameters satisfying
\[
    \tau\cong H
    \cong
    \frac{\hat H}{\sqrt M}.
\]
There exists an instance in
\(\ppp_{\zeta_\star,\tau}^{H,\hat H,B,\sigma}\) such that, for every constant
stepsize \(\eta>0\), fresh-sample SGD with replacement satisfies
\[
    \mathbb E\bigl[F(\bar x_{TM})-F(x^\star)\bigr]
    \ge
    c\left(
        \frac{\hat H B^2}{TM}
        +
        \frac{(\sigma+\zeta_\star)B}{\sqrt{TM}}
    \right).
\]
\end{theorem}

\begin{proof}
Set \(N=TM\).  Place the rare-curvature block, the calibration block, and the
noise block on orthogonal coordinates.  Allocate a constant fraction of the
radius budget \(B\) to the rare-curvature block and another constant fraction
to the noise block.  The calibration block is initialized at its optimum and
uses no radius.

By \Cref{lem:fresh_replacement_rare_block_lb} with
\(\lambda=\hat H/M\), the rare-curvature block contributes
\[
    c\frac{\lambda B^2}{T}
    =
    c\frac{\hat H B^2}{TM}.
\]
By \Cref{lem:standard_noise_block}, the noise block contributes
\[
    c\frac{(\sigma+\zeta_\star)B}{\sqrt{TM}}.
\]
The calibration block ensures that the average objective has smoothness of
order \(H\), and it does not affect the error or the heterogeneity.

The full objective is the direct sum of the three blocks, and the algorithm
evolves independently on the orthogonal coordinates.  Therefore the
suboptimalities add.  The full instance has average-objective smoothness of
order \(H\), maximum client smoothness of order \(\hat H\), and averaged
second-order heterogeneity of order
\[
    \tau\cong\frac{\hat H}{\sqrt M}\cong H.
\]
This proves the theorem, after absorbing the constant factors from the radius
split into the universal constant \(c\).
\end{proof}

    \newpage
	\bibliographystyle{unsrtnat}
	\bibliography{references}	
	
\end{document}

%% file: IC.tex

\begin{tikzpicture}[
    node distance=2cm and 3.5cm,
    line/.style={-Stealth, thick},
    node/.style={circle, draw, thick, fill=gray, minimum size=2cm},
    dashedline/.style={dashed, thick},
    timeline/.style={dashed, thick, -Stealth},
    continueline/.style={dashed, thick, -Stealth, shorten <=2pt},
    brace/.style={decorate, decoration={brace, amplitude=10pt, raise=4pt}, thick}
]

\foreach \x in {1,...,5} {
    \foreach \y in {1,...,4} {
        \node[node] (N-\x-\y) at (\x*7, -\y*3) {};
    }
}

\foreach \x [remember=\x as \lastx (initially 1)] in {2,...,5} {
    \foreach \y in {1,...,4} {
        \draw[line] (N-\lastx-\y) -- (N-\x-\y);
    }
}


\draw[brace] ([yshift=-10pt, xshift=-10pt]N-1-4.south west) -- ([yshift=10pt, xshift=-10pt]N-1-1.north west);
\node[left=0.8cm of N-1-2, anchor=south, rotate=90, align=center] {\Huge $\bm{M}$ \textbf{machines}}; 



\node[below=of N-1-4] (t0) {\Huge $\bm{t=0}$};
\draw[timeline] (t0) -- (N-1-1);

\draw[brace, decoration={mirror}] ([yshift=-10pt]N-1-4.south east) -- ([yshift=-10pt]N-3-4.south west) node[midway, below=15pt, align=center] {\Huge$\bm{K}$ \textbf{units of computation}};

\node[below=of N-3-4] (t2) {\Huge $\bm{t=K}$};
\draw[timeline] (t2) -- (N-3-1);

\draw[brace, decoration={mirror}] ([yshift=-10pt]N-3-4.south east) -- ([yshift=-10pt]N-5-4.south west) node[midway, below=15pt, align=center] {\Huge$\bm{K}$ \textbf{units of computation}};

\node[below=of N-5-4] (t4) {\Huge $\bm{t=2K}$};
\draw[timeline] (t4) -- (N-5-1);

\foreach \y in {1,...,4} {
    \draw[continueline] (N-5-\y) -- ++(1.5, 0);
}


\end{tikzpicture}